\documentclass[12pt]{article} %a4paper
\usepackage[english]{babel}
\usepackage[utf8]{inputenc}
\usepackage{amsmath,amsthm, amssymb}
\usepackage{graphicx}
\usepackage[table,xcdraw]{xcolor}
\usepackage[colorinlistoftodos]{todonotes}
\usepackage{url}
\usepackage{setspace}
\usepackage{float}
\usepackage{bbm}
\usepackage{fullpage}
\usepackage{enumitem}
\usepackage{mathtools}
\usepackage{bold-extra}
\allowdisplaybreaks
\usepackage[section]{placeins}
\theoremstyle{plain}
\newtheorem{thm}{Theorem}[section]

\newtheorem{prop}[thm]{Proposition}

\newtheorem{lemma}{Lemma}

\def\etal{{\em et al.}}
\def\eg{{\em e.g.},\ }

\def\ie{{\em i.e.},\ }
\def\PSSP{{\sc pssp}}
\def\CPH{{\sc cox}}  % versus CPH, Cox PH or Cox Proportional Hazard
\def\AFT{{\sc aft}}
\def\RSF{{\sc rsf}}
\def\RSFKM{{\sc rsf-km}}
\def\KM{{\sc km}}
\def\MTLR{{\sc mtlr}}
\def\ISD{{\sc isd}}
\def\CoxKP{{\sc cox-kp}}
\def\CoxENKP{{\sc coxen-kp}}
\def\Indic#1{{\cal I}\left[\,#1\,\right]}

\def\medi#1{\hat{t}^{(0.5)}_{#1}}

\def\tzero{t^*}
\def\inst#1{\vec{x}_{#1}}
\def\b#1{{B\##1}}
\def\One{1_\cdot}
\def\OneOne{1_{t^*}}
\def\OneAll{1_\forall}

\def\PRsub#1#2{S_{#1}(\, #2 \,)}

\def\CPRsub#1#2#3{\PRsub{#1}{#2\,|\,#3}}

\def\surv#1{S(\, #1\,)}
\def\Csurv#1#2{\surv{#1\,|\,#2}}
\def\estP#1#2{\hat{S}_{#1}(\, #2\,)}
\def\estCP#1#2#3{\estP{#1}{#2\,|\,#3}}

\def\pair#1#2{[\,#1,\ #2\,]}
\def\pid#1{Subject\##1}  % patient id

\long\def\comment#1{}
\def\RG#1{\note[RG]{#1}} % \todo{\tiny [RG: #1]} }
\def\RGi#1{\RG{#1}} % {\small\color{red} $\langle\langle$ RG: #1$\rangle\rangle$}}

\def\median{median}  % or should it be mean ?

\def\est#1#2{\hat{m}^{#1}_{#2}}

\def\death#1{d_{#1}}
\def\censor#1{c_{#1}}

\def\loss#1#2#3{\ell_{#1}(\, #2,\, #3\,)}

\def\cancer#1#2{#2}  % #1 if just with respect to cancer... #2 if the general version

\def\himHer{him/her}
\def\hding#1#2#3#4{\subsection{% paragraph% \cl#1:
#2: #3 (#4)}}
\def\SAc#1#2#3{[#1,$#2$,#3]}
\def\CP#1{\hbox{CP}(\, #1\,)}

\def\mean#1{\hat{\mu_{#1}}}

\def\NacdCol{{\sc Nacd-Col}}

\definecolor{LightCyan}{rgb}{0.88,1,1}

\usepackage{scalerel,stackengine}
\stackMath
\newcommand\reallywidehat[1]{%
\savestack{\tmpbox}{\stretchto{%
  \scaleto{%
    \scalerel*[\widthof{\ensuremath{#1}}]{\kern-.6pt\bigwedge\kern-.6pt}%
    {\rule[-\textheight/2]{1ex}{\textheight}}%WIDTH-LIMITED BIG WEDGE
  }{\textheight}% 
}{0.5ex}}%
\stackon[1pt]{#1}{\tmpbox}%
}

\newcolumntype{H}{>{\setbox0=\hbox\bgroup}c<{\egroup}@{}}

\title{Effective Ways to Build and Evaluate Individual Survival Distributions}
\author{
Humza Haider, 
Bret Hoehn,
Sarah Davis,
Russell Greiner 
\\
Department of Computing Science\\
University of Alberta\\
Edmonton, AB T6G 2E8 \\
\small
\texttt{\{hshaider, bhoehn, sdavis1, rgreiner% , pjin1, 
\}@ualberta.ca} \\
}

\date{\today}

\begin{document}

\comment{
\note[RG]{Confirm: t0 = start time, and t* is represent the reference time for prediction\\
Always CAPITALIZE Concordance\\
And lose the quotes ... ``1-calibration''..
\\
We need to define some terms:\\
for [P,1,g] : FRAMEWORK\\
Cox (or MTLR): MODEL\\
Gail (or some specific learned tool): SYSTEM or learned system or ....\\
}

\note[BH]{I've also requested HH run 2 more experiments:\\
a) run evaluation metrics on a simple ensemble of the best performing models\\
b) take 1 dataset (perhaps NACD Head \& Neck) and run all models with reduced numbers of training examples (but keeping the holdout data for each fold constant)\\
I suspect reviewers of this paper would be critical of our conclusions that MTLR is the best, when we don't compare to more recent methods.  If either of these experiments shows areas where other methods are complementary or perform better than MTLR, I think they would soften our "MTLR is the best!" results
}
}
\maketitle

\def\noCite{}  % {~\cite{#1}}

\begin{abstract} %
An accurate model of a patient’s individual survival 
distribution  
can help determine the appropriate treatment for terminal patients. 
 Unfortunately, risk scores (\eg from Cox Proportional Hazard models)
 do not provide survival {\em probabilities},
%and 
single-time probability models 
(\eg the Gail model, predicting 5 year probability) only provide for a single time point,
and 
standard Kaplan-Meier survival curves % (KM)
provide only {\em population averages}\
for a large class of patients
meaning they are not specific to individual patients.
This motivates an alternative class of tools that can learn a model
which provides an individual survival {\em distribution} 
which gives survival probabilities across all times
-- such as 
extensions
to the Cox model%
\nocite{kalbfleisch2002statistical}, Accelerated Failure Time\nocite{kalbfleisch2002statistical},
an extension to Random Survival Forests\nocite{RandomSurvivalForests},
and 
Multi-Task Logistic Regression% (MTLR)
\nocite{PSSP-NIPS}.
This paper first motivates such ``individual survival distribution'' (\ISD) models,
and explains how they differ from standard models.
It then discusses ways to evaluate such models -- namely Concordance, 1-Calibration,
Brier score,
and various versions of L1-loss-- and then motivates and defines a novel approach ``D-Calibration'',
which determines whether a model's probability estimates 
are meaningful.
We also discuss how these measures differ, 
and use them to evaluate several \ISD\ prediction tools,
over a range of survival datasets.
% some subtle connections between the measures.
% , and helpful, for patients, clinicians and researchers,
\end{abstract}

% earlier version of abstract ... at end of document

\def\keyword#1{\medskip \noindent {\bf #1:}}

\keyword{Keywords}
Survival analysis; risk model; patient specific survival prediction; calibration; discrimination

\comment{ Ideally, this would be an entire distribution -- providing useful information about the mean and variance of the survival time,
as well as other statistics (eg, probability of surviving for 1 year, or for 10 years).
Many tools address some aspect of this...
some provide distributions, but only at the class level (KM), but not individual.
Others are individual, but provide only ...
}

% \note[RG]{Literature: see \hbox{\url{https://docs.google.com/document/d/1fyePa33KY-UzGuSclsiQfNbZfaMi7-s3OY7JfY-TSLI/edit#}} }

\newpage
\section{Introduction} % 1
\label{sec:Intro}
When diagnosed with a terminal disease, many patients ask about their prognosis~\cite{WantToKnowPrognosis}: 
“How long will I live?”, 
or 
“What is the chance that I will live for 1 year... and the chance for 5 years?”.
% , with some specified treatment?
% \RGi{Option 1:}
Here it would be useful to have a meaningful 
``survival distribution''
$\Csurv{t}{\inst{}}$
that provides, for each time $t\geq 0$, 
 the % (estimated) 
 probability that this specific patient $\inst{}$ will survive at least an additional $t$ months. 
 Unfortunately, %  We explain below that 
 many of the standard survival analysis tools 
cannot accurately answer such questions: % are not appropriate here: % 
(1)~risk scores (\eg Cox proportional hazard~\cite{cox1972regression}) provide only {\em relative}\ survival measures, but not the calibrated probabilities desired; 
(2)~single-time probability models (\eg the Gail model~\cite{GailModel1999}) 
provide a probability value
but {\em only for a single time point};
and (3)~class-based survival curves (like Kaplan-Meier, \KM~\cite{kaplan1958nonparametric}) are {\em not specific to the patient}, but rather an entire population.

To explain the last point,
Figure~\ref{fig:StomCan}[left] shows the \KM\ curve for patients with stage-4 stomach cancer.
Here, we can read off the claim that 
% a patient with ... has a 50\% chance of surviving 11 months ...
50\% of the patients
will survive 11~months,
and 95\% will survive at least 2~months.%
\footnote{
\label{ftnote:SC}
In general, a survival curve is a plot where each $\pair{x}{y}$ point 
represents
(the curve's claim that) 
there is a $y\%$ chance of surviving at least $x$ time.
Hence, in Figure~\ref{fig:StomCan}[left],
the \pair{11~months}{50\%}\ point means this
curve  
predicts a 50\% chance of living at least 11 months 
(and hence a 100$-$50 = 50\% chance of dying within the first 11 months). 
The \pair{2~months}{95\%}\ point means 
a 95\% chance of surviving at least 2~months, 
and the \pair{51~months}{5\%}\ point means a 5\% chance of surviving at least 51~months.
}
% [note the median of Figure~\ref{fig:StomCan}[left]].
While these estimates do apply to the population,
{\em on average}, % in general, 
they are not designed to be “accurate'' for an individual patient since these estimates do not include patient-specific information such as age, treatments administered, or general health conditions. 
% Doctors can attempt to adjust these survival time predictions based on these individual differences,
% but such adjustments are typically {\em ad hoc}.
It would be better to directly, and correctly, 
incorporate these important factors $\inst{}$ explicitly in the prognostic models.
\comment{
 –- \ie use the clinical and other medical information  (such as blood tests and performance status assessments~\cite{oken1982toxicity}) 
 that doctors routinely collect during the diagnosis and treatment of the disease
 -- corresponding to the patient-specific $\inst{}$ used above.
as such data can reveal important information about the patient’s state, 
 (\eg about his/her immune system and organ functioning), 
it may be very useful for predicting how well a patient will respond to treatments and ultimately how long s/he will survive. 
}
\comment{ 
  Notice the the patient is asking about her specific outcome.
  To be accurate,
% Of course, the patient here would like information that is directly meaningful, and as accurate as possible.
  the prediction should be personalized,
based potentially on all available information about her
and not just about her ``class'' 
-- \eg in addition to the site and stage of her cancer,
also including information about her blood tests, metabolomic and genomic information, etc.
(This is why class-level tools, such as Kaplan-Meier~\cite{kaplan1958nonparametric}, 
which might provide statistics about site and stage, are not designed for this task.)
She would also like information that is directly meaningful,
such as the mean survival time.
}

This heterogeneity of patients,
coupled with the need to provide probabilistic estimates at several time points,
has motivated the creation of several
{\em individual survival time distribution}\ (\ISD) tools,
each of which can use this wealth of healthcare information from earlier patients, to learn 
a more accurate prognostic model,
which can then predict the \ISD\ of a novel patient
based on all available patient-specific attributes. 
% The \ISD\ models which 
This paper considers several \ISD\ models:
the Kalbfleisch-Prentice 
extension of % estimator extending 
the Cox (\CoxKP)~\cite{kalbfleisch2002statistical}
and the elastic net Cox (\CoxENKP)~\cite{coxEN} model,
the Accelerated Failure Time (\AFT) model~\cite{kalbfleisch2002statistical}, 
the
Random Survival Forest model with Kaplan-Meier extensions (\RSFKM),
and the % recent 
Multi-task Logistic Regression (\MTLR) model%
% Patient-Specific Survival Prediction model (\PSSP)
~\cite{PSSP-NIPS}.
% \note[BH]{PSSP ref is 2011, RSF ref is 2008 -- should we still refer to PSSP as ``recent''?}
Figure~\ref{fig:StomCan}(middle, right) show survival curves (generated by \MTLR)
for two of these stage-4 stomach cancer patients, 
which incorporate other information about these
individual % their respective associated 
patients, such as the patient’s age, gender, blood work, etc.
% (see Table~\ref{tab:dataset}).
%in general, it could also genetic information, metabolomic markers, and anything else known about this individual.)
We see that these prognoses are very different;
in particular, 
\MTLR\ predicts that [middle] Patient \#1's median survival time is 20.2 months, while [right] Patient \#2's is only 2.6 months.
% Note it also claims that \pid{41}\ has a 75\% chance of living 
% at least 13 months, and 25\% chance of living at least 44 months.%
% This tool predicted these curves based only on information available when the data was collected.
The blue vertical lines show the actual times of death;
we see  
% at these curves are also accurate: note
that each of these patients passed away 
very close to \MTLR's predictions of their respective median survival times.
% \note[RG]{Not that close for Patient1...Can you get a curve that is more different -- eg, that has longer median time?\\
% Also: can you include the KM plot in mid/right plots, to show how different it is.}\note[HH]{Done.}

% (See Section~\ref{sec:calibrate} for a discussion about evaluation in general.)
\comment{
\begin{figure} % Fig 1
\comment{Use (a), but for (b),(c): use these patients, but the \MTLR learned from all of NACD!
Explain that we get better results when we use statistics from ALL cancers. }
\hbox{\includegraphics[width=0.33\textwidth,height=1in]{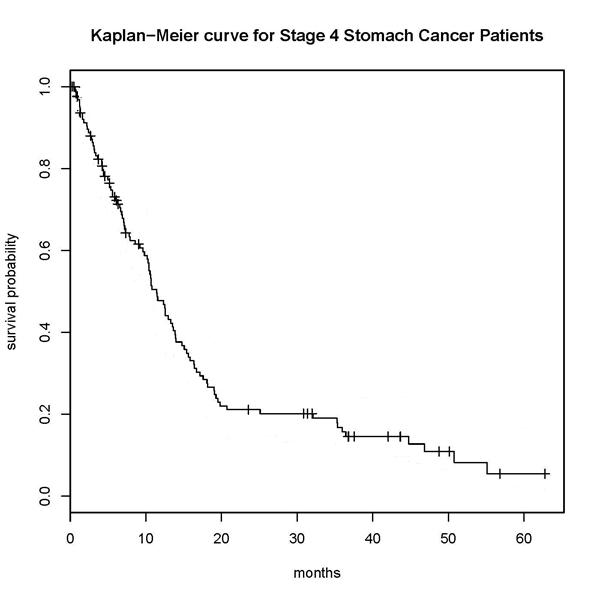}
\includegraphics[width=0.33\textwidth,height=1in]{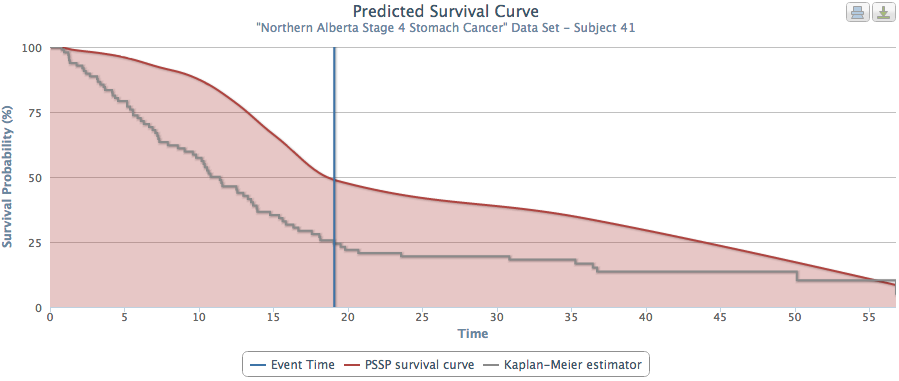}
\includegraphics[width=0.33\textwidth,height=1in]{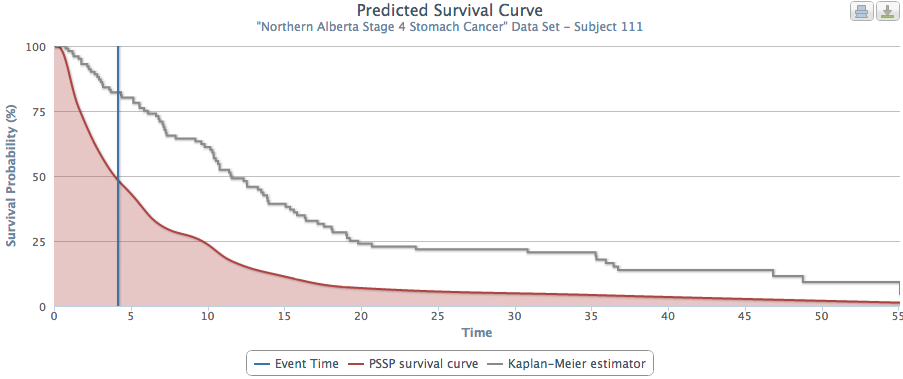}
}
\caption{\label{fig:StomCan}
[left] Kaplan-Meier curve, based on 128 patients with Stage-4 Stomach Cancer.
(middle, right) Two personalized survival curves, for two patients (\#41 and \#111) with Stage-4 Stomach Cancer.}
\end{figure}
}
\begin{figure} % Fig 1
\includegraphics[width=\textwidth]{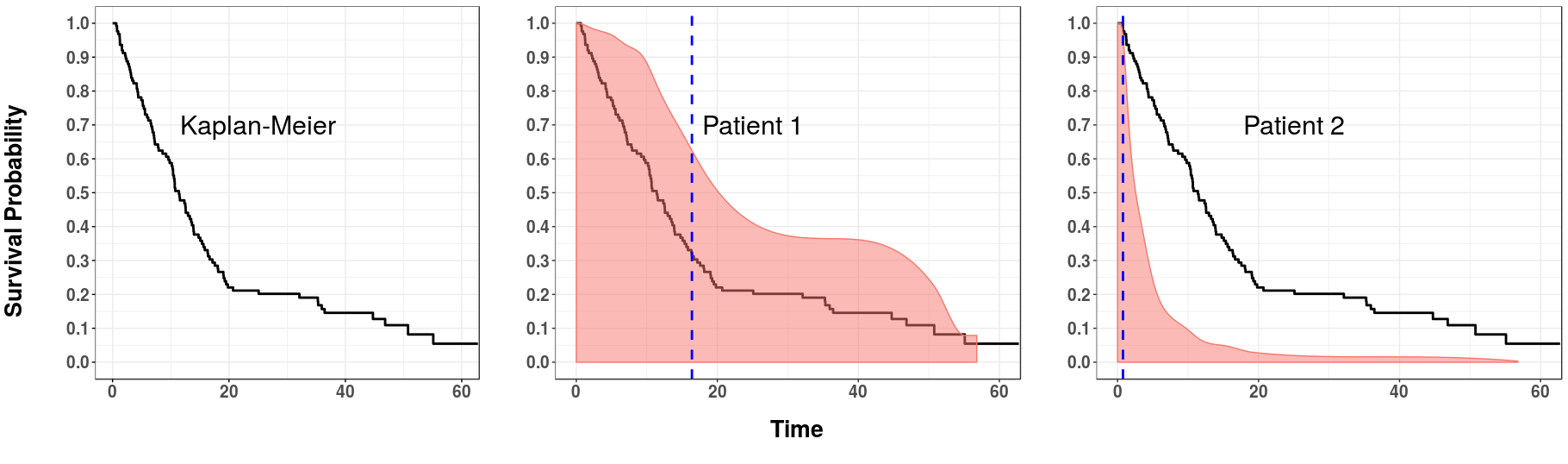}
\caption{\label{fig:StomCan}
[left] Kaplan-Meier curve, based on 128 patients with stage-4 stomach cancer.
(middle, right) Two personalized survival curves, for two patients (\#1 and \#2) with stage-4 stomach cancer. 
The blue dashed lines indicate the true time of death.}
\end{figure}

% \RG{decided to just use "he" for patient, rather than s/he... ok?}
\comment{
Of course, these decisions will only be helpful if our tool produces meaningful accurate results.
A second contribution of this paper is providing a way to determine whether a survival model,
in terms of ``D-Calibration''.
We can a survival model is 
Fortunately, \MTLR\ is typically ``D-Calibrated'':
that is, when it produces a survival curve for a particular patient, 
it is appropriate to tell that patient that he
% \footnote{ To simplify the description, we will use the male gender.}
has 50\% chance of surviving less than the median survival time, and 
50\% of surviving more than that median,
and that he has a 75\% chance of surviving 
until at least the time associated with the 25\% on the curve, etc. 
% (So for Figure~\ref{fig:StomCan}[middle], Patient41 has a 90\% chance of surviving at least 10 months.) 
Section~\ref{sec:Calibration} below further describes this important characteristic,
and demonstrates that our \MTLR\ typically has this property, 
but that some other survival analysis tools do not.
}

% In addition to visualization,
One could then use such curves to make decisions about the individual patient.
Of course, these decisions will only be helpful if the model is giving accurate information
--
\ie only if % after it produces a survival curve for a particular patient, 
it is appropriate to tell a patient 
that s/he
has a 50\% chance of 
dying before % surviving less than 
the median survival time of this predicted curve,
and 
a 25\% chance of dying before % a 75\% chance of surviving until at least 
the time associated with the 25\% on the curve, 
etc.

% \note[RG]{Next paragraph is lifted from Luke's paper ... to avoid self-plaigarizing, need to make sure it is different...}

We focus on ways to {\em learn}\ such models from 
a % labeled 
``survival dataset'' (see below),
describing earlier individuals. 
Survival prediction is similar to regression 
as both involve learning a model that regresses the covariates 
of an individual to 
estimate the value of a dependent real-valued response variable
-- here, that variable is ``time to event'' (where the standard event is ``death'').
But survival prediction differs from the standard regression task as 
its response variable is not fully observed in all training instances
-- this task allows many of the instances 
to be % are
``right censored'',
in that we only see a {\em lower bound}\ of the response value.
This might happen if a subject was alive when the study ended,
meaning we only know that she lived {\em at least}\ (say) 5 years
after the starting time, 
but do not know whether she actually lived 5 years and a day, or 30 years.
This also happens if a subject drops out of a study, after say 2.3 years, and is then lost to follow-up; etc.
Moreover, one cannot simply ignore such instances as it is common for many (or often, {\em most}) of the training instances to be right-censored; 
see Table~\ref{tab:datasets}.
Such ``partial label information'' is problematic for standard regression techniques,
which assume the label is completely specified for each training instance.
Fortunately, there are survival prediction algorithms that can learn an effective model, 
from a 
 cohort % ? training dataset 
that includes such censored data.
Each such ``survival dataset'' contains descriptions of a set of instances (\eg patients),
as well as two ``labels'' for each:
one is the time, 
corresponding to the 
{\em time from diagnosis to a final date} (either death, or time of last follow-up) and
the other is the {\em  status} bit,
which indicates whether the patient was alive at that final date. 
Section~\ref{sec:MLparadigm} summarizes several popular models for dealing with such survival data.
\comment{\note[RG]{Do we need next sentence? Notice we mention this in a footnote below.}\note[BH]{Since we mention this later, I'm taking it out here}
(We also note that there are many survival {\em analysis}\ tools
that instead use a survival dataset to identify {\em biomarkers}\ -- 
each a feature that, by itself, is related to an individual's survival.
We will later use this facility to reduce the number of features considered
by our survival prediction models.)}
% -- our goal is more similar to regression:  making a prediction about an individual patient.

This paper provides three contributions:
(1)~Section~\ref{sec:MLparadigm} motivates the need for such \ISD\ models
by showing how they differ from more standard survival analysis systems.
(2)~Section~\ref{sec:Eval} then discusses several ways to evaluate such models,
including standard measures (Concordance, 1-Calibration, Brier score),
variants/extensions to familiar measures (L1-loss, Log-L1-loss), 
and also a novel approach, ``D-Calibration'' which can be used to assess the quality of the individual survival curves generated by \ISD\ models.
(3)~Section~\ref{sec:EvalISD} evaluates several \ISD\ (and related) models 
(standard: \KM, \CoxKP, \AFT\
and more recent: \RSFKM, \CoxENKP, \MTLR)
% \note[BH]{PSSP published in 2011 -- should we say ``novel''?}
on 8 diverse survival datasets, 
in terms of all 5 evaluation measures. 
We will see that \MTLR\ 
does well -- 
typically outperforming % generally outperforms
the other models
in the various measures,
and % in particular,
often showing vast improvement 
% over other models
in terms of calibration metrics.

% end comment ... line 499

The appendices provide relevant auxiliary information:
Appendix~\ref{app:SC-to-0}
describes some important nuances
about survival curves. % (going to 0)
Appendix~\ref{app:Evaluation} provides further details concerning all the evaluation metrics 
and 
in particular, how each addresses censored observations. 
% particularly how to address censored observations for each of the metrics. 
It also % Additionally, Appendix~\ref{app:Evaluation} 
contains some relevant proofs 
about our % regarding the 
novel D-Calibration metric. 
% considered. 
Appendix~\ref{app:ISDDetails} 
then % further
explains some additional aspects 
of the \ISD\ models considered in this paper. 
Lastly, Appendix~\ref{app:EmpiricalDetail} gives the detailed results from empirical evaluation
-- 
\eg 
providing detailed tables corresponding to the
 results shown as figures in Section~\ref{sec:Empirical Results}.
% some results are given in Section~\ref{sec:Empirical Results} as figures, the corresponding tables are in Appendix~\ref{app:EmpiricalDetail}.

For readers who want an introduction to 
% We assume the readers are familiar with foundations of 
survival analysis and prediction,
we recommend \textit{Applied Survival Analysis} by Hosmer and Lemeshow~\cite{hosmer2011applied}.
% An earlier survey % 
% The prior work by 
Wang et al.~\cite{wang2017machine} surveyed
machine learning techniques and evaluation metrics for survival analysis. 
However, that work primarily 
\comment{\change[RG]{surveyed the many different models applicable for survival analysis, and gave an overview for each.
In addition, they briefly discussed some of the evaluation techniques and application areas that survival analysis has
such as healthcare, customer churn, and student/employment retention.}}
{overviewed the standard survival analysis models,
then briefly discussed some of the evaluation techniques and 
application areas.}
Our work, instead, focuses on the \ISD-based models --
first motivating why they are relevant for survival prediction 
% \annote[RG]
{(with a focus on medical situations)}
% {is this relevant? Delete?},
then providing empirical results showing the strengths and weaknesses of each of the models considered.

\section{Summary of 
Various Survival Analysis/Prediction Systems}
% Relation to Other Survival Analysis \hbox{Systems} % 2
\label{sec:MLparadigm}

There are many different survival analysis/prediction tools,
designed to deal with various different tasks.
We focus on tools that learn the model 
% an estimation / machine learning approach~\cite{WekaBook} here,
% of learning the patterns needed to produce a curve for a novel patient (based on his features)
% that produces this curve from the patient description
from a survival dataset,
\begin{equation}
D\quad=\quad \{\, 
[\inst{i},\, t_i,\, \delta_i]\, \}_i
\label{eqn:SurData}
\end{equation}
% ``labeled'' dataset of historical patients,
which provides the values for features 
$\inst{i} = [ x_{i}^{(1)},\, \cdots,\, x_{i}^{(k)}]$
for each member of a cohort of historical patients,
as well as the actual time of the ``event'' $t_i \in \Re^{\geq 0 }$
which is either death (uncensored) or the last visit (censored),
and a bit $\delta \in \{0,1\}$ that serves as the indicator for death.%
\footnote{ 
Throughout this work we focus on only Right-Censored survival data. Additionally, we constrain our work to the standard machine-learning framework,
where our predictions are based only on information available at fixed time $t_0$ 
(\eg start of treatment).
\comment{\note[RG]{Should we move next sentence to "future work", at end?}\note[BH]{This next statement seems to just muddy the water and I recommend removing it entirely}
We do {\em not}\ consider the challenge of making predictions at time $t_0$, based on both information known at time $t_0$ and 
also later (at $t_1 > t_0$), etc.}
While these descriptions all apply when dealing with the time to an
arbitrary
{\em event},
our descriptions will refer to ``time to death''.
} 
See Figure~\ref{fig:LearnISD},
in the context of our \ISD\ framework.

Here, we assume $\inst{}$ is a vector of 
 feature values describing a patients,
 using information
that are available when that patient entered the study -- 
\eg when the patient was first diagnosed with the disease, or started the treatment. Additionally, we % choose to 
assume each patient has a death time, $d_i$, and a censoring time, $c_i$, 
and assign $t_i := \min\{d_i, c_i\}$ and 
$\delta_i = \Indic{d_i \leq c_i}$
where $\Indic{\cdot}$ is the Indicator function
-- \ie 
$\delta_i := 1$ if $d_i \leq c_i$ or $\delta_i := 0$ if $d_i > c_i$. 
%Additionally, we % second use 
We follow the standard convention that $d_i$ and $c_i$ are assumed independent.

\begin{figure}\centering % Fig 2 %1.2  3.0
\includegraphics[% width=0.8\textwidth,
height=3.5in % 3.2in
]
{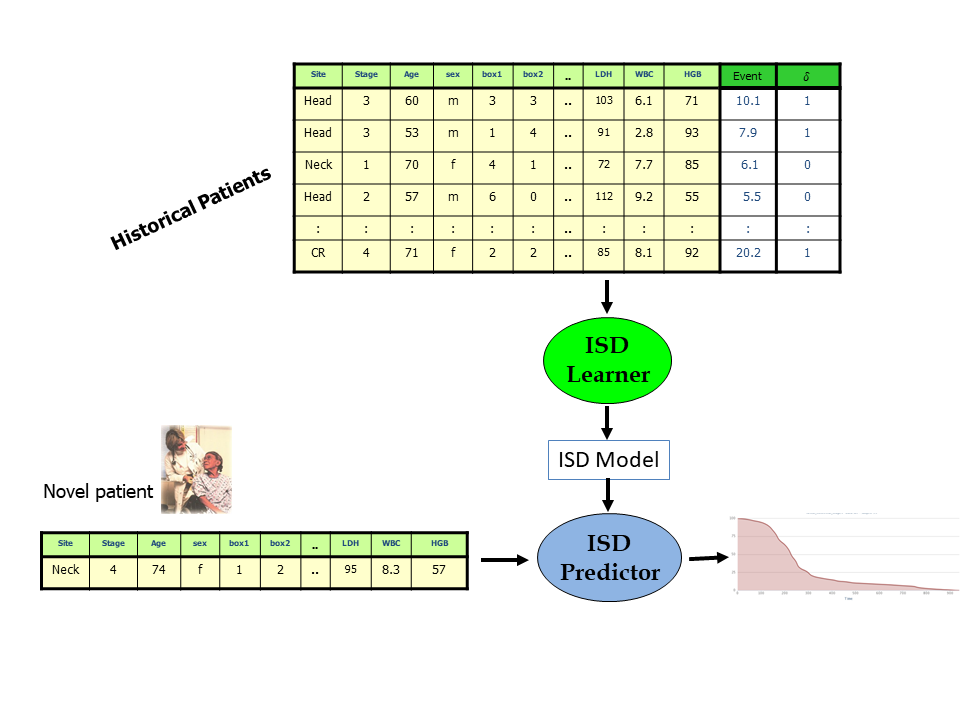}
\vspace{-0.5in}
\caption{Machine Learning paradigm for learning, 
then using,
an \ISD\ (Individual Survival Distribution) Model.}
\label{fig:LearnISD}
\end{figure}

To help categorize the space of survival prediction systems, 
we consider 3 independent % orthogonal 
characteristics:
\begin{itemize}
\item {\em [R vs P]} whether the system provides, 
for each patient, a risk score $r(\inst{}) \in \Re$
% (or perhaps $r(\inst{}, t) \in \Re$)
versus 
a probabilistic value 
$% \estCP{}{t}{\inst{}} 
\in [0,1]$ (perhaps $\estCP{}{t}{\inst{}}$).

\item {\em [$\OneOne$ vs $\OneAll$ vs $\infty$]}
whether the system returns a {\em single}\ value for each patient
(associated either with a single time 
``$\OneOne$''
or with the overall survival
``$\OneAll$''), 
versus a range of values, one for each time.
Here $\OneOne$ might refer to $\estCP{}{t^*}{\inst{}}\in [0,1]$ for a single time $t^*$ 
and
$\OneAll$ if there is a single ``atemporal'' value 
(think of the standard risk score, which is not linked to a specific time),
vs 
$\infty$ that refers to
$\{\ [t,\, \estCP{}{t}{\inst{}}]\ \}_{t \geq 0}$ over all future times $t\geq 0$.

\item {\em [i vs g]} whether the result is 
``{\em i}\,''~specific to a single individual patient
(\ie based on a large number of features $\inst{}$)
or is ``{\em g}\,''~general to the population.
This {\em g} also applies 
if the model deals with a {\em fixed set of subpopulations}\
-- perhaps each contains all patients
with certain values of only one or two features
(\eg subpopulation $p1$ is all men under 50,
 $p2$ are men over 50, and $p3$ and $p4$ are corresponding sets of women),
or each subpopulation is a specified range
of some computation
(\eg $p1'$ are those with BMI$<$20,
$p2'$ with BMI$\in [20,\,30]$
and $p3'$, with BMI$>$30).

% based on an intermediate computed value,
% that is binned, into a small number of bins.)

\end{itemize}
This section summarizes 5 (of the $2 \times 3 \times 2 = 12$) classes of survival analysis tools 
(see Figure~\ref{fig:TypesSA}),
giving typical uses of each,
% (see summary in Table~\ref{tab:CompareTools}), 
then discusses how they are interrelated.

\begin{figure} % Fig 3
\hspace*{-0.4in}\includegraphics[width=1.08\textwidth]{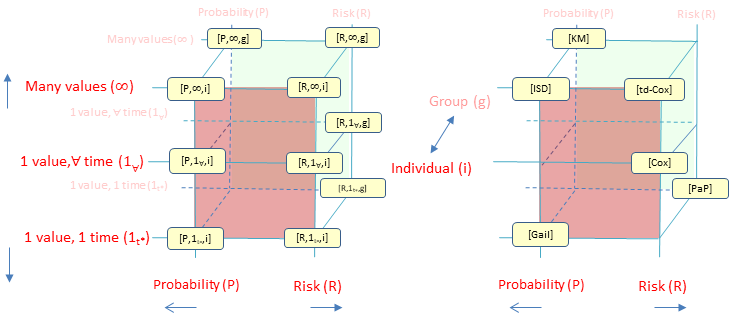} 
% SA-Frameworks.png use 1_1, and no tCox
% TypesOfSA-Ex.png
% {Figures/TypesOfSA.png}
% this is from PSSP-training.pptx ... also in Gdrive/My Papers / PSSP/Foundations/ ...
\caption{\label{fig:TypesSA}
Dimensions for cataloging types of Survival Analysis/Prediction tools [left] --
and examples of certain tools.}
\end{figure}

\hding{1}{\SAc{R}{\OneAll}{i}}{1-value Individual Risk Models}{\CPH}
% had been \hding{1}{\SAc{R}{\infty}{i}}{Atemporal Individual Risk Models}{\CPH}
\label{sec:GRisk}
% Other systems do incorporate information about patients,
% and are used to provide patient-specific information.
An important class of survival analysis tools compute ``risk'' scores, $r(\inst{}) \in \Re$ for each patient $\inst{}$,
with the understanding that $r(\inst{a} ) > r(\inst{b})$
corresponds to predicting that $\inst{a}$ will die before $\inst{b}$.
Hence, this is a {\em discriminative}\ tool
for comparing pairs of patients,
or perhaps for ``what if'' analysis of a single patient
(\eg if he continues smoking, versus if he quits).
% THESE APPEAR BELOW ... see \cite{Framingham}).
% \note[RG]{any examples of what-if uses?  Perhaps Framingham studies -- telling patients to stop smoking...
% Q: Doesn't this really require unbiased training -- RCT ...Or maybe this is non-issue, due to strong independence assumption.}
These systems are typically evaluated using a discriminative measure, 
such as ``Concordance'' (discussed in Section~\ref{sec:Concordance}).
Notice these tools each return a single real value for each patient.
% Here, this value is actually atemporal -- \ie it does not depend on a specific time.
% \note[RG]{But see ?tCox model, which returns different risk scores for different times.}

\comment{
Here, we consider tools that provide atemporal risk assessments;
the ``$\infty$'' in the categorization does NOT mean that that the mode produces a large number of points, but only that its value is not specific to a single time-point.
% Eg, Cox gives a single value, but it still applies to all time.
The \SAc{R}{1}{i}\ class, % cl{2}\ 
described below,
considers various tools that provide scores for a range of specific time.
}

One standard generic tool here is 
the Cox Proportional Hazard (\CPH) model~\cite{cox1972regression},
which is used in a wide variety of applications.
% \RG{find some!}
This models the hazard function%
\footnote{
The hazard function (also known as the failure rate, hazard rate, or force of mortality) 
$h(t; \inst{})\ =\ p(t\,|\, \inst{}) / \Csurv{t}{\inst{}}$ 
is essentially the chance that $\inst{}$ will die at time $t$, given that s/he has lived until this time,
using the survival PDF $p(t\,|\, \inst{})$.
When continuous,
$h(t; \inst{})\ =\ - \frac{d}{dt} \log \Csurv{t}{\inst{}}$.
% the ratio of the probability density function P(x) to the survival function S(x), given by
}
as 
\begin{equation}
h_{\CPH}(\,t,\,\inst{}\,)\quad =\quad \lambda_0(t)\, \exp(\vec{\beta}^T \inst{})
\label{eqn:CoxPH}
\end{equation}
where $\vec{\beta}$ are the learned weights for the features,
and $\lambda_0(t)$ is the baseline hazard function.
We view this as a Risk Model by ignoring $\lambda_0(t)$ (as $\lambda_0(t)$ is the same for all patients), 
and focusing on just
$ \exp(\vec{\beta}^T \inst{}) \ \in \Re^+$.
(But see the \CoxKP\ model below, in \SAc{P}{\infty}{i}.)
% Such scores are now used to predict an individual's Risk Index
% at many websites, 
\comment{
\note[RG]{Previously had the Use Cox for FS here... moved to S 2.6}
Note many researchers use the Cox model (Equation~\ref{eqn:CoxPH}) for this,
by testing if the $\hat{\beta}_i$ coefficient associated with feature $x_i$ is significantly different from 0.
\note[RG]{We will later use this approach to select features,
as a pre-processing step, before running the actual survival prediction model.
Hmmm... this appears in Sec1 -- should it remain there, and also here?}
}
There are many other tools for predicting an individual's risk score,
typically with respect to some % often specific to the 
disease;
see for example the Colditz-Rosner model~\cite{colditz2000cumulative},
and the myriad of others appearing on the Disease Risk Index website% 
\footnote{\url{http://www.diseaseriskindex.harvard.edu/update/}}.
For all of these models, the value returned is atemporal --
\ie it does not depend on a specific time.
There are also tools that produce \SAc{R}{\infty}{i}\ models,
that return a risk score associated across all time point;
see Section~\ref{sec:Concordance}.

\comment{
{\tt http://aje.oxfordjournals.org/content/152/10/950.long}\\
 nonlinear Poisson regression that accounts for time (and summarizes risk to age 70 years\\
 Comparison of the long-term impact of risk factors for breast cancer is often difficult because risk factors frequently change in magnitude and even direction over different periods of life. Cumulative incidence to age 70 years provides one measure that avoids these limitations of age-specific relative risks.\\
 %
% Thus, λi = Ci+I/Ci represents the rate of increase of breast cell divisions from age i to age i + 1. Log (λi) is assumed to be a linear function of risk factors that are relevant at age i
 }

\comment{
It is typically not used to predict survival
time since the hazard function is incomplete without the baseline hazard 
$\lambda_0(t)$ function. 
Although we can fit
a non-parametric survival function for $\lambda_0(t)$ after the coefficients of Cox regression are determined [2], 
this requires a cumbersome 2-step procedure  }

\hding{2}{\SAc{R}{\OneOne}{g}}{Single-time Group Risk Predictors: Prognostic Scales}{PPI, PaP}
\label{sec:1timePred}
Another class of risk predictions explicitly focus on a single time,
leading to prognostic scales,
some of which are computed using 
% nomograms~\cite{NomogramsOncology2015} or
Likert scales~\cite{rogers2001use}.
For example,   
the Palliative Prognostic Index (PPI)~\cite{morita1999palliative}
% http://pda.rnao.ca/content/palliative-prognostic-index-ppi
computes a risk score for each terminally ill patient,
which is then used to assign that patient into one of three groups.
It then uses statistics about each group 
to predict that patients in one group will do better at this specific time (here, 3 weeks),
than those in another group.
% relative chance that patient will survive longer than 3 weeks, or longer than 6 weeks.
Similarly, the Palliative Prognostic Score  (PaP)~\cite{pirovano1999new}
uses a patient's characteristics to assign \himHer\ into one of 3 
risk groups, which can be used to estimate the 30-day survival risk.
% \url{http://www.eperc.mcw.edu/EPERC/FastFactsIndex/ff_124.htm}
(There are many other such prognostic scales, including~\cite{chuang2004prediction,anderson1995palliative,haybittle1982prognostic}.)
Again, these tools are typically evaluated using Concordance.%
\footnote{
Here, they do not compare pairs of individuals from the same group,
but only patients from different groups, 
whose events are comparable (given censoring); 
see Section~\ref{sec:Concordance}.}

\hding{3}{\SAc{P}{\OneOne}{i}}{Single-time Individual Probabilistic Predictors}{Gail, PredictDepression}
Another class of single-time predictors 
each produce a {\em survival probability}\ $\estCP{}{\tzero}{\inst{}} \in [0,1]$ 
for each individual patient $\inst{}$,
for a single fixed time $\tzero$ 
-- which is the {\em probability} $\in [0,1]$ that $\inst{}$ will survive to at least time $\tzero$.
% -- rather than a risk score $r(\inst{}) \in \Re$.
For example, the Gail model [Gail]~\cite{GailModel1999}
\footnote{\url{http://www.cancer.gov/bcrisktool/}}
estimates the probability that a woman will develop breast cancer within 5~years %
%\footnote{\note[RG]{ is this needed, given Sec 2.6?}The Gail website also predicts the 25 year survival, as well as 5 year.This is considered ``1'' as these predictions are based on different models.}
% or within 25 years, 
based on her responses to a number of survey questions.
Similarly, 
the PredictDepression system [PredDep]~\cite{wang2014prediction}
\footnote{\url{http://predictingdepression.com/}}
predicts the probability that a patient will develop a major depressive episode in the next 
4~years based on a small number of responses.
The Apervite\footnote{\url{https://apervita.com/community/clevelandclinic}} 
and R-calc\footnote{\url{http://www.r-calc.com/ExistingFormulas.aspx?filter=CCQHS}} websites
each include dozens of such tools,
each predicting the survival probability
for 1 (or perhaps 2) fixed time points,
for certain classes of diseases.

%  $\estCP{}{\tzero}{\inst{}} \in [0,1]$ 

Notice these probability values 
% These extend the {\em Prognostic Scale}\ models by providing specific values for each patients, that 
have semantic content, and are labels for {\em individual patients}
(rather than risk-scores, which are only meaningful 
within the context of other patients' risk scores).
%for a {\em pair of patients}\,).
These systems should be evaluated using a calibration measure, such as 
1-Calibration 
% \note[RG]{Do we need the quotes here, and elsewhere? (After first appearance)}
% \change[SD]{and}
{or} Brier score (discussed in Sections~\ref{sec:1-Calib} and~\ref{sec:BrierScore}).

\hding{4}{\SAc{P}{\infty}{g}}{Group Survival Distribution}{\KM}
\label{sec:KM-model}
There are many systems that can produce a survival distribution:
a graph of $[t, \estP{}{t}]$,
showing the survival probability $\estP{}{t} \in [0,1]$
for each time $t\geq 0$;
see Figure~\ref{fig:StomCan}.
The Kaplan-Meier analytic tool (\KM) is at the ``class'' level,
producing a distribution designed to apply to everyone in a sub-population:
$\estCP{}{t}{\inst{}} = \estP{}{t}$, 
for every $\inst{}$ in some class --
\eg the \KM\ curve in Figure~\ref{fig:StomCan}[left] applies to 
every patient $\inst{}$ with stage-4 stomach cancer.
The SEER website\footnote{\url{http://seer.cancer.gov/}} provides a set of Kaplan-Meier curves for various cancers.
While patients can use such information to estimate their survival probabilities,
the original % another typical
goal of that analysis is to better understand the disease itself,
perhaps by seeing whether some specific feature made a difference,
or if a treatment was beneficial.
For example, 
we could produce one curve for all stage-4 stomach cancer patients who had treatment tA,
and another for the disjoint subset of patients who had no treatment;
then run a log-rank test~\cite{LogRankTest} to determine whether (on average) 
patients receiving treatment tA survived statistically longer than those who did not.
% The next 
Section~\ref{sec:Eval} below
% \SAc{P}{\infty}{i}\ % \cl{5}
describes various ways to evaluate 
\SAc{P}{\infty}{i}\ % \ISD\ 
models;
we will use these measures to evaluate \KM\ models as well.

\comment{
For example, Kaplan-Meier analysis (\KM) 
tries to determine if some specific feature led to statistically different performance.
This involves comparing sub-populations, 
each corresponding to 
% However, as \KM\ curves were designed to compare different subpopulations, each curve applies to 
a large group of individuals, characterized by only a few features.
For example, we could imagine extending 
Figure~\ref{fig:StomCan}[left] to produce one \KM\ curve that applies to everyone with stage-4 stomach cancer who is {\em male}, 
and another \KM\ curve for everyone with stage-4 stomach cancer who is {\em female},
and then run a log-rank test~\cite{LogRankTest} to determine whether (on average) one gender survives statistically longer than the other. 
In that case (not shown), each of the curves would include all patients with the same site and stage of cancer;
but differ only with respect to gender.
Of course, there are still many differences between individuals within each such sub-population, based on (say) age, histology, etc.
However these \KM\ curves do not incorporate such information,
as they were not designed to tell each specific patient about their
survival distribution.
}

\hding{5}{\SAc{P}{\infty}{i}}{Individual Survival Distribution, \ISD}{\CoxKP, \CoxENKP, \AFT, \RSFKM, \MTLR}
\label{sec:ISD}
The previous two subsections described two frameworks: % probabilistic models:
\begin{itemize}
\item 
\SAc{P}{\OneOne}{i}\ % \cl{3}\ 
tools, 
% \note[RG]{tools in the P1i framework ??}
which produce an 
{\em individualized}\ probability value $\estCP{}{\tzero}{\inst{i}} \in [0,1]$,
but only for a single time $\tzero$; and
\item \SAc{P}{\infty}{g}\ % \cl{4}\ 
tools, which produce the entire survival probability curve  $[t,\estP{}{t}]$
{\em for all points $t \geq 0$},
but are not individuated -- \ie the same curve for all patients
$\{\, \inst{i}\,\}$.
\end{itemize}
Here, we consider an important extension: a tool that produces
{\em the entire survival probability curve  $\{\ [t,\,\estCP{}{t}{\inst{i}}]\ \}_t$
for all points $t \geq 0$},
{\em specific to each individual patient, $\inst{i}$}.
As noted in the previous section, 
this is required by any application
that requires knowing meaningful survival probabilities for many time points.
% --  $\estCP{}{t}{\inst{}}$, for many $t$'s, for this $\inst{}$ --
This model also allows us to compute other useful statistics,
such as a specific patient's expected survival time.

We call each such system an ``Individual Survival Distribution'' model, \ISD.
% There are several tools that provide this functionality.
While the Cox model is often used just to produce the risk score,
it can be used as an \ISD, given an appropriate (learned)
baseline hazard function $\lambda_0(t)$;
see Equation~\ref{eqn:CoxPH}.
We estimate this using the Kalbfleisch-Prentice estimator~\cite{kalbfleisch2002statistical},
and call this combination ``\CoxKP'';
we also consider a regularized Cox model, 
%{\em viz.},
namely % BH was unfamiliar with viz. and seemed like unnecessary use of a fancy word
the elastic net Cox with the Kalbfleisch-Prentice extension (\CoxENKP).
We also explore three other models:
Accelerated Failure Time model~\cite{kalbfleisch2002statistical} with the Weibull distribution (\AFT), 
% as well as 
Random Survival Forests with the Kaplan-Meier extension (\RSFKM , 
described in Appendix~\ref{app:RSFKM})~\cite{RandomSurvivalForests} and 
Multi-task Logistic Regression system (\MTLR)~\cite{PSSP-NIPS}.
% which also makes strong parametric assumptions;
% , which again means its curves do not cross
Figure~\ref{fig:AllCurves} shows the curves
from these various models,
each over the same set of individuals.
% (from the same dataset). 

\begin{figure} % Fig 4
\includegraphics[width=\textwidth]
{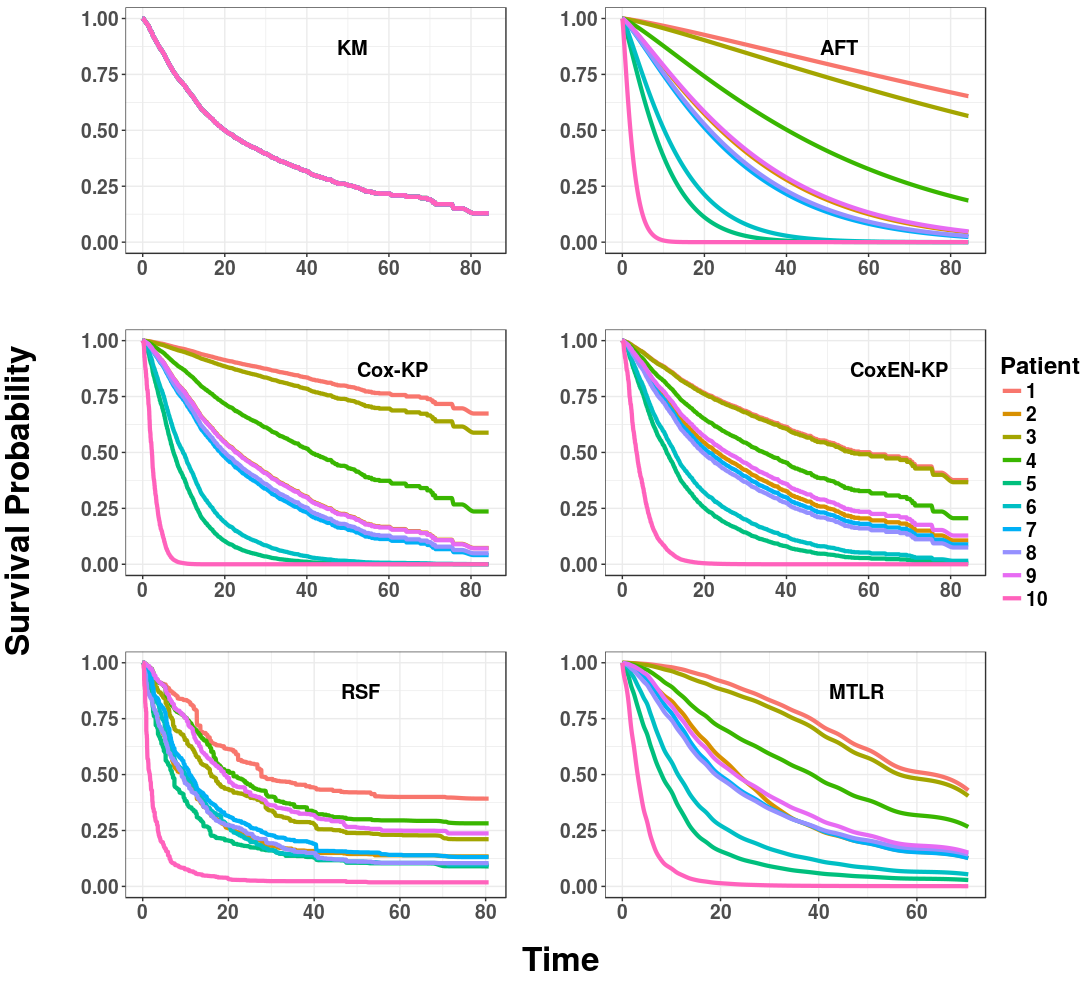}%
\caption{\label{fig:AllCurves}
Survival curves (within one of 5 folds) of 10 cancer patients for
the \KM\ model and all 5 \ISD\ models considered here,  
evaluated on the NACD dataset (described in Section~\ref{sec:DataSets}). 
As the \KM\ curve (top left) is the same for all patients by definition,
we provide only 1 curve
-- here smoothed.
Note that
the set of curves for \AFT\ (with the Weibull distribution),
\CoxKP, and \CoxENKP\ 
each have roughly the same shape, and
do not cross,
% all have curves that cannot cross
due to the 
{proportional hazards assumption},
whereas the curves for \RSFKM\ and \MTLR\ 
can cross. % have crossing survival curves. 
}
\end{figure}

%Old Cox and\AFT\ Curve Figure:
%%%%%%%%%%%%%%%%%%%%%%%%%%%%%%%%%%%%%%
\comment{
\begin{figure} % Fig 3
\includegraphics[width=0.5\textwidth]% ,height=1.5in
% [width=0.5\textwidth] % ,height=2in]
{Figures/Cox-results.jpg}%
\includegraphics[width=0.5\textwidth] % ,height=2in]
{Figures/aft.pdf}
\caption{\label{fig:Cox}
Survival curves for 4, of the 2402, patients in the NACD dataset,
using the  
[left] Cox Proportional Hazard and [right] Accelerated Failure Time.
Notice the curves in each sub-figure 
share the same basic shape,
and they do not cross.
}
\end{figure}}
%%%%%%%%%%%%%%%%%%%%%%%%%%%%%%%%%%%%%

%Figure containing 20\MTLR\ curves from NACD.
%%%%%%%%%%%%%%%%%%%%%%%%%%%%%%%%%%%%%%%%%%%%%%%%%%%%%%%%%%%%%%%%%%%%%%%%%%%%%
\comment{
\begin{figure} % Fig 4
\includegraphics[width=\textwidth] % ,height=2in]
{Figures/15curves-NACD.png}
\caption{\label{fig:allCurves}
\note[HH]{Remove this figure}
\MTLR's survival curves for 15, of the 2402, patients in the NACD dataset.
Note different lines have very different shapes, and they can cross one another.
The dotted-grey curve is the Kaplan-Meier curve for this dataset.
See {\tt  http://pssp.srv.ualberta.ca/predictors/1407/examine\_percentiles}}
\end{figure}
}
%%%%%%%%%%%%%%%%%%%%%%%%%%%%%%%%%%%%%%%%%%%%%%%%%%%%%%%%%%%%%%%%%%%%%%%%%%%%%
Above,
we briefly mentioned three evaluation methods: Concordance, 1-Calibration, and Brier score.
We show below that we can use any of these methods to evaluate a \ISD\ model.
In addition, we can also use variants of ``L1-loss'', to see how far a predicted single-time differs from the true time of death; see Section~\ref{sec:L1-loss}.
Each of these 4 methods considers only a single time point of the distribution,
or an average of scores, each based on only a single time,
or a single statistic (such as its median value).
We also consider a novel evaluation measure, 
``D-Calibration'',
which uses the entire distribution of estimated survival probabilities%,
%The wording of "uses the entire distribution as a distribution" was confusing 
%as a distribution%
% \footnote{As opposed to just a sequence of individual time-specific measures, like ``Integrated Brier Score''; in Section~\ref{sec:BrierScore}. }
;
see % discussed in 
Section~\ref{sec:d-calibration}.

\comment{
An \ISD\ resembles \KM\ in that it also produces survival curves.
However, as \MTLR\ is a ``survival {\em prediction}'' tool,
% which is designed to describe individual patients,
it produces a different curve for each patient, 
based on the characteristics of that patient.
Figure~\ref{fig:StomCan}(middle,right) earlier presents two such curves;
see Figure~\ref{fig:allCurves} for 15 such \MTLR\ curves, 
along with the single \KM\ curve for this entire population.
% By contrast, \PSSP\ was designed for this specific-patient prediction task,
% which means it produces a curve for each patient; see Figure~\ref{fig:allCurves}.
We can also compute simple point estimates
for each patient
-- 
\eg his median survival time --
then see whether these times are close 
to when the patient actually died.
Our earlier paper~\cite{PSSP-NIPS} compared
% Section~\ref{sec:otherEval} compares 
the quality of these point-estimates based on \KM,
to those based on \MTLR.
Here, we found that \MTLR's were 
significantly better,
which  is not surprising, given that \PSSP\ uses all of the information known about an individual, while \KM\ uses only the few features mentioned above.
% (here, just cancer type and stage, and perhaps gender).
% Of course, \PSSP's Concordance index is also much better than \KM's, which is simply 0.5.
\comment{
(but not the calibrated probability that a particular patient will die at any specific time).
This information can be used to estimate whether each individual feature (by itself) is important in understanding the survival time of an individual -- 
}
}

\subsection{Other Issues} % 2.6?
\label{sec:OtherIssues}
%As noted above, 
The goal of many Survival {\em Analysis}\ tools is to identify relevant variables,
which is different from our challenge here,
of making a prediction about an individual.
Some researchers use \KM\ to test whether a variable is relevant --
\eg they partition the data into two subsets,
based on the value of that variable,
then run \KM\ on each subset,
and declare that variable to be relevant
if a log-rank test claims these two curves are significantly different~\cite{LogRankTest}.
It is also a common use of the basic Cox model --
in essence, by testing if the $\hat{\beta}_i$ coefficient associated 
with feature $x_i$ 
(in Equation~\ref{eqn:CoxPH}) 
is significantly different from
0~\cite{therneau2013modeling}.
% \add[RG]{?? are you thinking of $e^{\beta_i}$?}\note[HH]{Yep my bad, I changed it back.}
(We will later use this approach 
to select features,
as a pre-processing step, before running the actual survival prediction model; see Section \ref{sec:DataSets}.)

\comment{\annote[RG]{(We will later use this approach to select features,
as a pre-processing step, before running the actual survival prediction model.)}
{Hmmm... this appears in Sec1 -- should it remain there, and also here?  Bret votes remove from section 1}}

Note this ``{\em g vs i}\,'' distinction is not always crisp, 
as it depends on how many variables are involved
-- \eg models that ``describe'' each instance using no variables (like \KM)
are clearly ``{\em g}\,'',
while models that use dozens or more variables,
enough to distinguish % individuate  single
each  patient from one another, 
are clearly ``{\em i}\,''.
But models that involve 2 or 3 variables typically will 
place each patient into one of a small number of ``clusters'',
and then assign the same values to each member of a cluster.
By convention, 
we will catalog those models as ``{\em g}\,''
as the decision is 
not intended to be at an individual level.

The ``$\OneOne$'' vs ``$\infty$'' distinction can be blurry,
if considering a system that produces a small number $k>1$ 
of predictions for each individual --
\eg the Gail model provides a prediction of both 5 year and 25 year survival.
We consider this system as a pair of ``$\OneOne$''-predictors, as those two models are different.  
(Technically, we could view them as ``Gail[5year]'' versus ``Gail[25year]'' models.)

Finally, recall 
there are two types of frameworks that each return a single value for each instance:
% that the ``$\One$'' in  ``$\One$ vs $\infty$''
% ``\SAc{P}{1}{i}'' and in ``\SAc{R}{1}{g}'' and ``\SAc{P}{1}{i}''
% means these models each return a single value.
% However, as noted above,
the single value returned by the \SAc{R}{\OneAll}{i}-model \CPH\ 
is {\em atemporal}\ -- \ie applies to the overall model -- 
while each single value returned by 
the \SAc{P}{\OneOne}{i}-model Gail 
and the \SAc{R}{\OneOne}{g}-model PaP,
is for a specific time, $t^*$.
(Note there can also be 
%\SAc{R}{\OneOne}{i}-models that are 
%with respect to a single time
%and
\SAc{P}{\OneAll}{i}- and \SAc{R}{\OneAll}{g}-models that are atemporal.) 
\comment{
\note[RG]{Hmmm... there could be a single survival quantity -- 
which is the probability of the event ever happening 
-- eg with respect to loans: many people will NOT default. Or if there are competing risks? 
Humza votes no to bringing up this tangent since 
it's more difficult to come up with medical applications}}

\subsection{Relationship of Distributional % \ISD\
Models to Other Survival Analysis Systems}
\label{sec:Relate}
% \note[RG]{System or Model or Framework or ..?}\note[BH]{"Framework" = [P/R,1/inf,i/g] designation, "Model"=algorithm, "Predictor" = model trained on a particular dataset}
We will use the term ``Distributional Model'' to refer to algorithms within the 
\SAc{P}{\infty}{g}\ % \cl{4}\ 
and 
\SAc{P}{\infty}{i}\ frameworks % \cl{5}\ 
-- \ie both \KM\ and \ISD\ models.
Note that such models can match the functionality
of the first 3 ``personalized'' approaches.
First, 
to emulate 
\SAc{P}{\OneOne}{i}, % \cl{3},
we just need to evaluate the % \ISD\
distribution at the specified single time $\tzero$
-- \ie $\estCP{}{\tzero}{\inst{}}$.
So for Patient \#1  (from Figure~\ref{fig:StomCan}), for
$\tzero=\,$``48 months'',
this would be 20\%. % , associated with 48 months.
Second, to emulate \SAc{R}{\OneOne}{i},
% (where the ``1'' refers to a single time point),
we can just use the negative of this value as the time-dependent risk score
-- so the 4-year risk for Patient \#1 would be -0.20.
Third, to deal with % the other 
\SAc{R}{\OneAll}{i}, 
% (where the ``1'' refers to an atemporal value),
we need to reduce the %\ISD\
distribution to a single real number, 
where larger values indicate shorter survival times.
% \note[RG]{HH: are we f'sure using median? If so, this text should be reversed: start with median, then }
A simple candidate is 
the individual distribution's median value, 
which is where the survival curve crosses 50\%.%
\footnote{
Another candidate is the mean value of the distribution,
which corresponds to the area under the survival curve;
see Theorem~\ref{thm:conditionalKM}.
% written as  $% \hbox{mean}
% E[\,\estCP{}{\cdot}{\inst{i}}\,] = \mean{i}$.
}
So for Patient \#1 in Figure~\ref{fig:StomCan}, 
the median is 
$\medi{1}\ =\ $16 months.
We can then view (the negative of) this scalar as the risk score for that patient.
So for Patient \#1,
the ``risk'' would be $r(\inst{1})\ =\ -16$ .
\comment{
A simple candidate is the individual distribution's mean value 
(which corresponds to the area under the survival curve;
see Theorem~\ref{thm:conditionalKM}),
written as 
$% \hbox{mean}
E[\,\estCP{}{\cdot}{\inst{i}}\,] = \mean{i}$.
% $\mean{\inst{}}$.
We can then view (the negative of) this scalar as the risk score for a patient, \(\inst{i}\).
So for Patient \#1,
the ``risk'' would be $r(\inst{1})\ =\ -\mean{1}$ = -22.97 months.%
\footnote{
Another candidate is the median value of the distribution,
which is where the survival curve crosses 50\%.
So for Patient 1 in Figure~\ref{fig:StomCan}, 
the median is 16 months.
}
}
Fourth, to view the \ISD\ model in the \SAc{R}{\OneAll}{g}\ framework,
we need to place the patients into a small number of 
``relatively homogeneous'' bins.
Here, we could quantize the (predicted) mean value
-- \eg
mapping a patient to Bin\#1 if that mean is in [0,\,15),
Bin\#2 if in [15,\,27), and Bin\#3 if in [27,\,70].
(Here, this patient would be assigned to Bin\#2.)
Fifth, to view the \ISD\ model in the \SAc{R}{\OneOne}{g}\ framework,
associated with a time $\tzero$,
we could quantize the $\tzero$-probability % =\,$``48 months'',
-- \eg quantize the 
 $\estCP{}{\tzero = \hbox{48 months}}{\inst{}}$
into 4 bins corresponding to the intervals
[0,\,0.20),
[0.20,\,0.57), 
[0.57,\,0.83],
and
[0.83,\,1.0].

These simple arguments show that a distributional % \ISD\
model can produce the scalars used by five other frameworks
\SAc{P}{\OneOne}{i},
\SAc{R}{\OneOne}{i},
\SAc{R}{\OneAll}{i},
\SAc{R}{\OneAll}{g},
and \SAc{R}{\OneOne}{g}.
Of course, a distributional model can also provide other information about the patient
-- % An \ISD\ extends their functionality by
not just the probability associated with one or two time points, 
but at essentially any time in the future,
as well as the mean/median value.
Another advantage of having such survival curves is {\em visualization}\,
(see Figure~\ref{fig:StomCan}): 
it allows the user (patient or clinician) to see the {\em shape}\ of the curve,
which provides more information than simply knowing the median, or the chance of surviving 5 years, etc.

% \note[RG]{Where does this go? What are other issues? }
There are some subtle issued related to producing meaningful survival curves
-- \eg many curves end at a non-zero value:
note the \KM\ curve in Figure~\ref{fig:AllCurves}(top left)
stops at (83, 0.12), rather than continue to intersect the x-axis at, perhaps (103, 0.0).
This is true for many of the curves produced by
the \ISD{}s.
Indeed, some of the curves do not even cross $y=0.5$,
which means the median time is not well-defined;
{\em cf.}\ 
the top orange line on the \AFT\ curve (top right),
which stops at (83, 0.65),
as well as % and for
many of the other curves throughout that figure.
This causes many problems, 
in both interpreting and evaluating \ISD\ models.
Appendix~\ref{app:SC-to-0} shows how we address this.

\section{Measures for Evaluating Survival Analysis/Prediction Models} %  3
\label{sec:Eval} 

The previous section mentioned 5 ways to evaluate a survival analysis/prediction model:
Concordance, 1-Calibration, Brier score, L1-loss, and D-Calibration.
% see especially the final column of Table~\ref{tab:CompareTools}.
This section will describe these -- quickly summarizing the first four (standard) evaluation measures
(and leaving the details, including discussion of censoring, for Appendix~\ref{app:Evaluation})
then providing a more thorough motivation and description of the fifth, D-Calibration.
The next section shows how the 6 distribution-learning % \ISD-ish 
models perform with respect to these evaluations.

For notation, we will assume models were trained on a training dataset, formed from the same triples as shown in Equation~\ref{eqn:SurData},
% \ =\ \{\, [\inst{i},\, \event{i},\, \cbit{i}]\, \}_i$,
that is $D = D_U \cup D_C$ where 
$D_U\ =\ \{\, [\inst{j}, \death{j}, \delta_j = 1]\,\}_j $
is the set of {\em uncensored}\ instances (notice the event time, \(t_j\), here is written as 
$\death{j}$), and
$D_C\ =\ \{\, [\inst{k}, \censor{k}, \delta_k = 0]\, \}_k $
is the set of {\em censored}\ instances (\(t_k\), here is written as 
$\censor{k}$).
Note also that this training dataset $D$ is disjoint from the validation dataset, $V$. Since models are evaluated on \(V\) and we save discussion of censoring for 
Appendix~\ref{app:Evaluation},
we assume here that all of $V$ is uncensored
-- \ie \(V\ =\ V_U\ =\ \{ \,[\inst{j},\, \death{j},\, \delta_j = 1]\, \}_j
\ \approx\ \{ \,[\inst{j},\, \death{j}]\, \}_j\)
(to simplify notation).

\subsection{Concordance} % 3.1
\label{sec:Concordance}

As noted above, each individual risk model 
\SAc{R}{\One}{-}\ 
(\ie \SAc{R}{\One}{i}\ or \SAc{R}{\One}{g},
where $\One$ can be either $\OneOne$ or $\OneAll$)
assigns to each individual $\inst{}$,
a ``risk score'' $r(\inst{}) \in \Re$, where $r(\inst{a}) > r(\inst{b})$ means the model is predicting that $\inst{a}$ will die before $\inst{b}$.
Concordance (a.k.a.~C-statistic, C-index) is commonly used to validate such risk models.
Specifically, Concordance considers each pair of patients, and asks whether the predictor’s values for 
those patients matches what actually happened to them.  
In particular, if the model gives $\inst{a}$ a higher score than $\inst{b}$, 
% meaning it predicts that  will live longer than ,
then the model gets 1 point if 
$\inst{a}$ dies before $\inst{b}$.
% this is what actually happens, for this $(\inst{a},\, \inst{b})$ pair.  
If instead $\inst{b}$ died before $\inst{a}$, the model gets 0 points for this pair. 
Concordance computes this for all pairs of {\em comparable}\ patients, and returns the average. 

% \note[RG]{How is this? Earlier version is still as comment...}\note[HH]{Looks good to me.}
When considering only uncensored patients, every pair is comparable, 
which means there are 
${ n \choose 2 } = \frac{n\cdot (n-1)}{2}$ pairs from
$n = |V_U|$ elements. 
Given these comparable pairs, Concordance is calculated as,

\begin{equation}
\reallywidehat{C}(\, V_U,\, r(\cdot)\,)\quad =\quad \frac{1}
% { |V_U| \choose 2 } 
 {\frac{|V_U|\ \cdot\ (|V_U|-1)}{2}}
\cdot
\sum_{[\vec{x}_i, d_i] \in V_U} \,
\sum_{
[\vec{x}_j, d_j] \in V_U\,:\ d_i < d_j} 
\Indic{r(\vec{x}_i) > r(\vec{x}_j)} \ .
\label{eqn:Concordance}
\end{equation}
As an example, consider the table of death times $d_i$ and risk scores,
for 5 patients,
shown in Table~\ref{tab:Concord-ex}[left]. Table~\ref{tab:Concord-ex}[right] shows that these risk scores
are correct in 7 of the ${5 \choose 2} = 10$ pairs,
so the Concordance here is 7/10 = 0.7. 

\begin{table}[t]   \centering
\begin{tabular}{>{\columncolor{LightCyan}}c|c|c}
	\rowcolor{gray!50} Id & $d_i$ & Risk$_i$\\
\hline
1& 1 & 6\\
2& 3 & 3 \\
3& 4 & 5 \\
4& 6 & 2 \\
5& 9 & 4 \\
 \end{tabular}
 \qquad\qquad
 \begin{tabular}{>{\columncolor{LightCyan}}r|ccccc}
 	\rowcolor{gray!50}   &1 & 2 & 3 & 4 & 5\\
   \hline
 1 & \\
 2 & +\\
 3 & + & 0\\
 4 & + & + & +\\
 5 & + & 0 & + & 0\\
 \end{tabular}
 \caption{ \label{tab:Concord-ex}
 Simple example to illustrate Concordance (here, with only uncensored patients).
 Left: time of death, and risk score, for 5 patients.
 Right: ``+'' means the row-patient had a lower risk, and died after, the column-patient; otherwise ``0''.}
\end{table}

 This Concordance measure is very similar to 
the area under the receiver operating curve (AUC) and 
equivalent when \(d_i\) is constrained to
values \{0, 1\}~\cite{li2016multi}.
\comment{
\note[RG]{I don't think we need probabilistic interpretation ...
it is now in a comment...}
Given two patients, \(\inst{a}\) and \(\inst{b}\), 
% we have that \(\reallywidehat{C}\) can be interpreted as 
Concordance is 
the probability that \(r(\inst{a}) > r(\inst{b})\) given \(\inst{a}\) dies before \(\inst{b}\) -- \ie 
\begin{equation}
\reallywidehat{C}(??)\quad =\quad \Pr\left[\ r(\inst{a}) > r(\inst{b}) \ |\  d_a < d_b\ \right].
\end{equation} \note[BH]{Is this equivalent to $P(d_a < d_b \,|\, r_a > r_b)$ ?}

From the probabilistic interpretation one should recognize the similarity between Concordance and 
the area under the receiver operating curve (AUC) 
from binary classification. 
In fact, if \(d_i\) is binary, then Concordance is exactly equivalent to AUC ~\cite{heller2016estimating}. 
}

This Concordance measure is relevant when the goal is to {\em rank}\ or \textit{discriminate} between patients
-- \eg when one wants to know who will live longer between a pair of patients. 
(For example, if we want to transplant an available liver to the patient 
who will die first -- this corresponds to ``urgency''.)
Concordance is the desired metric here due to it's interpretation, \textit{i.e.} given two \textit{randomly selected} patients, \(\inst{a}\) and \(\inst{b}\), if a model with Concordance of 0.9 
assigns a higher risk score to \(\inst{a}\) than \(\inst{b}\),
then there is a 90\% chance that \(\inst{a}\) will die before \(\inst{b}\).

While \SAc{R}{\OneAll}{i}\ models (such as \CPH) 
provide a risk score that is independent of time,
there are also  
% \SAc{R}{\OneOne}{i}\ models that provide a risk score for an instance that  relates only to a specified time $\tzero$,  and 
 \SAc{R}{\infty}{i}\ models that 
 produces a risk score $r(\inst{}, t)$ 
 for an instance $\inst{}$ that depends on time $t$;
 such as Aalen's additive regression model \cite{aalen2008survival} or
 time-dependent Cox (td-Cox) \cite{fisher1999time},
 which uses time-dependent features. These models can be evaluated using {\em time-dependent Concordance}
(aka, ``time-dependent ROC curve analysis'') \cite{heagerty2005survival}.

Finally, the \SAc{R}{-}{g}\ systems  compute a risk score,
but then bin these scores into a small set of intervals.
When computing Concordance,
they then only consider patients in different bins.
For example, if Bin1 = [0, 10] and Bin2 = [11, 20],
then this evaluation would only consider pairs of patients $(\inst{a}, \inst{b})$
where one is in Bin1 and the other is in Bin2 
-- \eg $r(\inst{a}) \in [0,\, 10]$ and $r(\inst{b}) \in [11,\, 20]$.
(Hence, it will not consider the pair $(\inst{c}, \ \inst{d})$ if both
$r(\inst{c}),\ r(\inst{d}) \in [11,\, 20]$.)

See Appendix~\ref{app:Concordance} for more details,
including a discussion of how this measure deals with censored instances and tied risk scores/death times.

\subsection{L1-loss} % 3.2
\label{sec:L1-loss}

Survival prediction is very similar to regression: 
given a description of a patient, predict a real number 
(his/her time of death).
With this similarity in mind, one can evaluate a survival model 
using the techniques used to evaluate regression tasks,
such as L1-loss -- 
the average absolute value of the difference between 
the true time of death, $d_i$,
and the predicted time $\reallywidehat{d}_i$: % , yielding
$\frac{1}{n}\sum_i | d_i  - \reallywidehat{d}_i|$. 
(We consider the L1-loss, rather than L2-loss which squares the differences,
as the distribution of survival times is often right skewed,
and L1-loss is less swayed by outliers than L2-loss.)
% which squares the differences.

One challenge in applying this measure to
our \SAc{P}{\infty}{-}\ models
is identifying 
 the predicted time, $\reallywidehat{d}_i$.
Here, we will use the predicted median
% There are a number of ways to estimate \(\reallywidehat{d}_i\), one being the mean%
survival time, that is $\reallywidehat{d}_i = \medi{i}$,
% REMOVE THIS, as already footnote about this... \footnote{Another commonly-used value is the \textit{mean} survival time, $\reallywidehat{d}_i\ =\ E[\,\estCP{}{\cdot}{\inst{i}}\,] = \mean{i}$} 
leading to the following measure:
% Doing so gives us a measure for the L1-loss:
\comment{
\note[RG]{fixed Eq 6, as 
earlier version suggested that $V_U$ is just a set of indices $j$,
and the "input" was a set of $\mu_i$s.}}
\begin{eqnarray}
L1(\ V_U,\ 
 \{\,
\estCP{}{\cdot}{\inst{i}} \, \}_i\,) &=& 
\frac{1}{|V_U|} 
\sum_{[\inst{i}, d_i] \in V_U} % \sum_{j \in V_U} 
   \left|\death{i} - \medi{i}\right| .
    \label{eqn:L1Unc}
\end{eqnarray}

While we would like this value to be small,
we should not expect it to be 0:
if the distribution is meaningful, 
there should be a non-zero chance of dying at other times as well.
For example, while the  L1-loss is 0
for the % (appropriate) 
Heaviside distribution at the time of death
(shown in green in Figure~\ref{fig:L1=0}),
this is unrealistic.
\comment{
Indeed, the only 
% completely-
calibrated 
distribution 
with L1-loss being 0, is the Heaviside distribution, 
shown in green in Figure~\ref{fig:L1=0}.
% \note[RG]{Why ``completely-calibrated''? Isn't this the ONLY distribution with that property??}
\note[BH]{Here when we say "calibrated", do we mean "perfectly 1-calibrated at all timepoints for all possible bin-sizes"?  Should we be using the term "calibrated" before the next section?  If we meant D-calibrated or relax the "for all timepoints" or "all possible bin-sizes" restrictions, I think I can come up with alternative examples of curves that would give 0 L1-loss and perfect calibration}
\comment{Example for D-calibration: Suppose you have 100 (all uncensored) patients with \(t_1 > t_2 > ... > t_{100}\).  Let \(S_1\) be a heaviside dist with the drop at \(t_1\).  Let \(S_2\) be a line from (0,1) to (\(t_2\), 0.99) connected to a line from (\(t_2\), 0.99) to (\(t_{2b}\),0) where \(t_{2b}\) is found such that the mean of \(S_2\) is \(t_2\).  Similarly \(S_3\) passes through (\(t_3\), 0.98), \(S_4\) passes through (\(t_4\), 0.97), ..., \(S_{100}\) passes through (\(t_{100}\), 0.01).  This construction has 0 L1-loss (by the selection of each \(t_{ib}\)) and has perfect D-calib due to 10 patients being in [0.9,1], 10 in [0.8,0.9), ..., 10 in [0,0.1).}
}

\begin{figure}\centering
\includegraphics[width=0.7\textwidth,
height=1.0in]{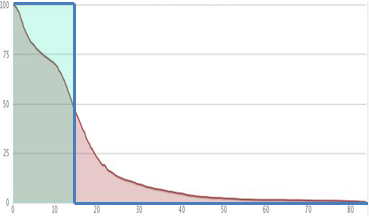}
\caption{% (top) Standard survival curve; (bottom)
Example of a % \MTLR\ 
survival curve (in red), 
superimposed (in green) with a degenerate curve 
that puts all of its weight on a single time point 
(which means it assigns 100\% chance of dying at exactly this time).
}
\label{fig:L1=0}
\end{figure}

Appendix~\ref{app:L1-loss} discusses many issues with the L1-loss measure,
related to censored data,
and reasons to consider using the $\log$ of survival time.

\subsection{1-Calibration} % 3.3
\label{sec:1-Calib}
\comment{
\note[RG]{Should we re-order so 1-Calibration is later -- and use that to segue to D-calibration?
I.e., make this 3.4, then D-calib 3.5
... I moved L1-loss earlier...}
\note[HH]{I think we could move L1 up but I think it is beneficial to having 1-Calibration before Brier score -- otherwise we have to introduce the concept of calibration within the Brier score section. Doing this would require new motivation for the Brier score-- right now we are using the fact that p-values from 1-Cal. can't be ranked to motivate the use of Brier score.}
}
The \SAc{P}{\OneOne}{i}\ tools estimate the survival probability $\estCP{}{\tzero}{\inst{}} \in [0,1]$
for each instance $\inst{}$, at a single time point $\tzero$.
For example, the PredictDepression system~\cite{wang2014prediction}
predicts the chance that a patient
will have a major depression episode within the next 4 years,
based on their current characteristics
-- \ie this tool produces a single probability value 
$%\estP{i}{\hbox{4yr}} \ =\ 
\hat{S}(\,4\text{yr}\,|\,\inst{i}\,)\ \in [0,1]$ for each patient described
as $\inst{i}$.
We can use 1-Calibration to measure the effectiveness of such predictors.
To help explain this measure,
consider the ``weatherman task'' of predicting, on day $t$, 
whether it will rain on day $t+1$.
Given the uncertainty, forecasters provide probabilities.
Imagine, for example, there were 10 times that the weatherman, Mr.W, predicted 
that there was a 30\% chance that it would rain tomorrow.
Here, if Mr.W was calibrated, we expect that it would rain 
3 of these 10 times -- \ie 30\%.
Similarly, of the 20 times Mr.W claims that there is an 80\% chance of rain tomorrow,
we expect rain to occur 16 = 20 $\times$ 0.8  of the 20 times.

Here, we have described a 
binary
 probabilistic prediction problem
-- \ie predicting the chance that it will rain the next day. 
One of the most common calibration measures for such binary prediction problems is 
the Hosmer-Lemeshow goodness-of-fit test~\cite{hosmer1980goodness}. 
First, we sort the predicted probabilities 
for this time $\tzero$ for all patients $\{\,\estCP{}{\tzero}{\inst{i}}\,\}_i$
and 
group them into a number ($B$) of ``bins''; %  (categories);
commonly into deciles, \ie \(B = 10\) bins. 
% \note[RG]{HH: does this look right?}\note[BH]{I think it's more concise to skip the patient renumbering} \note[RG]{Need to mention re-numbering!}
\comment{
(So if there are 200 patients, renumber the patients to match their
$\estCP{}{\tzero}{\inst{i}}$ values -- hence, 
$\estCP{}{\tzero}{\inst{1}} \geq \estCP{}{\tzero}{\inst{2}} \geq \cdots
\geq \estCP{}{\tzero}{\inst{200}}$.
Here, the first bin would include the 20 patients with the largest
$\estCP{}{\tzero}{\inst{i}}$ values 
-- hence $\{\inst{1}, \dots, \inst{20}\}$ --
and the second bin, the patients with the next highest set of values
(hence $\{\inst{21}, \dots, \inst{40}\}$), and so forth,
for all 10 bins.
Next, {\em within each bin},
we calculate the expected number of events $\bar{p_j}$ --
here $\bar{p_1} = \frac{1}{20} \sum_{i=1}^{20} (1- \estCP{}{\tzero}{\inst{i}})$,
then $\bar{p_2} = \frac{1}{20} \sum_{i=21}^{40}(1-\estCP{}{\tzero}{\inst{i}})$,
etc.
}
Suppose there are 200 patients;
the first bin would include the 20 patients with the largest
$\estCP{}{\tzero}{\inst{i}}$ values,
the second bin would contain the patients with the next highest set of values, and so on,
for all 10 bins.
Next, {\em within each bin},
we calculate the expected number of events, $\bar{p_j} = \frac{1}{|B_j|} \sum_{\inst{i} \in B_j} (1- \estCP{}{\tzero}{\inst{i}})$.
We also let
$n_j = |B_j|$ be the size of the $j^{th}$ bin (here, $n_1 = n_2 = \dots = n_{10} = 200/10 = 20$),
and 
$O_j$ be the number of patients (in the $j^{th}$ bin) who died before \(t^*\).
Recalling that \(d_i\) denotes Patient \#i's time of death and 
letting  
\(o_i\ =\ \Indic{d_i \leq \tzero}\)
denote the event status of the \(i^{th}\) patient at $\tzero$:
% we say that \(o_i = 1\) if \(d_i \leq \tzero\) or \(o_i = 0\) if \(d_i > \tzero\). 
% Thus,
for the \(j\)th bin, \(B_j\), we have \(O_j = \sum_{\inst{i} \in B_j} o_i \).
% and compared to the observed number of events 
%\annote[RG]{and scaled appropriately,
%such that the sum of these comparisons follows a chi-squared distribution.}{What does this %mean?}
Figure~\ref{fig:1-calib} graphs the 10 values of 
observed $O_j$ and expected $n_j\,\bar{p_j}$ 
for the deciles, for two different tests
(corresponding to two different  \ISD-models,
on the same dataset and $\tzero$ time). Additionally, see Appendix~\ref{app:1-Calib} for an example walking through 1-Calibration.

\begin{figure}[tb]
\includegraphics[width = \textwidth]{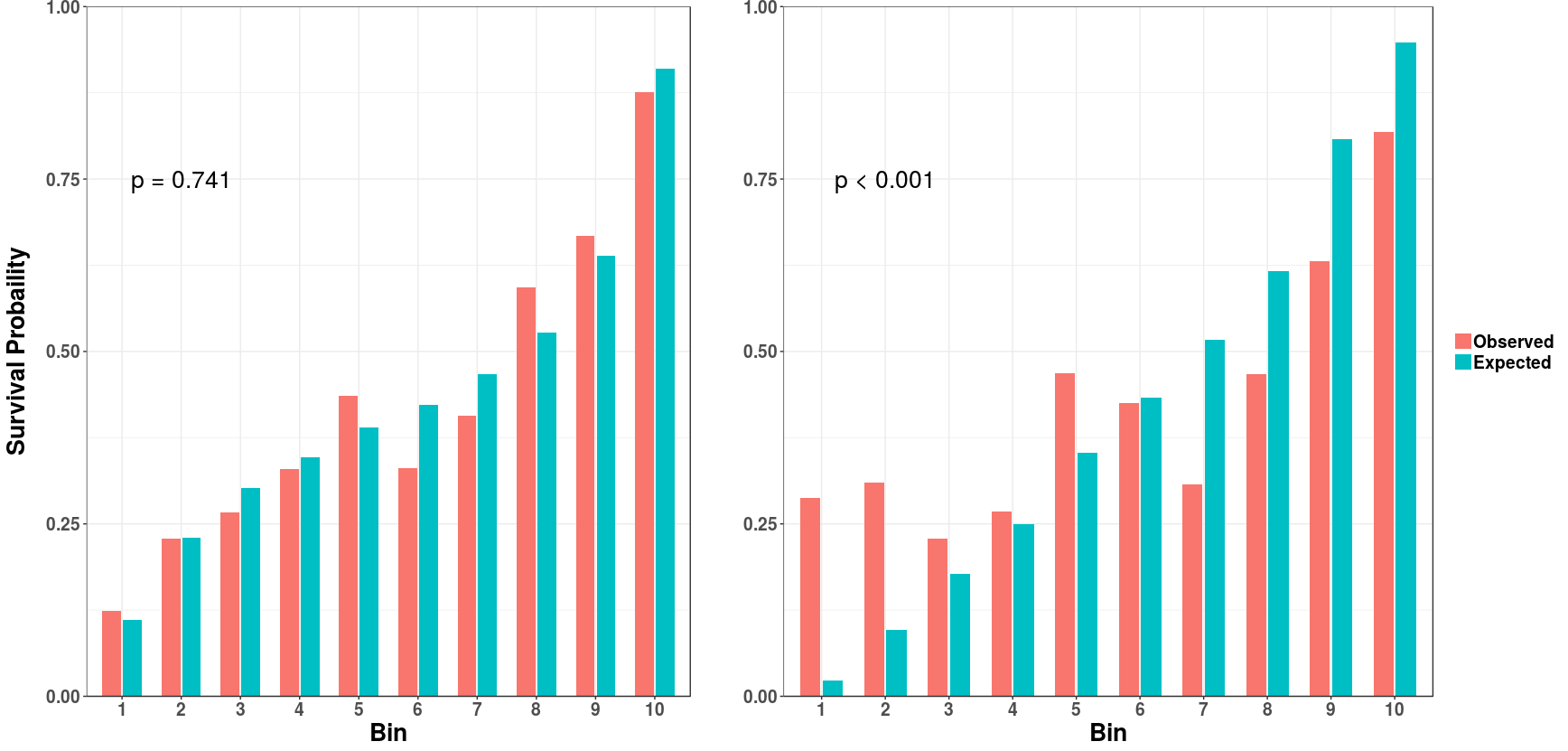}
\caption{ \label{fig:1-calib}
The bin observed and expected probabilities associated with two 1-Calibration computations, for the \MTLR\ [left] model and the \RSFKM\ model applied to the GBM dataset for the 50th percentile of time (369.5 days).}
% \hrule
\end{figure}

For each test, we can then compute % Hence, % Mathematically, given \(B\) bins,
the Hosmer-Lemeshow test statistic: %  is given by
\begin{align}
\label{eqn:hlstat}
\reallywidehat{HL}(\,V_U,\, \estCP{}{\tzero}{\cdot}\,\,)\quad =\quad \sum_{j=1}^B 
\frac{(O_j - n_j\bar{p}_j)^2}{n_j\,\bar{p_j}\,(1-\bar{p_j})},
\end{align}
% where \(O_j\) is the number of observed events, 
% \(n_j\) is the number of patients,  \(\bar{p_j}\) is the average predicted probability, and the subscript \(j\) refers to the \(j^{th}\) of the \(B\) bins. 
If the model is 1-Calibrated, then
this statistic follows a \(\chi^2_{B-2}\) distribution,
which then can be used to find a \(p\)-value.% 
% where \(p < 0.05\) suggests a model is a poor fit to the data -- \ie not (1-)calibrated.
%
\comment{ this appears above ...
Translating this to survival analysis, we choose a specific time point, \(\tzero\), 
and set the predicted probability 
\(p_i \ =\ 1 - \estCP{}{\tzero}{\inst{i}}
% \hat{S}(\tzero|\inst{i})
\).
Recalling that we denote the patient's time of death as \(d_i\) and 
letting  
\( o_i = \Indic{d_i \leq \tzero}\) denote the event status of the \(i^{th}\) patient:
% we say that \(o_i = 1\) if \(d_i \leq \tzero\) or \(o_i = 0\) if \(d_i > \tzero\). 
% Thus,
for the \(j\)th bin, \(B_j\), we have \(O_j = \sum_{i \in B_j} o_i \)
}
For a given time \(\tzero\), finding \(p < 0.05\)
suggests the survival model is \underline{not} well calibrated at \(\tzero\)
-- \ie the predicted probabilities of survival at \(\tzero\) may not be representative of patient's true survival probability at \(\tzero\).
\comment{ For interpretive purposes,
given a time \(\tzero\), and a \(p\)-value such that \(p < 0.05\), 
this test would suggest the survival model is 
{\em not}\ % \underline{not} 
well calibrated at \(\tzero\), 
that is, the predicted probabilities of survival at \(\tzero\) may not be representative of patient's true survival probability at \(\tzero\).}

Returning to Figure~\ref{fig:1-calib}, the HL statistics are 5.99 and 321.44, for the left and right,
leading to the $p$-values $p=$0.741 and $p < \,$0.001 
-- meaning the left one passes but the right one does not.
(This is not surprising, given that
each pair of bars on the left are roughly the same height, 
while the pairs of the right are not.)

% One distinction to clarify is \SAc{P}{\infty}{i} models can give probabilities for multiple time points so a \SAc{P}{\infty}{i} model may not be calibrated at \(\tzero\) but may be calibrated at some other time point since, \(O_j, n_j\), and \(\bar{p_j}\) are dependent on the chosen time point. This distinction between a model which is calibrated at a single time point versus across a distribution of time points is explored further by D-Calibration in Section~\ref{sec:D-Cal+1-Cal}. Additionally, details and a discussion concerning the handling of censored patients is given in Appendix~\ref{app:1-Calib}.

Note that a \SAc{P}{\infty}{i} model,
which gives probabilities for multiple time points,
may be calibrated at one time \(t_1\),
but not be calibrated at another time $t_2$,
since
\(O_j\), and \(\bar{p_j}\) are dependent on the chosen time point. 
This issue motivated us to define a notion of calibration across
a distribution of time points,
D-Calibration, in Section~\ref{sec:d-calibration}. % sec:D-Cal+1-Cal 
Appendix~\ref{app:1-Calib} provides further details about 1-Calibration
including ways to handle censored patients.

\subsection{Brier Score} % 3.3
\label{sec:BrierScore}
We often want a model to be both 
% \annote[RG]
{discriminative (high Concordance) and 
calibrated (passes the 1-Calibration test).}
% {Will the reader understand these terms?} \note[HH]{We have already introduced the ideas of discrimination and calibration  in the Concordance sections and 1-Calibration sections so I believe they should.}
While one can rank Concordance scores to compare two models' discriminative abilities, 1-Calibration cannot rank models besides suggesting one model is calibrated (\(p \geq 0.05)\) and one is not \((p < 0.05)\)
(as 
% Note that choosing the model with the largest \(p\)-value is inappropriate as 
\(p\)-values are not intended to be ranked). 
% Originally developed as a tool to compare probabilistic forecasts  
The Brier score~\cite{brier}
is a commonly used metric that
measures both calibration and discrimination;
see Appendix~\ref{app:BS-decomp}.
% {What exactly does this mean?} \note[HH]{The Brier Score can be broken into two pieces, a calibration and a discrimination component. The sum of these two components gives the overall Brier score and so the literature says that Brier Score is measuring both of these components simultaneously. There are details in Appendix B.4.1 showing the two components -- have you looked at this and are still confused or have you not seen this yet?}
Mathematically, the Brier score is the mean squared error between the 
\{0, 1\} event status at time \(\tzero\) and 
the predicted survival probability at \(\tzero\). 
Given a fully uncensored validation set \(V_U\), 
the Brier score, at time $\tzero$, is % given by 
\begin{align}
\label{eqn:brier}
BS_{\tzero}\left(\, V_U, \ 
\estCP{}{\tzero}{\cdot} 
% \hat{S}(\tzero| \cdot)
\,\right)
\quad=\quad \frac{1}{|V_U|} 
\sum_{[\inst{i}, d_i] \in V_U} 
\left(\ \Indic{d_i \leq \tzero}\ -\
\estCP{}{\tzero}{\inst{i}}\ \right)^2.
\end{align}
Here, % From Equation (\ref{eqn:brier}), one can see that 
a perfect model (that only predicts 1s and 0s as survival probabilities and is correct in every case) 
will get the perfect score of 0,
whereas a reference model 
that gives 
$\estCP{}{\tzero}{\cdot}  = 0.5$
% \(\hat{S}(\tzero|\inst{}) = 0.5\) 
for all patients
will get a score of 0.25. 

\def\IBS{\hbox{IBS}}

An extension of the Brier score to an interval of time points is the \textit{Integrated} Brier score%
%(\IBS)
, which will give an average Brier score across a time interval,%
\begin{equation}
\IBS(\, \tau% \tzero
,\,  
V_U,\,
\estCP{}{\cdot}{\cdot}\,  % \hat{S}(\cdot| \cdot),
% \{ \inst{i} \}_i\, 
)\quad =\quad \frac{1}{\tau% \tzero
}\int_0^{\tau} % \tzero}
 BS_t\left(\,V_U,\,
 \estCP{}{t}{\cdot}\, \right) dt\ .
\label{eqn:IBS}    
\end{equation}
We will use this measure for our analysis,
where $\tau$ % \tzero$
is the maximum event time of the combined training and validation datasets -- 
this way,
the interval evaluated is equivalent across cross-validation folds.

As noted above, 
the Brier score measures both calibration and discrimination,
implying it should be used when 
seeking a model that 
must % is expected to 
perform well on 
both calibration and discrimination,
or when one is investigating the overall performance of survival models. 
Appendix~\ref{app:brier} shows how to incorporate censoring into the Brier score, and discusses the decomposition of the Brier score into calibration and discriminative components.
% \add[HH]{One weakness of IBS is that it does not ensure the full distribution of probability values is utilized (see Appendix \hbox{\ref{app:IBSWeakness})}, which we address below through D-Calibration.} 
% \note[HH]{Remove the previous sentence and stick with the footnote?}

\subsection{D-Calibration} % Is the Model ``D-Calibrated''? % 3.4
\label{sec:d-calibration}

The previous sections summarized several 
%fairly standard
common
ways to evaluate 
%fairly 
standard survival prediction models, % approaches.
that produce only a single value 
for each patient
-- \eg the patient's risk score, 
perhaps with respect to a single time,
% a single probability value (associated with a specific time),
or the mean survival time.
(Each is a
% These are very appropriate for tools that compute only a single value, such as the ones in the 
\SAc{-}{\One}{-}\ model.)
However, the \SAc{P}{\infty}{-}\ tools produce a distribution
-- \ie each is a function that maps $[0, \infty]$ to $[0,1]$  (with some constraints of course),
such as the ones shown in Figure~\ref{fig:AllCurves};
see Footnote~\ref{ftnote:SC}.
% (See the explanation of survival curves in Footnote~\ref{ftnote:SC} above.) 
It would be useful to have a measure that examines
the entire distribution as a distribution.%
\footnote{While the Integrated Brier score does consider % measures
all the points across the distribution, 
it % IBS
simply views that distribution as a set of $(x,y)$ points;
see % does not ensure the entire distribution is utilized. See
Appendix~\ref{app:IBSWeakness} for further explanation.
}

\comment{
Our goal is a tool that can help patients accurately assess their remaining time, 
based on all information available about that patient, acquired at the first meeting (\eg at the time of diagnosis). 
While we cannot know the actual time, we can nevertheless produce a distribution, such as the ones shown in Figures~\ref{fig:StomCan},~\ref{fig:allCurves} and~\ref{fig:Cox}.
(See the explanation of survival curves in Footnote~\ref{ftnote:SC} above.)
}
 \def\estPinv#1#2{\hat{p}^{-1}_{#1}(\, #2\,)}
 \def\DataInt#1#2#3{#1_#2(\, #3\,)} % so D_{\Theta}( [0, 0.1] )

A distributional calibration (D-Calibration)~\cite{andres2018novel} measure addresses the critical question:
\begin{equation}
\hbox{\small\em Should the patient believe the predictions implied by the survival curve? }
\label{eqn:BelievePSSP}
\end{equation}
First, consider population-based models \SAc{P}{\infty}{g}, like Kaplan-Meier curves --
\eg Figure~\ref{fig:StomCan}[left], 
for patients with stage-4 stomach cancer.
If a patient has stage-4 stomach cancer, 
should s/he believe 
that his/her median survival time is 11~months,
and that s/he has a 75\% chance of surviving more than 4~months?
To test this,
we could take 1000 new patients (with stage-4 stomach cancer)
and ask whether $\approx$500 of these patients 
lived at least 11 months, and if $\approx$750 lived more than 4 months.

\comment{ Don't need this!
In general, let $\estPinv{}{\cdot}$ be the inverse of the survival function $\estP{}{\cdot}$
-- \ie for any probability $\rho \in [0,1]$,
$\estPinv{}{\rho} = \tau$ iff $\estP{}{\tau} = \rho$.
Hence $\estPinv{}{0.5}$ is the median survival time, which is 11months for Figure~\ref{fig:StomCan}[left].

For any time $\tau$, let $D_{\tau}$ be the instances who died before $\tau$.
We say our \SAc{P}{\infty}{g}\ model $\Theta$ is good (D-Calibrated)
if, for every time $\tau$,
$\estP{\Theta}{\tau}$ corresponds to the ratio of $D_{\tau}$ to $D$
-- \eg as $\estP{\Theta}{11months} = 0.5$,
$D_{11month}$ should be 1/2 of $D$ and 
$D_{4month}$ should contain $\estP{\Theta}{4months} = 0.75$
of $D$.
That is 
\begin{equation}
\frac{|D_{\tau}|}{|D|}\quad=\quad \estP{\Theta}{\tau}
\qquad \hbox{that is}\qquad
\frac{| D_{\estPinv{\Theta}{\rho}} |}{|D|}\quad=\quad \rho
\end{equation}
}

For notation, given a dataset, $D$, and \SAc{P}{\infty}{g}-model $\Theta$,
and any interval $[a,b] \subset [0,1]$,
let 
\begin{equation}
\DataInt{D}{\Theta}{[a,b]} \quad=\quad
 \{\, [\inst{i},\death{i},\delta = 1] \in D \ |\ \estP{\Theta}{\death{i}} \in [a,b]
 \,\}
 \label{eqn:g-DataInt}
\end{equation}
 be the subset of (uncensored) patients in $D$ whose time of death
 is assigned a probability (by $\Theta$) % $\estP{\Theta}{\death{i}}$
 in the interval $[a,b]$.
 For example, $\DataInt{D}{\Theta}{[0.5,\,1.0]}$ is the subset of patients 
who lived at least the median survival time
(using $\estP{\Theta}{\cdot}$'s median),
and 
$\DataInt{D}{\Theta}{[0.25,\,1.0]}$
is the subset who died after the 25th percentile of $\estP{\Theta}{\cdot}$.
By the argument above, we expect 
$\DataInt{D}{\Theta}{[0,\,0.5]}$ to contain about 1/2 of $D$,
and $\DataInt{D}{\Theta}{[0.25,\,1.0]}$ to contain about 3/4 of $D$.
Indeed, for any interval $[a,\,1.0]$,
we expect 
\begin{equation}
\frac{|\DataInt{D}{\Theta}{[a,\,1.0]}|}{|D|}\quad =\quad 1-a
\end{equation}
or in general
\begin{equation}
\frac{|\DataInt{D}{\Theta}{[a,\,b]}|}{|D|}\quad =\quad b-a
\end{equation}

This leads to the idea of a survival distribution 
 \SAc{P}{\infty}{g} model, $\Theta$,
being D-Calibrated:
For each uncensored patient $\inst{i}$,
we can observe when s/he died $\death{i}$,
and also determine the percentile for that time, based on $\Theta$:
 $\estP{\Theta}{\death{i}}$. 
If $\Theta$ is D-Calibrated,
% we expect roughly 50\% of the patients to die in the [50\%, 100\%] interval
% -- \ie each before his/her {\em predicted median time}, and 50\% afterwards -- \ie in [0\%,\,50\%].
% We similarly
we expect roughly 10\% of the patients to die in the [90\%, 100\%] interval
-- \ie $\frac{|\DataInt{D}{\Theta}{[0.9,\,1.0]}|}{|D|}\ \approx\
 1 - 0.9 = 0.1$
--
and another 10\% to die in the [80\%, 90\%) interval, and so forth for each of the  10 different 10\%-intervals. 
More precisely, 
the set $\{\,\estP{i}{\death{i}}\,\}$ over all of the patients
should be distributed uniformly on $[0,1]$,
which means
% \change[SD]{, in particular,}{} 
that each of the 10 bins would contain 10\% of \(D\).

This suggests a measure to evaluate a % {\em learned}\ 
distributional model: 
see how close each of these 10 bins is to the expected 10\%.
We therefore use Pearson's \(\chi^2\) test:
compute the $\chi^2$-statistic with respect to the ten 10\% intervals,
and ask whether the bins appear uniform, at (say) the $p>\,$0.05 level. 
Theorem~\ref{thm:DCalGOF} (in Appendix~\ref{app:D-Calib})
states and proves the appropriateness of the 
{Pearson's $\chi^2$} {goodness-of-fit}
test.
\comment{
The proof detailing the appropriateness of the
{Pearson's $\chi^2$} {goodness-of-fit}
test is given by Theorem \ref{thm:DCalGOF}
in Appendix~\ref{app:D-Calib}.
}

This addresses the question posed at the start of this subsection
(Equation% Question
~\ref{eqn:BelievePSSP}):
\begin{eqnarray*}
&&\hbox{Yes, % given that this goodness-of-fit test reports $p>0.05$,
a patient should believe the prediction from the survival curve}\\
&&\hbox{whenever this goodness-of-fit test reports $p>0.05$.}
\end{eqnarray*}

\subsubsection{Dealing with {\em Individual} Survival Distributions, \ISD} % 3.4.1
Everything above was for a {\em population}-based distributional model \SAc{P}{\infty}{g}.
These specific results do not apply to {\em individual}\ survival distributions
\SAc{P}{\infty}{i}:
For example, consider a single patient,  Patient \#1,
whose curve is shown in Figure~\ref{fig:StomCan}[middle].
Should he believe this plot,
which implies that his median survival time is 18~months,
and that he has a 75\% chance of surviving more than 13~months?

If we could observe 1000 patients exactly identical to this Patient \#1, 
we could verify this claim by seeing their actual survival times:
this survival curve is meaningful if its predictions matched the outcomes of those  
copies
-- \eg if around 250 died in the first 13 months,
another $\approx$250 in months 13 to 18, etc.

Unfortunately, however, we do not have 1000 ``copies'' of Patient \#1.
But here we do have many % $>$1000 
other patients, 
each with his/her own characteristic survival curve, 
including the 4 curves shown in Figure~\ref{fig:4subjects}.
Notice each patient has his/her own distribution, and hence his/her own quartiles -- 
\eg the predicted median survival times 
for Patient A\ (resp., B, C and D),
are 28.6 (resp., 65.7, 11.4, and 13.9) months;
see Table~\ref{tab:4patients}.
For these historical patients, we know the actual event time for each.%
\footnote{
Here we just consider uncensored patients;
Appendix~\ref{app:D-Calib} extends this to deal with censoring.}  
Here, if our predictor is working correctly, 
we would expect that 2 of these 4 would pass away before respective median times, 
and the other 2 after their median times.
Indeed, we would actually expect 1 to die in each of the 4 quartiles;
the blue vertical lines (the actual times of death) 
show that, in fact, this does happen.
See also Table~\ref{tab:4patients}.

\begin{figure} % Fig 5
\includegraphics[width=\textwidth]{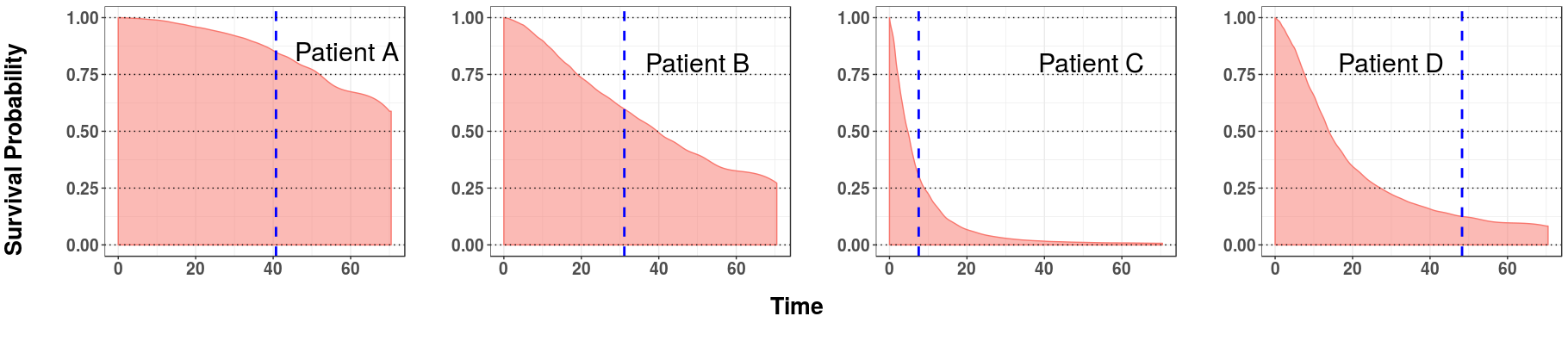}%
\caption{\label{fig:4subjects}
Four patients from the complete NACD dataset. 
Notice each died in a different quartile (shown with a vertical dashed line);
see Table~\ref{tab:4patients}.}
\end{figure}

\comment{
\begin{figure} % Fig 5
\includegraphics[width=0.25\textwidth,height=1.1in]{Figures/s1069.png}\hspace*{-0.00in}%
\includegraphics[width=0.25\textwidth,height=1.1in]{Figures/s347.png}\hspace*{-0.05in}%
\includegraphics[width=0.25\textwidth,height=1.1in]{Figures/s2050.png}\hspace*{-0.05in}%
\includegraphics[width=0.25\textwidth,height=1.1in]{Figures/s1962.png}
\caption{\label{fig:4subjects}
\note[HH]{We should make these numbers not so random... Change to 1,2,3,4. Its probably worth remaking these graphs to match style of Figure~\ref{fig:AllCurves}}
Four patients from the NACD dataset --  
from left to right: \#1069, \#347, \#2050, \#1962.
Notice each died in a different quartile; see Table~\ref{tab:4patients}.}
\end{figure}
}
\begin{table} % 1
\centering
\caption{Description of 4 patients from the NACD Dataset. 
(See also Figure~\ref{fig:4subjects})}
\label{tab:4patients}
\begin{tabular}{ >{\columncolor{LightCyan}}c|cccc}
\rowcolor{gray!50} Patient ID&	Median Survival Time& Event time&	Event Percentage&	Quartile\\
\hline
A & 85.5	&\ 43.4 &	84.7	&\#1\\
 % No http://pssp.srv.ualberta.ca/subjects/210174
B&  39.6	&\  31.1 &	59.8&	\#2\\
 % No http://pssp.srv.ualberta.ca/subjects/209452
C& 4.7	&\ 7.5 & 30.4&	\#3\\
 % No http://pssp.srv.ualberta.ca/subjects/211155
D&  13.9	&\ 48.3& 	12.8&	\#4\\
 % No http://pssp.srv.ualberta.ca/subjects/211067
\end{tabular}
\end{table}

With a slight extension to the earlier notation (Equation~\ref{eqn:g-DataInt}),
for a dataset $D$ and \SAc{P}{\infty}{i}-model $\Theta$,
and any interval $[a,b] \subset [0,1]$,
let 
\begin{equation}
\DataInt{D}{\Theta}{[a,b]} \quad=\quad
 \{\ [\inst{i},\,\death{i},\,\delta = 1] \in D \ |\ \estCP{\Theta}{\death{i}}{\inst{i}} \in [a,b]
 \ \}
\end{equation}
 be the subset of (uncensored) patients in the dataset $D$ whose time of death
 is assigned a probability (based on its individual distribution, computed by $\Theta$) % $\estP{\Theta}{\death{i}}$
 in the interval $[a,b]$.
 
As above, we could put these $\estCP{\Theta}{\death{i}}{\inst{i}}$
into ``10\%-bins'', and then ask if each bin holds about 10\% of the patients.
The right-side of Figure~\ref{fig:CalibrationHistogram}
plots that information, for the 
\ISD\ $\Theta$ % Individual Survival Distribution (\ISD)
learned by \MTLR\ from the NACD dataset 
(described in Section~\ref{sec:DataSets}),
as a sideways histogram.
\comment{(The left-side shows patient\#27's predicted survival curve,
and her time of death, at 12.7 months;
as $\estCP{\Theta}{d_{27}}{\inst{27}}\ =\ $39.4\%,
this patient contributed to the [30, 40) bin.)}
We see that each of these intervals is very close to 10\%. 
In fact, the 
$\chi^2$
goodness-of-fit test yields $p = \,$0.433,
which suggests that this \ISD\ 
is sufficiently uniform
that %and
we can believe that these survival curves are D-calibrated.

\begin{figure} % Fig 5
\includegraphics[width=\textwidth,height=2.05in]{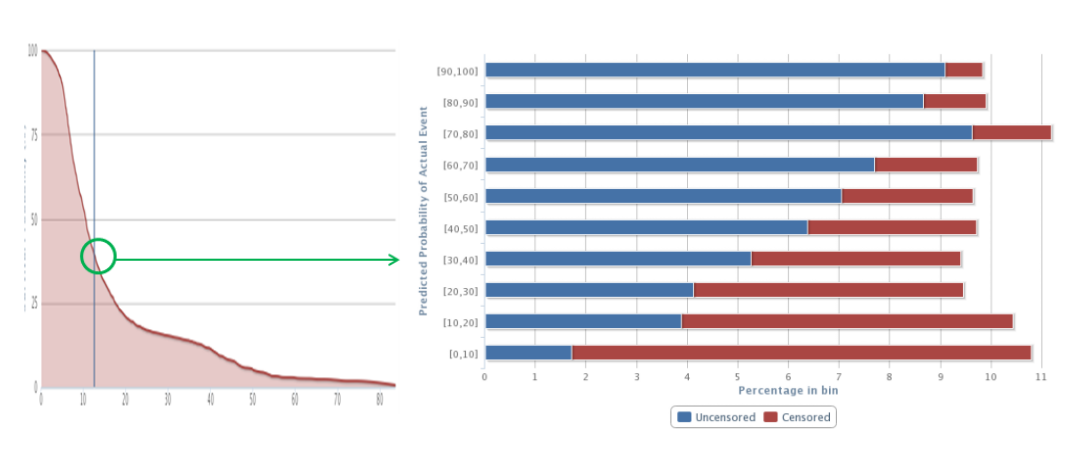}
\caption{\label{fig:CalibrationHistogram}
The right side shows the ``calibration histogram'' associated with the NACD dataset.
The left portion shows the survival curve for a patient $\inst{27}$
-- here we see that this patient's event $d_{27}=$12.7~months,
corresponds to 
$\estCP{%\Theta
}{d_{27}}{\inst{27}}\ =\ $39.4\%,
which means the patient
contributed to the [30, 40) bin.
In a completely D-calibrated model, 
each of these horizontal bars would be 10\%;
here, we see that each of the 10 bars is fairly close.
See also Figure~\ref{fig:DCal-Hist}.
}
\end{figure}

\comment{

Now imagine we had the ``true'' survival curve for each patient -- call it $p_i(\cdot)$ for patient\#i --
and that this patient died at time $e_i$.
It is easy to show that these true-\ISD\ curves are calibrated:
that is, the set $\{\,p_i(e_i)\,\}$ over all of the patients is distributed uniformly on $[0,1]$.
This would mean, in particular, 
that each of the 20 bins shown in Figure~\ref{fig:CalibrationHistogram}
would have 5\%.
}

Note that Figure~\ref{fig:CalibrationHistogram} is actually showing  5-fold cross-validation results: 
the survival % $\hat{p}_i(\cdot)$
curve for each patient was computed based on the model learned from the other 4/5 of the data, which is then applied to this patient~\cite{WekaBook}.
Also, the rust-colored intervals correspond to the censored patients;
see Appendix~\ref{app:D-Calib} for an explanation.

\comment{The next subsections  
discuss how well \PSSP\ works on other survival datasets (in terms of D-Calibration),
then 
whether other survival analysis/prediction models are calibrated 
on these various survival datasets.
}

\subsubsection{Relating D-Calibration to 1-Calibration} % 3.4.1
\label{sec:D-Cal+1-Cal}
This standard notion of 1-Calibration is very similar to D-Calibration, as both involve binning probability values and applying a goodness-of-fit test.
However, 1-Calibration involves a single prediction time -- 
here $\estCP{}{\tzero}{\inst{i}}$, % that is a single value for each patient,
which is 
{the}
probability that the patient $\inst{i}$ will survive at least to the specified time, \(\tzero\). 
Patients are then sorted by these probabilities,
partitioned into equal-size bins, and assessed as to
% assesses 
whether the observed survival rates for each bin match the predicted rates
using a Hosmer-Lemeshow test.
By contrast, D-Calibration considers the entire curve,
% \estP{i}{t} = 
$\estCP{}{t }{\inst{i}}$ over all times $t$
-- producing curves like the ones shown in Figures~\ref{fig:StomCan},~\ref{fig:AllCurves}, %~\ref{fig:Cox}
and~\ref{fig:4subjects}.
Each curve corresponds to a patient, who has an associated time of death, $\death{i}$.
Here, we are considering the model's (estimated) probability of the patient's survival at his/her time of death, given by 
$\estCP{i}{d_i}{\inst{i}}$. 
These patients % probabilities
are then placed into \(B = 10\)
bins,%
\footnote{Note the number of bins does not have to be 10 -- 
we chose 10 to match the typical value chosen for the 
1-Calibration % Hosmer-Lemeshow
test.}
based on the values of their associated probabilities, 
$\estCP{i}{d_i}{\inst{i}}$.
% according to which decile $\estP{i}{d_i|\inst{i}}$ belongs. 
Here the goodness-of-fit test measures whether the resulting bins are approximately equal-sized, as would be expected if \(\Theta\) accurately estimated the true survival curves (argued further in Appendix~\ref{app:D-Calib}). 
% \note[BH]{I reworded this section a bit to try and make the differences between 1-calibration and D-calibration more clear}
\comment{removed text ", then $\estCP{i}{d_i}{\inst{i}}$ will follow a uniform distribution  allowing the goodness-of-fit test to be used where the proportion of patients each bin will be equal to \(\frac{1}{B}\)."}

Note D-Calibration tests the proportion of instances in bins across the entire \([0,1]\) interval,
but this is not required for the ``single probability'' 1-Calibration.
For example, 
the single probability estimates
for the \RSFKM\ curve in Figure~\ref{fig:TypesSA}, at time 20, 
range only from  0.05 to 0.62.
That is, the distribution calibration 
$\{ \, \estCP{i}{\death{i}}{\inst{i}}\, \}$ 
should match the uniform distribution over [0,1],
while the single probability calibration 
 $\{ \, \estCP{i}{\tzero}{\inst{i}}\, \}$
is instead expected to match the empirical percentage of deaths.

\comment{
Also, some $[P,1,i]$ models % `single time'' probability models 
actually involve several models, one for each time (\eg the Gail model has both 5 and 25 year predictions).
One challenge is here how to aggregate the results -- 
\eg if comparing two tools, what if one tool is better the 5-years, but the other is better for 25 years.

Also, one should note that while single time probability models \textit{can} be used to estimate multiple time points, 
\eg 1 year, 5 years, 10 years, there is no clear way of aggregating the results. \note[BH]{I don't understand this statement} 
}

Table~\ref{tab:1calDcal} summarizes the differences between D-Calibration and 1-Calibration.%
\footnote{Further differences occur when considering how censored patients are handled; see Appendices~\ref{app:1-Calib} and~\ref{app:D-Calib}.}
To see that they are different, 
Proposition~\ref{Lma:1Cal.not.DCal}, 
in Appendix~\ref{app:D-Calib},
gives a simple example of a model that is 
perfectly D-Calibrated but clearly not 1-Calibrated, 
and another example that is perfectly 1-Calibrated but clearly not D-Calibrated.
In addition, we will see below several examples of this
-- \eg \CoxENKP\ is D-Calibrated for the GLI dataset,
but it is not 1-Calibrated at any of the 5 time points considered,
and 
\AFT\ is 1-Calibrated for the 50th and 75th percentiles of GBM but 
is not D-Calibrated.
% \note[RG]{example of other direction?} \note[HH]{We don't have any ``strong'' examples of the other direction, where strong means 1-Calibrated at 3 or more points and NOT D-Calibrated. We do have 1-Calibrated at 2 points and NOT D-Calibrated, e.g. \underline{RSFKM is 1-Cal for NACD at the 10th and 75th percentiles but not D-Calibrated} or \underline{AFT is 1-Cal for the 50th and 75th percentiles of GBM but not D-Calibrated.}}
% DCalIsNot1Cal.tex in overleaf ...
% }\note[HH]{DCal not implying 1-Cal: RSFKM is DCal for GLI but not 1-Cal at any points.}

\begin{table}
\caption{\label{tab:1calDcal}
Summary of differences % comparisons
between 1-Calibration and D-Calibration.}
\begin{tabular}{ | >{\columncolor{LightCyan}}l || c|c |} % p{2.2in} | p{2.2in} | }
\hline
\rowcolor{gray!20}  & \quad 1-Calibration  &\quad  D-Calibration \\
\hline\hline
Objective & Evaluate Single Time Probabilities
    & Evaluate Entire Survival Curve \\ 
    \hline 
Values considered & $\{\ \hat{p}(\,\tzero\, |\, \inst{i})\ \}$
    & $\{ \ \hat{p}(\,d_i\, |\, \inst{i})\ \}$\\ 
    \hline
Should match & Empirical number of deaths & Uniform \\ 
  \cline{1-1} \hline
Statistical Test & Hosmer-Lemeshow test & Pearson's $\chi^2$ test \\ 
\hline
\end{tabular}
\end{table}

\section{Evaluating \ISD\ Models} % 4
\label{sec:EvalISD}
Sections~\ref{sec:KM-model} and~\ref{sec:ISD}
listed several distributional models
(\KM, and the \ISD{}s: \CoxKP, \CoxENKP, \AFT, \MTLR, and \RSFKM), and
Section~\ref{sec:Eval} provided 5 different evaluation measures:
Concordance,  L1-loss, 1-Calibration, Integrated Brier score, 
and D-Calibration.
This section provides an empirical comparison of these 6 models,
with respect to all 5 of these measures, over 8 datasets.

Of course, these 6 models do not include all possible 
% The models selected are not inclusive of all 
survival models;
they instead serve as a sample of the types of models available. 
The \KM, \CoxKP, and \AFT\ model are all very common
--
these are
standard approaches used throughout survival analysis and 
represent non-parametric, semi-parametric, and parametric models, respectively. 
As our preliminary studies with \CoxKP\ suggested it was overfitting,
we also included a regularized extension,
using elastic net, \CoxENKP.
Since Random Survival Forests (\RSF) were introduced in 2008,
they have had a large impact on the survival analysis community.
However, as the Kaplan-Meier extension to transform \RSF\ into an \ISD\
is not well known, 
it is  summarized in Appendix~\ref{app:RSFKM}.
More recent still is the \MTLR\ technique~\cite{PSSP-NIPS} 
that directly learns a survival distribution,
by essentially learning the associated probability mass function 
(whose sequential right-to-left sum, when smoothed, is the survival distribution).
We found some subsequent similar models, including
``Multi-Task Learning for Survival Analysis'' (MTLSA)~\cite{li2016multi},
% (whose  survival curves were not bounded above by 1) 
some deep learning variants~\cite{ranganath2016deep, katzman2016deepsurv,luck2017deep}, and a computationally demanding Bayesian regression trees 
model~\cite{mcculloch2016},
but for brevity, we focused on just the first 
such
model, \MTLR.

% \note[RG]{Btw, can you explain
% "most of these deep learning models do not result in an \hbox{\ISD}"}\note[HH]{I believe I meant that (most) of these models were returning risk scores, not survival probabilities.}

Note the distribution class ${\cal D}$ chosen for \AFT\  
 certainly influences its performance 
 -- \eg it is possible that 
 \AFT[Weibull] on a dataset may fail D-Calibration whereas \AFT[Log-Logistic] may pass;
 similarly for 1-Calibration at some time $\tzero$,
and the scores for Concordance, L1-loss and Integrated Brier score
will depend on that distribution class.
This paper will focus on \AFT[Weibull] because,
while still being parametric,
the Weibull distribution is versatile enough to fit many datasets.

\comment{ EARLIER VERSION:
More recent still is the representation of survival analysis as a multi-task learning problem where different times are treated as different tasks
-- each learning the parameters for an interval of a PMF (probability mass function).
We did find some similar models, including
`'`Multi-Task Learning for Survival Analysis'' (MTLSA)~\cite{li2016multi},
was created but the survival curves generated are not bounded above by 1 thus not acting as probabilities and thus MTLSA is not considered here.
Of note is the recent work on adapting survival analysis as a deep learning model \cite{ranganath2016deep, katzman2016deepsurv,luck2017deep}, however, for brevity and the fact that most of these deep learning models do not result in an \ISD, we do not consider deep learning models in our empirical assessments.
}
% See \cite{PSSP-NIPS} for further details about the implementation.

\subsection{Datasets and Evaluation Methodology} % 4.1
\label{sec:DataSets}

There are many % publicly available 
different
survival datasets;
here, we selected 8 
publicly available
medical datasets 
in order
to cover a 
wide
range of 
% with focus on varying levels of 
sample sizes, number of features, and proportions of censored patients. 
We excluded small datasets (with fewer than 150 instances)
% size less than 150 were excluded entirely as 
to % \change[BH]{ensure models had a sufficient number of samples from which to learn.}
{reduce the variance in the evaluation metrics.} 
%\note[BH]{I thought smaller sets were excluded because the evaluations would have high variance} 
% In total, we considered 8 publicly available datasets, with the number of patients ranging from 170 up to 2402 and number of features ranging from 12 to 7401. 
Our datasets ranged from 170 to 2402 patients,
from 12 to 7401 features,
and percentage of censoring from 17.23\% to 86.21\%;
see Table~\ref{tab:datasets}.  Note that we have not included extremely high-dimensional data (with tens of thousands of features, often found in genomic datasets), as such data raises additional challenges beyond the scope of standard survival analysis; see \cite{witten2010} for methods to handle extremely high-dimensional data.

The Northern Alberta Cancer Dataset~(NACD),
with 2402 patients and 53 features, %  and 36.59\% censor-percentage,
is a conglomerate of many different cancer patients,
including lung, colorectal, head and neck, esophagus, stomach, and other cancers.
In addition to using the complete NACD dataset,
we considered the subset of 950 patients with colorectal cancer (\NacdCol),
with the same 53 features.
% and 51.89\% ratio. 

Another four % of 38 possible 
datasets were retrieved from data generated by The Cancer Genome Atlas (TCGA) Research Network~\cite{TCGAData}:
Glioblastoma multiforme (GBM; 592 patients, 12 features), 
Glioma (GLI; 1105 patients, 13 features), 
Rectum adenocarcinoma (READ; 170 patients, 18 features), and  
Breast invasive carcinoma (BRCA; 1095 patients, 61 features).
To ensure a variety of feature/sample-size ratios,
we consider only the clinical features in our experiments. 
% The datasets used entailed
 
Lastly, we included two high-dimensional datasets:
the Dutch Breast Cancer Dataset (DBCD)~\cite{van2006cross} contains 4919 microarray gene expression levels for 295 women with breast cancer,
and the Diffuse Large B-Cell Lymphoma (DLBCL)~\cite{li2016multi} dataset contains 7401 features focusing on Lymphochip DNA microarrays for 240 biopsy samples.

We applied the following pre-processing steps to each dataset:
We first removed any feature that was missing over 25\% of its values,
as well as any features containing only 1 unique value.
For the remaining features, 
we ``one-hot encoded'' each nominal feature and then passed each feature to a univariate Cox filter,
and removed any feature that was not significant at the $p \leq 0.10$ level. Following feature selection,
we 
 replaced any missing value
with the respective feature's mean value.
% \note[RG]{What about siblings -- eg, other features that were created by 1-hot encoding?}
% \note[HH]{One hot encoding occurs prior to feature selection, so in the example of race the cox filter would evaluate White/Not White, Black/Not Black, etc. independently. If these were passed in together it would no longer be univariate. Also imputation follows after feature selection so I changed the text to reflect this.}
% for feature selection. If a feature was significant at the $p \leq 0.10$ level the feature was kept and discarded otherwise. 
(Note this feature selection was found to % equally 
benefit all \ISD\ models across all performance metrics; data not shown.)
Table~\ref{tab:datasets} provides the dataset statistics and 
a full breakdown of feature numbers in each step.

\begin{table}[tb]
\caption{\label{tab:datasets} Overview of datasets used for empirical evaluations. 
From top to bottom:
(1)~the number of patients in each dataset, 
(2)~the percent of patients censored, 
(3)~the number of features contained in the original dataset, 
(4)~the number of features after removal of features containing over 25\% missing data or only 1 unique value,
(5)~the number of features after univariate Cox selection, and 
(6)~the feature-to-sample\_size ratio.}
\centering
	\resizebox{\textwidth}{!}{
	\begin{tabular}{ >{\columncolor{LightCyan}}r|cccccHccc}
  \hline
\rowcolor{gray!20} & GBM & GLI & \NacdCol & NACD & READ & THCA & BRCA &  DBCD & DLBCL \\ 
  \hline
Number of patients: $N$ & 592 & 1105 & 950 & 2402 & 170 & 503 & 1095 & 295 & 240  \\ 
  \% Censored & 17.23 & 44.34 & 51.89 & 36.59 & 84.12 & 96.82 & 86.21 & 73.22 & 42.50 \\ 
  \# Features Originally: $f_{raw}$ & 12 & 13 & 53 & 53 & 18& 21 & 61 & 4921 &  7401  \\ 
  \# Features Post-Processing: $f_{proc}$ & 9 & 10 & 45 & 53 & 13 & 19& 59  & 4921 &  7401  \\ 
  \# Features Selected: $f_{final}$ & 6 & 10 & 34 & 46 & 8 & 14 & 28 & 2330 &  1771  \\ 
  $f_{final}\, / \, N $ & 0.010 & 0.009 & 0.036 & 0.019 & 0.047 & 0.028 & 0.026 & 7.898 & 7.379\\
   \hline
\end{tabular}}
\end{table}

Following feature selection, features were normalized (transformed to zero mean with unit variance)
and passed to models for five-fold cross validation (5CV). 
We compute the folds by sorting the instances by time and censorship,
then placing each censored (resp., uncensored) instance sequentially into
the folds -- meaning all folds had roughly the same distribution of times, and censorships.
% Folds were deterministic in that survival times were sorted and stratified by censorship and then placed into folds in a sequential matter. 

For \CoxENKP, \RSFKM, and \MTLR\, 
we used an internal 5CV for hyper-parameter selection.
There were no hyper-parameters to tune for the remaining models: 
\CPH, \KM, and \AFT.
% \note[HH]{One may argue that the \textit{distribution} chosen for AFT could be a parameter and that should have been selected for using internal 5CV as opposed to sticking with 1 distribution.
% Weibell, LogNormal, Logistic, LogLogistic}

As 1-Calibration required specific time points, 
and as models might perform well on 
some % early
survival times  but poorly 
on others, % for later survival times,
we chose five times to assess the calibration results of each model: 
the 10th, 25th, 50th, 75th, and 90th percentiles of survival times for each dataset. 
Here, we used the D'Agostino-Nam translation to include censored patients for these evaluation results -- see Appendix~\ref{app:1-Calib}.
% Additionally,
Appendix~\ref{app:1-CalDetails} 
presents % also includes
all 240 values
(6 models $\times$ 8 datasets $\times$ 5 time-points); % percentiles) 
here we instead summarize the number of datasets 
that each model passed 
as 1-Calibrated (at $p\geq$0.05)
% on 1-Calibration
for each percentile. 

For all evaluations, 
we report the averaged 5CV results for Concordance, Integrated Brier score, and L1-loss. 
As Concordance requires a risk score,
we use the negative of the median survival time and similarly use the median survival time for predictions for the L1-loss.
To adjust for presence of censored data, 
we used the L1-Margin loss, given in Appendix~\ref{app:L1-loss}, 
which extends the ``Uncensored L1-loss'' given in Section~\ref{sec:L1-loss}
(which considers only uncensored patients).
Additionally, as 1-Calibration (resp., D-Calibration) results are reported as \(p\)-values, 
and it is not appropriate to average over the folds,
we
combined the predicted survival curves from all cross-validation folds
for a single evaluation,
and report the resulting \(p\)-value. 

Empirical evaluations were completed in R version 3.4.4.
The implementations of \KM, \AFT, and \CoxKP\ can all be found in the \textit{survival} package~\cite{survival-package} whereas \CoxENKP\ uses the \textit{cocktail} function found in the \textit{fastcox} package~\cite{coxKPEN}. 
Both \RSF\ and \RSFKM\ come from the \textit{randomForestSRC} package~\cite{rfsrc}. 
An implementation of \MTLR\ (and of all the code 
used in this analysis) is publicly available on the  GitHub account\footnote{https://github.com/haiderstats/ISDEvaluation} 
of the lead author.

%\MTLR's source code is freely available from its web site%
%\footnote{ http://pssp.srv.ualberta.ca/
%}.
% \MTLR\ is freely available as an online web tool\footnote{http://pssp.srv.ualberta.ca/} where source code can also be accessed.

\subsection{Empirical Results} % 4.2
\label{sec:Empirical Results}

\def\Nice{{\sc Nice}}
\def\HD{{\sc High-Dimensional}}
\def\HC{{\sc High-Censor}}

Below, we consider a dataset to be ``\Nice''
if its feature-to-sample-size ratio was less than 0.05 
{(for the final feature set)}
and its censoring was less than 55\%;
this includes four of the 8 datasets:
GBM, \NacdCol,  GLI, NACD --
% \annote[HH]
{which are shown first in all of our empirical studies.}
We let ``\HC'' datasets refer to READ and BRCA and
``\HD'' datasets refer to the other two (DLBCL and DBCD).

\def\hding#1{\medskip \noindent {\bf #1:}}

\def\hdingtwo#1{\medskip \noindent {\textbf{#1:}}}

\subsubsection{Concordance, Integrated Brier score, and L1-loss Results} % 4.2.1
\label{sec:Emp-C,IBS,L1}

Figures~\ref{fig:ConcEvaluation},~\ref{fig:BrierEvaluation}~and~\ref{fig:L1Evaluation} 
give the empirical results for Concordance, Integrated Brier score, and L1-Margin loss
respectively,
where each circle is the score of the associated model on the dataset,
and lines correspond to one standard deviation 
(over the 5 cross-validation folds).
Appendix~\ref{app:EmpiricalDetail} provides the exact empirical results for these measures.

% We focus first on the \ISD-models (\ie ignoring \KM):

\hding{Best Performance}
% We first note that, for most datasets and evaluation measures, 
% the scores of the 5 distributional models appear fairly similar.
The blue circles represent the best performing models, for each dataset;
here we find that \MTLR\ performs best on a majority of datasets:
six of eight for Concordance and L1-loss, and seven of eight for the Integrated Brier score.
% Concordance, Integrated Brier score, and L1-loss. 
% \note[BH]{I suspect the Concordance and Brier results aren't statistically significant, so we might want to downplay this a bit}
% \note[RG]{Should discuss which of these are statistically significant.}

\hdingtwo{\textsc{Nice}\ Datasets}
Recall that the first 4 datasets are \Nice.
Here, we find that most models performed comparably
-- and in particular, 
% However, in datasets where the feature-to-sample-size ratio is less than 0.05,
\AFT\ and \CoxKP\ perform nearly as well as the other, more complex, models.
% -- 
\AFT\ even performs best in terms of L1-loss on GBM.
% in terms of Concordance, Integrated Brier score, and L1-loss 
The only exception was \RSFKM,
which did much worse on GBM and GLI, in all three measures.

\KM\ was worse than the various \ISD-models for all 3 measures.
%\note[RG]{well .. not even defined for Concordance ..}\note[HH]{It is defined for concordance since we include ties. (We actually chose to include ties because then KM is defined).}
%\note[RG]{re-wrote next -- ok?}
(The only exception was \RSFKM,
which was worse on for the datasets GLI and GBM for 
Integrated Brier score, 
and for those datasets and also \NacdCol\ for L1-loss.)

\medskip \noindent \textbf{\textsc{High-Censor} Datasets -- READ and BRCA:}
% We now consider the highly-censored datasets, READ and THCA. 
Note first that the variance in the evaluation metrics is generally higher on READ for all models (except \KM) due to the small number of uncensored patients within each test fold -- this is not present in BRCA due to the larger sample size (1095). 
Again we find that \CoxENKP and \MTLR\ are similar for all three measures, but \RSFKM\ performs consistently worse across all three metrics for both READ and BRCA.
\AFT\ and \CoxKP\ are either comparable
(or inferior)
to the other three \ISD-models:
% In general, \AFT\ and \CoxKP\ are either similar, or worse, than the other.
Concordance: worse performance but within error-bars for READ and BRCA;
Integrated Brier score: similar for both READ and BRCA;
 L1-loss: slightly worse for READ and BRCA. Additionally, \AFT\ and \CoxKP\ tend to show higher variance in evaluation estimates on READ than other models for all three measures.
% -- with the exceptions of THCA and READ, which had 96.82\% and 84.12\% censored patients, respectively. 
% Further, we see poor performance from \RSFKM\ -- worse than the basic \AFT\ and \CoxKP\ models -- for three of the ``nicer'' datasets (smaller feature to sample size ratio and smaller proportion of censored patients).

\KM\ is worse than all 5 \ISD-models for Concordance,
but comparable to the best for Integrated Brier score and L1-loss (actually scoring better than \CoxKP\ and \AFT\ for L1-loss on READ and beating \CoxKP\ on BRCA).

\hdingtwo{\textsc{High-Dimensional}\ Datasets -- DBCD and DLBCL}
There are no entries for \CoxKP\ for these two datasets
as it failed to run on them, 
likely due to the large number of features.
As \AFT\ is  unregularized, 
it is not surprising that it does poorly across all measures for these high-dimensional datasets -- indeed, even worse than \KM, which did not use any features! % , DBCD and DLBCL. 
% Further, \CoxKP\ fails to run for these two datasets,
% which is why those entries are empty in the figures.
% and hence no empirical results are available for \CoxKP\ on DLBCL and DBCD. 
We see that the other three \ISD-models -- \CoxENKP, \MTLR\, and \RSFKM\ --
perform similarly to one another here, and \KM\ 
also achieves similar results
(ignoring Concordance where \KM\ always achieves 0.5, as it gives identical predictions for all patients).

\begin{figure} % Fig 14
\centering
\includegraphics[height=17.5cm]{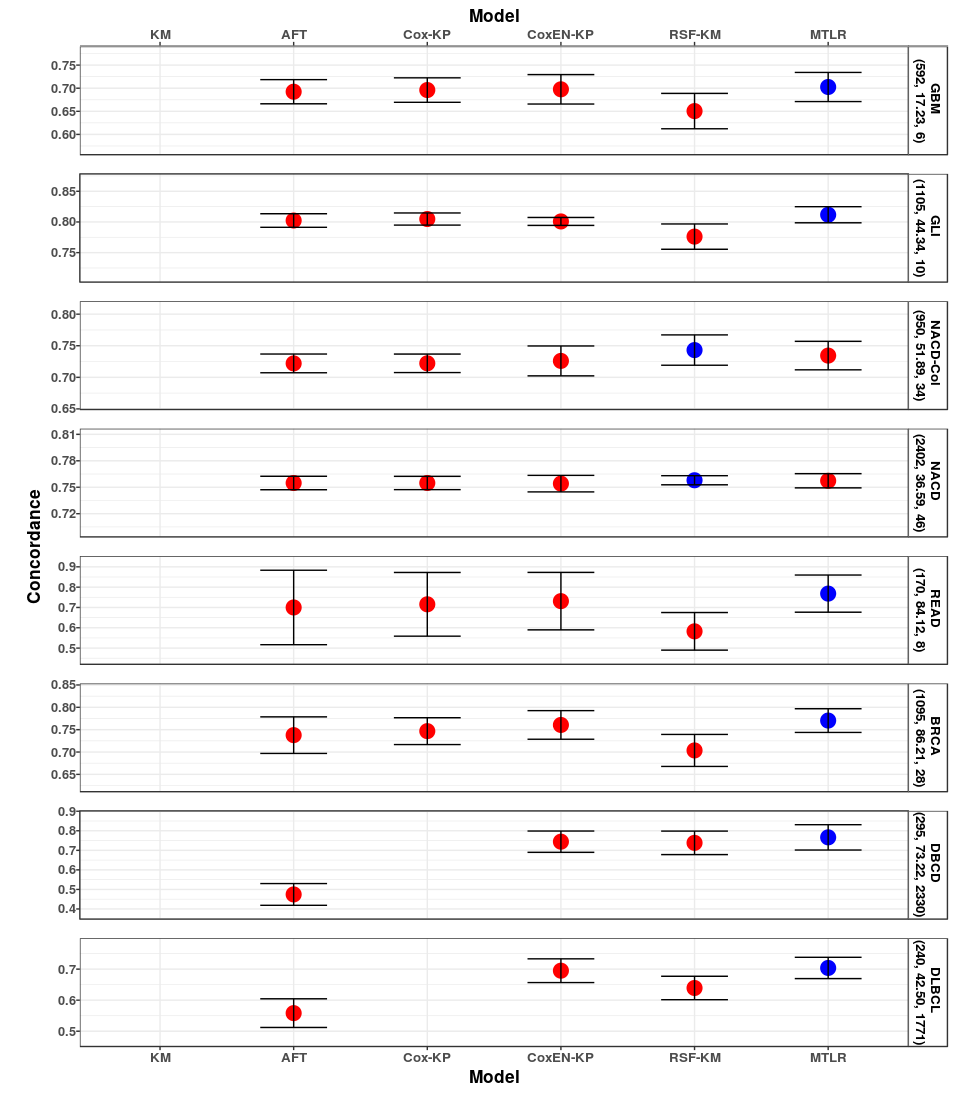}
\caption{\label{fig:ConcEvaluation}
Concordance means and one standard deviation are given by 
circles
and error bars, respectively.
The best (highest Concordance) scoring model is given in blue;
all other models are in red.
% Note that \CoxKP\ failed to run for datasets DBCD and DLBCL, hence having no results for these datasets.
We included \KM\ so this figure would ``line up'' with Figures~\ref{fig:BrierEvaluation} and \ref{fig:L1Evaluation},
but left the value blank, 
as the Concordance scores for \KM\ are always 0.5.
For these 3 figures: % Here, and Figures~\ref{fig:BrierEvaluation} and \ref{fig:L1Evaluation}:
As \CoxKP\ failed to run for datasets DBCD and DLBCL, 
those entries are blank.
% We intentionally left \KM\ blank for aesthetics, as all Concordance scores for KM are equal to 0.5.
The description at the right gives the name of the dataset,
and the 3 numbers ``under'' each dataset name are 
\hbox{($f_{final}$, \% Censored, $N$)}. 
}
\end{figure}

\begin{figure} % Fig 15
\centering
\includegraphics[height=17.5cm]{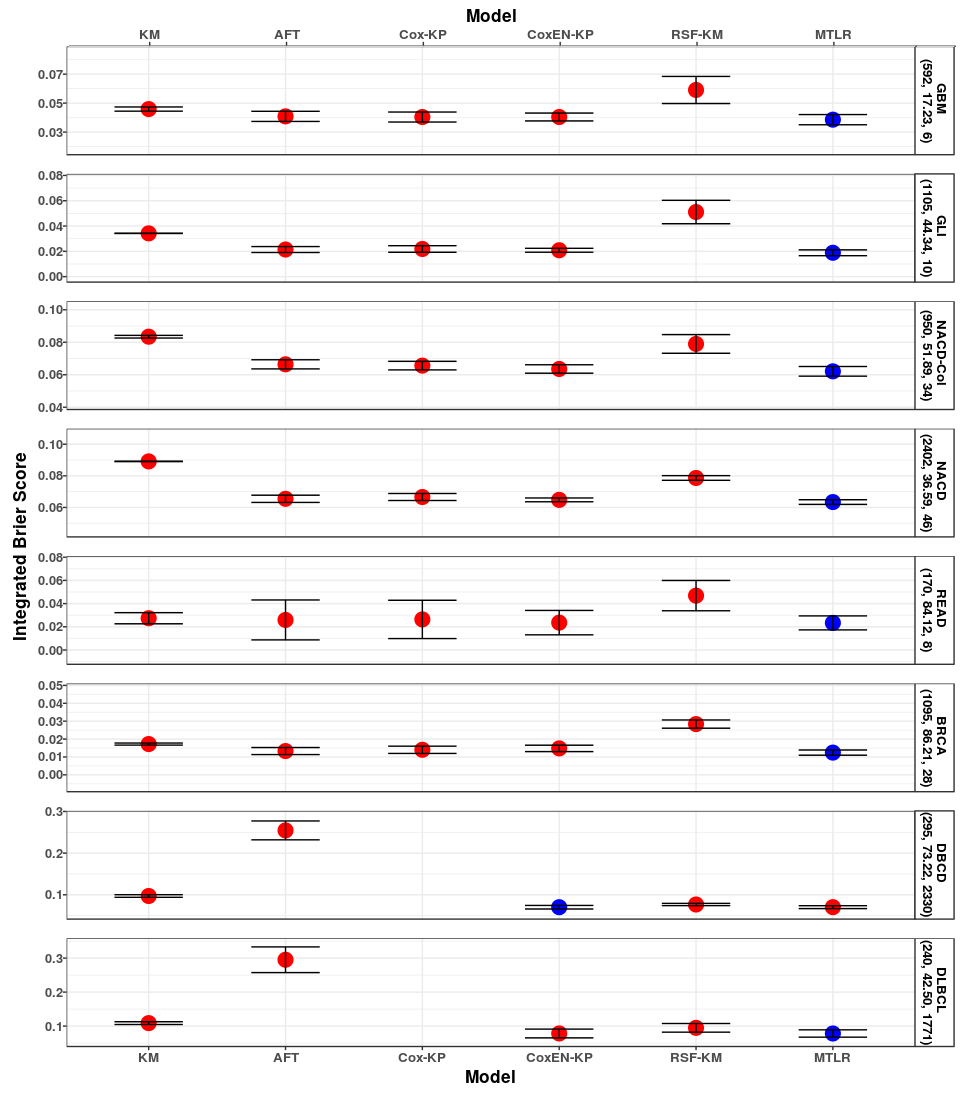}
\caption{\label{fig:BrierEvaluation}
Integrated Brier score means and one standard deviation are given by circles and error bars, respectively.
The best (lowest Integrated Brier score) scoring model is given in blue;
all other models in red.
% See also Figure~\ref{fig:ConcEvaluation}'s caption.
% As \CoxKP\ failed to run for datasets DBCD and DLBCL, these entries are blank. 
}
\end{figure}

\begin{figure} % Fig 16
\centering
\includegraphics[height=17.5cm]{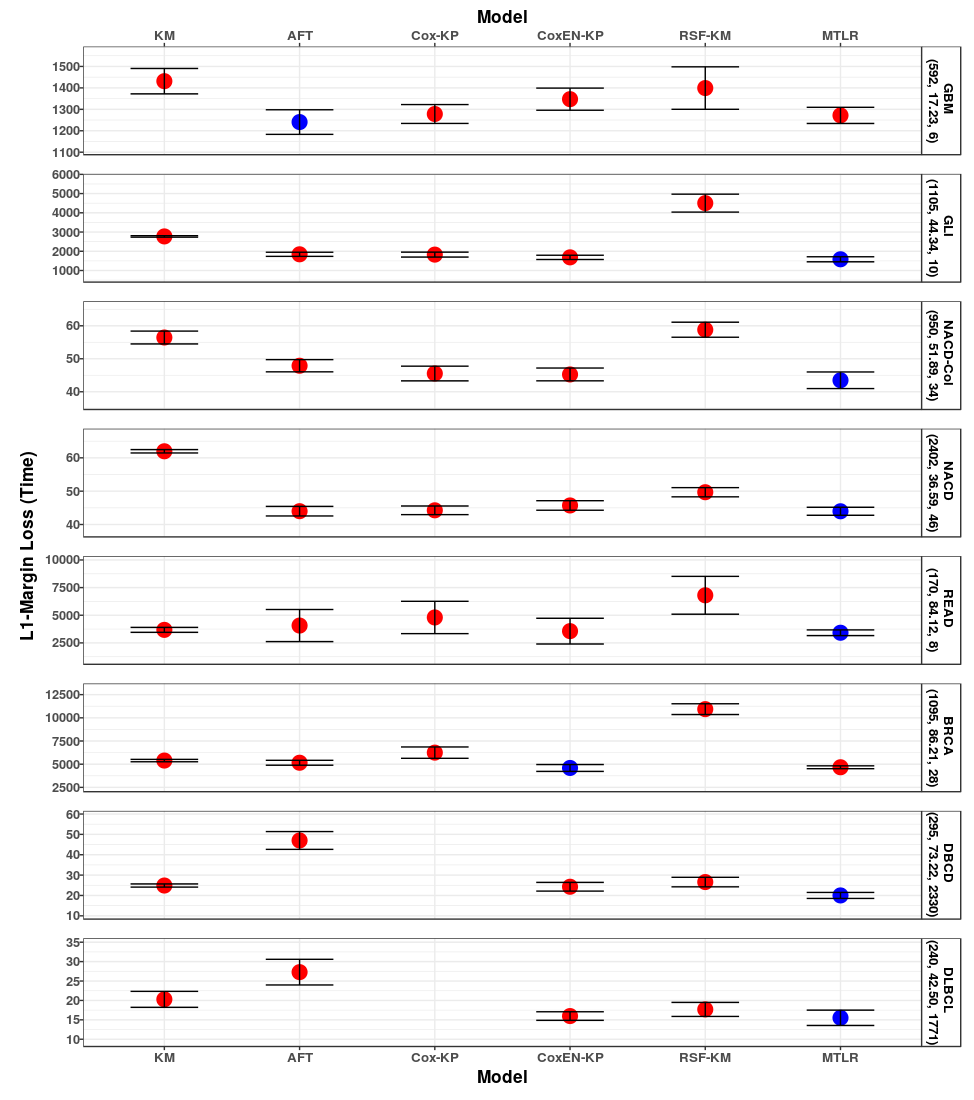}
\caption{\label{fig:L1Evaluation}
L1-loss means and one standard deviation are given by circles and error bars, respectively. 
The best (lowest L1-loss) scoring model is given in blue, all other models in red. 
% Note that \CoxKP\ failed to run for datasets DBCD and DLBCL, hence having no results for these datasets. 
% Additionally, as the time unit changes across datasets,
As different datasets use different time units,
we simply give the units of L1-loss 
as ``Time'' rather than days/months/years.
%As \CoxKP\ failed to run for datasets DBCD and DLBCL, those entries are blank.
}
\end{figure}

\begin{table}[t]
\caption{\label{tab:cum1Cal}
	Results from 1-Calibration evaluations. Columns represent percentiles used for each time point and rows indicate the model used.
	% Recall that the maximum number of datasets to pass 1-Calibration is 8 datasets, 
	Recall there are 8 datasets -- meaning no model performed perfectly for any of the percentiles.
}
\label{tab:1Cal}
\centering
\begin{tabular}{>{\columncolor{LightCyan}}r|cccccc}
	\hline
	\rowcolor{gray!20} & 10th  & 25th & 50th & 75th & 90th \\ 
	\hline
	\AFT & 4 & 2 & 1 & 1 & 0  \\ 
	\CoxKP & 4 & 2 & 2 & 1 & 0 \\ 
	\CoxENKP & 4 & 3 & 1 & 2 & 2 \\ 
	\RSFKM & 4 & 2 & 2 & 1 & 0\\ 
	\MTLR & \textbf{6} & \textbf{8} & \textbf{6} & \textbf{3} & \textbf{4} \\ 
	\hline
\end{tabular}
\end{table}

\subsubsection{1-Calibration Results} % 4.2.2
\label{sec:1-cal-emp}
Table~\ref{tab:1Cal} gives the number of datasets each model passed for 1-Calibration 
%\textit
{for each time of interest}. 
% The number of datasets each model passed for 1-Calibration \textit{for each time of interest} are given in Table~\ref{tab:1Cal} and \(p\)-values for D-Calibration are given in Table~\ref{tab:Dcal}.
% \note[RG]{Rewrote -- original in comment below.}
We see % Table~\ref{tab:cum1Cal} shows 
that \MTLR\ is typically
1-Calibrated across the percentiles of survival times. 
Specifically, \MTLR\ is 1-Calibrated for at a minimum of four of eight datasets for the 10th, 25th, 50th, and 90th percentiles, outperforming all other models considered. 
The 90th percentile appear to be the most challenging in general,
as some models (\AFT, \CoxKP, \RSFKM) are not 1-Calibrated for any datasets,
\CoxENKP\ is 1-Calibrated for two, and \MTLR\ is 1-Calibrated for four. The 75th percentile also showed to be challenging, however \AFT, \CoxKP, and \RSFKM\ were 1-Calibrated for one, \CoxENKP\ is 1-Calibrated for two, and \MTLR\ is 1-Calibrated for three. 
The most challenging datasets for \RSFKM\ once again were GBM, GLI, BRCA, and READ, for which \RSFKM\ was 1-Calibrated only at the 10th percentile for READ 
-- see Appendix~\ref{app:1-CalDetails}. 
Additional challenging datasets include the complete NACD and DBCD which were challenging for all models.
As \KM\ assigns an identical prediction for all patients,
it cannot partition patients into different bins,
meaning it cannot be evaluated by 1-Calibration.

\comment{
In addition to other performance measures,
we see \MTLR\ performing much better 
in 1-Calibration across all the percentiles of survival times. 
Specifically, \MTLR\ is 1-Calibrated for six of eight dataset for the 10th, 25th and 75th percentiles whereas other models are 1-Calibrated for at most four of the datasets. 
Most challenging appears to be the 90th percentile where most models (\AFT, \CoxKP, and \RSFKM) are not 1-Calibrated for any datasets, \CoxENKP\ is 1-Calibrated for three and \MTLR\ is 1-Calibrated for four.
Note % Of note is
that \RSFKM\ has the worst performance,
being 1-Calibrated for at most three datasets for the 10th percentile, 
one dataset for the 25th and 50th percentile and no datasets for the 75th and 90th. 
The two most challenging datasets for 1-Calibration appeared to be GBM and the complete NACD -- see Appendix~\ref{app:1-CalDetails}. 
As \KM\ assigns an identical prediction for all patients,
it cannot partition patients into different bins,
meaning it cannot be evaluated by 1-Calibration.
% Note that \KM\ assigns an identical prediction for all patients and thus 1-Calibration is not a suitable metric to evaluate \KM (as patients cannot be partitioned into different bins).
}

\subsubsection{D-Calibration Results} % 4.3.3
\label{sec:D-cal-emp}

Table~\ref{tab:Dcal},
which gives the D-Calibration \(p\)-values for each model and dataset,
shows that both \KM\ and \MTLR\ pass D-Calibration for every dataset, 
with \KM\ receiving the highest possible \(p\)-value, $p =$1.000, for each.
(In fact, 
Lemma~\ref{lem:KM-D-Calib} in Appendix~\ref{app:D-Calib} 
proves that \KM\ is asymptotically D-Calibrated).
While \KM\ will tend to be D-Calibrated,
it 
is also % will also tend to be 
the \textit{least} informative model,
since it % \KM\
assigns all patients the same survival curve. 
\MTLR\ is also D-Calibrated for all datasets,
but in addition, it also provides each patients with his/her own survival curve.

% \MTLR\ has the benefit of being D-Calibrated for all datasets but also provides patients with a survival curve on an individual level. 

Following \KM\ and \MTLR, 
\CoxENKP\ performed next best, only failing to be D-Calibrated for one dataset:
% the complete 
NACD. 
\RSFKM\ followed closely behind, being D-Calibrated for five of eight datasets, failing on GBM, GLI, and NACD.
\AFT\ performed similarly to \CoxKP, each of which being D-Calibrated on three of eight datasets.

% \eg a choice of Weibull may have \AFT\ fail D-Calibration whereas log-logistic may make \AFT\ D-Calibrated.

% \note[HH]{Maybe include a histogram of a failed D-Calibration... from \RSFKM? Does anyone think this would be helpful or does it just add noise?}
% \note[RG]{Yes -- have some picture of those side-way histograms, showing great, good, bad... to help emphasize Fig 8.}

Figure~\ref{fig:DCal-Hist} provides (sideways) histograms,
to help visualize D-calibration.
For each subfigure, each of the 10 horizontal bars should be 10\%;
we see a great deal of variance for the not-D-Calibrated \CoxKP\ [left],
a small (but acceptable) variability for the D-Calibrated \MTLR\ [middle],
and essentially perfect alignment for the D-Calibrated \KM\ [right].
See also Figure~\ref{fig:CalibrationHistogram}.

\begin{figure}[tb] % H] % Fig 17
\centering
\includegraphics[width = \textwidth]{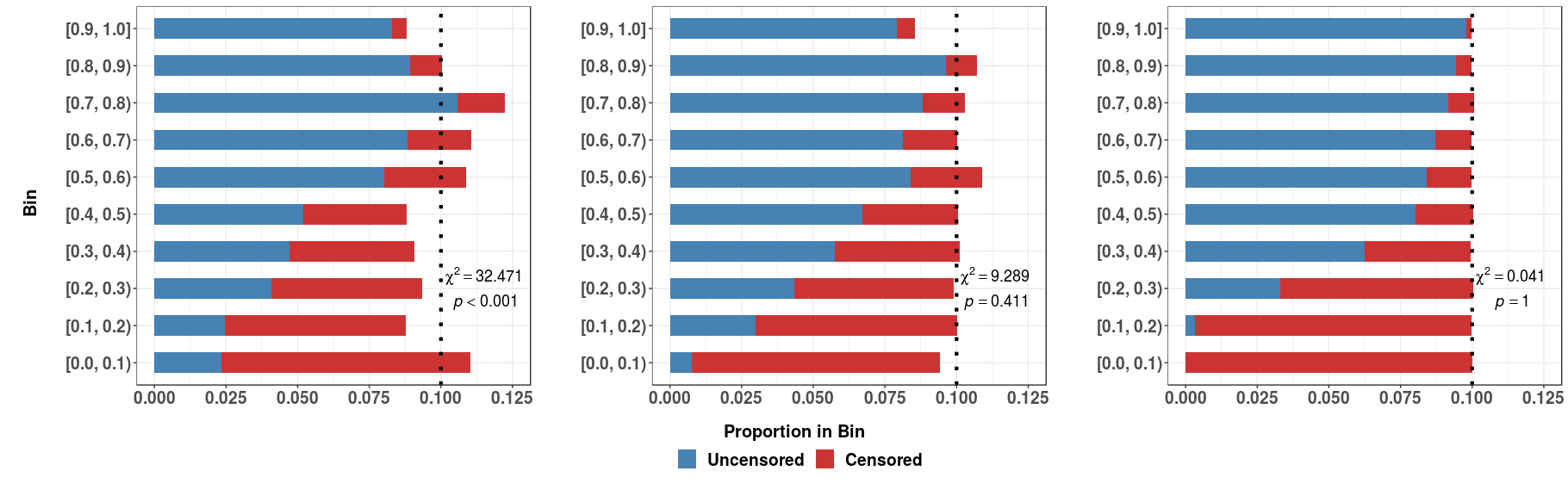}
\caption{\label{fig:DCal-Hist} % fig:DCalExample}
These figures show the (sideways) decile histogram 
used for the D-Calibration test.
Each of these is run on the NACD dataset;
from left to right: running  \CoxKP, \MTLR\ and \KM.}
\end{figure}

\def\NA{-}  % NA

\begin{table}[tb] % H
	\caption{\label{tab:Dcal} Results for D-Calibration evaluations. Columns correspond to the dataset and rows to the model. Results are the \(p\)-value from the goodness-of-fit test. 
	\textbf{Bold} values indicate that a model passed D-Calibration, 
	\ie \(p \geq 0.05\); and  ``\NA'' means the algorithm did not return an answer.}
	\centering
	\resizebox{\textwidth}{!}{
		\begin{tabular}{>{\columncolor{LightCyan}}r|cccccHccc|c}
			\hline
			\rowcolor{gray!50}  & GBM & GLI & \NacdCol & NACD & READ & THCA & BRCA & DBCD & DLBCL &Total \\ 
			\hline
			% \IncludeStats
		\KM & \textbf{1.000} & \textbf{1.000}& \textbf{1.000} & \textbf{1.000} & \textbf{1.000} & \textbf{1.000} & \textbf{1.000} & \textbf{1.000} & \textbf{1.000} & 8 \\ 
		  \AFT & 0.000 & 0.017 & \textbf{0.290} & 0.000 & \textbf{0.807} & \textbf{1.000} & \textbf{0.988} & 0.000 & 0.000 & 3\\ 
		  \CoxKP & 0.046 & 0.049 & \textbf{0.107} & 0.000 & \textbf{0.939} & 0.000 & \textbf{0.995} & \NA & \NA & 3 \\ 
            \CoxENKP & \textbf{0.447}  & \textbf{0.128} & \textbf{0.691} & 0.000 & \textbf{1.000} & \textbf{1.000} & \textbf{1.000} & \textbf{0.430} & \textbf{0.758} & 7 \\ 
              \RSFKM & 0.000 & 0.000 & \textbf{0.403} & 0.000 & \textbf{0.974} & \textbf{1.000} & \textbf{0.757} & \textbf{0.911} & \textbf{0.574}  & 5\\ 
              \MTLR & \textbf{0.688} & \textbf{0.883} & \textbf{0.656} & \textbf{0.411} & \textbf{1.000} & \textbf{1.000} & \textbf{0.995} & \textbf{0.994} & \textbf{0.755} & 8\\ 
			\hline
		\end{tabular}}

	\end{table}

\section{Discussion} % 5
\label{sec:Discussion}
\hding{Comparing different \ISD-models}
Steyerberg~\etal~\cite{steyerberg2010assessing} noted two different types of performance measures of a survival
% compose of two important characteristics of a 
analysis model
-- calibration and discrimination --
each of which can be assessed separately:
\begin{description}
\item [Calibration:] 
``Of 100 patients with a risk prediction of $x$\%, 
do close to $x$ experience the event?''
% ``Do close to x of 100 patients with a risk prediction of x\%, experience the event?''
\item[Discrimination:]
``Do patients with higher risk predictions experience the event sooner than those who have lower risk predictions?''
\end{description}
Discrimination is a very important measure
for some situations -- 
\eg if we have 2 patients who each need a kidney transplant, but there is only a single kidney donor, then we want to know which patient will die faster \textit{without} the transplant~\cite{kamath2007model}.
%  https://en.wikipedia.org/wiki/Model\_for\_End-Stage\_Liver\_Disease] 
As discussed in Section~\ref{sec:Concordance}, Concordance measures how well a predictor does, in terms of this discrimination task.

% \note[RG]{Here... should motivate why calibration is useful.. Or at least motivate why indiv surv curve is relevant to calibration}

This paper, however,  
motivates and studies models that produce
an individual survival curve for a specific patient. 
Such
\ISD\ tools %, 
may not be optimal for maximizing discrimination (and therefore Concordance);
% \note[RG]{rewrote -- see comment below}
and even tools like \CPH\ and \RSF,
that were originally developed for discrimination, 
were then extended to produce these individual survival curves. % \note[HH]{The prior sentence is a little confusing.}
% -- \eg \CPH\ and \RSF\ were originally developed as tools to estimate patient hazards (discriminate/rank patients) but were then extended to produce individual survival curves. 
Given this qualifier, we see 
(over the set % within the realm 
of \ISD\ tools tested),
\MTLR\ scored best on Concordance for six of the eight datasets tested and
\RSFKM\ scored the best on the other two.
(The relatively low performance of \CoxKP\ is unexpected given the claim that
``a method designed to maximize the Cox's partial likelihood also ends up (approximately) maximizing the [concordance]% CI
''~\cite{steck2008ranking}.)
% https://papers.nips.cc/paper/3375-on-ranking-in-survival-analysis-bounds-on-the-concordance-index
%\note[RG]{Should we mention the ``Cox = concordance'' paper? }
However,
when we look at the 
\Nice\ datasets% 
\comment{ with ``nice'' conditions 
 (feature to sample size ratio of less than 0.05 and less than \(\sim 50\)\% censoring 
-- \ie 
% GBM, GLI, NACD, \NacdCol
GBM, \NacdCol,  GLI, NACD
\note[RG]{I reordered them, to match the tables -- GBM, Nacd-Col, GLI, NACD?
Or is there any reason to keep earlier order??}
)
},
4 of the 5 \ISD-models give nearly identical results
(\RSFKM\ differs by giving noticeably lower performance on GBM and GLI). 
These findings suggest that,
for \Nice\ datasets, 
more complex models (\MTLR, \RSFKM , and \CoxENKP) do not offer large benefits in terms of Concordance. For the \HD\ datasets, \MTLR\ and \CoxENKP\ performed only marginally better than \RSFKM\ for DBCD
but noticeably better than \RSFKM\ on DLBCL. 
Although these are only two datasets, 
this 
suggests % may be an indicator '
that \RSFKM\ may not be optimal for these high-dimensional datasets, in terms of Concordance. For the \HC\ datasets \RSFKM\ saw much worse performance for Concordance (among other metrics) suggesting \RSFKM\ may not be suitable for datasets with a high proportion of censored data.

%For the High-Dimensional datasets, mtlr performed  only  marginally  better  than rsf-km and coxen-kp for DBCD but noticeably better on DLBCL. Although these are only two datasets, this suggests that mtlr does a better job on these high-dimensional datasets, in terms of Concordance.

As noted above, Concordance is only one measure for an \ISD\ tool.
% We reiterate that Concordance is not an ideal measure for an \ISD\ tool since it is an irrelevant metric for an individual patient. 
Given that an \ISD\ tool can produce 
a survival curve for each patient
(and not just a single real-valued score),
it can be used for various tasks,
with various associated evaluations.
For example, 
consider % suppose individual survival curves are generated for 
patients who are deciding whether to undergo an intensive medical procedure.
Using % \note[HH]{Can we specify the misuse?} 
the plots from Figure~\ref{fig:4subjects},
note that Patient~C 
has a very steep survival curve with a low median survival time,
while Patient~A has a shallow survival curve with a large median survival time.
If we were to use this to predict the outcome of a procedure,
we might expect
Patient~C % that Patient~C  may choose 
to opt-out of the procedure,
but Patient~A to go through with it.
% while Patient~A. may want to go through with the procedure.
Note the decision for Patient~C is completely independent of Patient~A,
in that we could give the procedure to one, both, or neither of them.
As these patients are 
not being ranked for a limited procedure, 
Concordance is not 
an appropriate % a relevant 
metric and instead we need to evaluate such predictors using a calibration score%
\comment{
\note[RG]{How useful is this? Delete?}
To make this calibration versus discrimination distinction more tangible:
Imagine you want a fast car,
and you have a tool that provides a value for estimating the ``fastness'' of 
each car.
Sometimes, this fastness value is only designed to be relative,
to predict whether Car A will beat Car B in a race.
It makes sense to evaluate this fastness estimate based on how well it predicts the outcomes of all of the 2-car races.
This is a ``discriminative'' measure -- such as Concordance.
By contrast, other fastness measures actually estimate the speed of a car.
Such measures are very important if you want to purchase a car that can go at least 100~km/hour.
We should use calibration to evaluate such a fastness measure.
}
 -- perhaps % which can be measured with
 1-Calibration or D-Calibration, 
 as discussed in Sections~\ref{sec:1-Calib} and~\ref{sec:d-calibration}.

As discussed in Section~\ref{sec:1-Calib}, 
1-Calibration is particularly relevant for \SAc{P}{\OneOne}{i}\ models --
\ie models that produce a probability score for only 1 time point (for each patient). 
We also noted that % Nonetheless,
\ISD\ models,
that produce individual survival curves,
can also be evaluated using 1-Calibration,
once the evaluator has identified the relevant specific time \(\tzero\).
Here, we evaluated a variety of time points:
% Our evaluations chose a variety of time points, 
% selecting 
the 10th, 25th, 50th, 75th and 90th percentiles of survival times for each dataset. 
We found \MTLR\ to be %strictly
superior to all the models 
considered here
for all percentiles.
%,other than the 10th  where \RSFKM\ matched \MTLR. 
\comment{ By passing 1-Calibration for a range of time points 
(over % let alone for
a large number of diverse datasets), 
}%
The observation that
\MTLR\ was 1-Calibrated for a range of time points,
across a large number of diverse datasets, 
suggests % this evidence is suggestive 
that the probabilities assigned by \MTLR's survival curves 
are representative of the patients' true survival probabilities;
the observation that the other models were not 1-Calibrated as often,
calls into question their effectiveness here. 

% considerably outperform the other models 
% considered suggesting the probabilities assigned by \MTLR\ are more representative of the patients true survival probabilities.

% \note[RG]{re-wrote -- ok?  (Earlier version commented below.)}\note[HH]{Looks good.}
Of course, our analysis is performing the 1-Calibration test for 5 models (\KM\ is excluded) across 8 datasets and 5 percentiles,
meaning we are performing 200 statistical tests.
We considered applying some \(p\)-value corrections
-- \eg the Bonferroni correction --
to reduce the chance of ``false-positives'',
which here would mean declaring a model that was truly calibrated, as not.
However, the actual \(p\)-values (see Appendix~\ref{app:1-CalDetails}) 
show that including these corrections would actually benefit \MTLR\ the most, 
further strengthening the claim that \MTLR\ has excellent 1-Calibration performance.

\comment{
A criticism of our evaluation may be that by performing the 1-Calibration test for 6 models across 8 datasets and 5 percentiles,
we are then performing 240 statistical tests 
without any \(p\)-value corrections, \eg the Bonferroni correction.
Recall these correction are often included to avoid ``false-positives'' where a false-positive for 1-Calibration would indicate a model having a lack of fit when the model is truly calibrated. 
By examining the exact \(p\)-values in Appendix~\ref{app:1-CalDetails}, 
one can see that including these corrections would actually benefit \MTLR\ the most, 
thus
strengthening the claim % making the conclusion 
that \MTLR\ has 1-Calibration performance.
% the best % better performance across 1-Calibration even stronger.
}

Our D-Calibration results further support 
the use of \MTLR's individual survival curves 
over other \ISD-models,
% is given by the D-Calibration results, which showed that % where
by showing that
\MTLR\ was the only \ISD-model to be D-Calibrated for all datasets.
(Recall that \KM\ is technically not an \ISD\ since it % \KM\
gives one curve for all patients.) 
%
% \add[HH]{Note that while \hbox{\KM} and \hbox{\ISD} models give different survival curves there is an added level of difficulty to maintain D-Calibration when modelling meaningful,individual curves based on feature as opposed to a general, population level survival curve.}
We see that different \ISD-models are quite different for this measure
% Given this distinction, not all \ISD{}s are equivalent 
-- 
\eg \AFT\ and \CoxKP\ produce significantly worse performance for D-Calibration,
being D-Calibrated for only three datasets. 
As discussed in Section~\ref{sec:Empirical Results},
% but reiterated here for emphasis, 
\AFT\ is a completely parametric model,
which 
means it cannot produce different shapes
(see Figure~\ref{fig:AllCurves}[top-right]),
% can hinder the ability of the model to shape curves differently,
likely impacting its ability to be D-Calibrated. 
(Our analysis showed only that \AFT[Weibull] is here not D-Calibrated;
\AFT[$\chi$] for some other distribution class $\chi$,
might be D-Calibrated for more datasets.)
% Since Weibull is the only parametric form considered, 
% it may be the case where a different distribution could show to be D-Calibrated for some of the datasets which the Weibull form failed. 
%While \RSFKM\ performed poorly for D-Calibration
%in general,
%note that it was D-Calibrated (only) for the two
%\HD\ datasets
%  -- 
%suggesting that \RSFKM\ may be able to generate effective survival curves when there are a large number of features (and likely interactions between those features). %  in the data.

% one should note that the two datasets in which \RSFKM\ was D-Calibrated were the high-dimensional datasets -- 
% indicating that while in general \RSFKM\ does not supply calibrated survival curves, \RSFKM\ is able to generate effective survival curves when there are a large number of features (and likely interactions between those features) in the data.

In addition to discussing discrimination (Concordance) and calibration (1-Calibration, D-Calibration) separately,
we can also consider 
%Up to now  we have discussed discrimination (Concordance) and calibration (1-Calibration, D-Calibration) separately, we can consider 
a hybrid evaluation metric
-- the Integrated Brier score --
which measures a combination of both calibration and discrimination -- 
see Section~\ref{sec:BrierScore} and Appendix~\ref{app:brier}. 
We see \MTLR\ performing the best for seven of the eight datasets, however, \MTLR\ is no longer superior for DBCD, one of the high-dimensional datasets, 
even though % as 
it was 
superior
for Concordance. 
Instead, \CoxENKP, \RSFKM, and \MTLR\ all perform nearly identical for these \HD\ datasets.
%supporting the idea that \RSFKM\ does well in terms of calibration for high-dimensional datasets
%-- here, so well that it overcomes the lower discrimination (Concordance) score
%it has received for those two datasets.
%However, for the other datasets, we see \RSFKM\ at a clear disadvantage, 
%typically scoring worse than all other \ISD-models.
% , other than \KM. 

The Integrated Brier scores,
along with 1-Calibration and D-Calibration results,
collectively show \MTLR\ outperforms other models (for calibration),
and is followed by \CoxENKP\ and \RSFKM. Specifically, \CoxENKP\ and \RSFKM\ are competitive to \MTLR\ for \HD\ datasets --
the 1-Calibration metric 
shows % identified
that both \CoxENKP\ and \RSFKM\ match the performance of \MTLR\ for DLBCL (\CoxENKP\ and \MTLR\ are 1-Calibrated across all percentiles and \RSFKM\ is 1-Calibrated across three of five, though \(p\)-values are very close to the 0.05 threshold for the other two). DBCD appeared to be the more challenging \HD\ dataset -- \MTLR\ and \CoxENKP\ were 1-Calibrated for two of five percentiles and \RSFKM\ was 1-Calibrated for one. This, coupled with the findings for Integrated Brier Score and D-Calibration,
suggest that \CoxENKP, \RSFKM\ and \MTLR\ are equally competitive for
modeling individual patients' survival probabilities 
\textit{when dealing with % in the presence of
a large number of features}. However, this does not apply to
smaller-dimensional datasets.
\comment{
\RSFKM\ 
is % shows to be
competitive for the \HD\ datasets --
the 1-Calibration metric 
shows % identified
that \RSFKM\ matches the performance of \MTLR\ for DLBCL and actually outperforms \MTLR\ for DBCD on the 25th percentile (although \(p\)-values are very close to the 0.05 threshold)
and otherwise matches performance for DBCD
(see Appendix~\ref{app:1-CalDetails}). 
This, coupled with the findings for Integrated Brier Score and D-Calibration,
suggest that \RSFKM\ and \MTLR\ are equally competitive for
modeling individual patients' survival probabilities 
\textit{when dealing with % in the presence of
a large number of features}. 
However, this does not apply to
smaller-dimensional datasets.
}

\RSFKM\ was not 
1-Calibrated across any percentiles for GBM, GLI, BRCA, and
only 1-Calibrated at the 10th percentile for READ,
and was not D-Calibrated for GBM and GLI.
This, along with the poor performance of \RSFKM\ for all measures of GBM, GLI, READ, and BRCA
suggests that \RSFKM\ does not produce 
effective individual survival curves for low-dimensional % feature
datasets. 
Other experiments (not shown) suggest that
\hbox{\RSFKM} tends to 
% \annote[RG]
{overfit to the training set}
% {? How? Underfit??}
when given too few features.  
Additional meta-parameter tuning in these experiments was unable to correct for overfitting.
% \add[HH]{Based on further experiments, \hbox{\RSFKM} tended to overfit to the training set when given too few features (data not shown). Additional meta-parameter tuning in these experiments was unable to correct for overfitting.}

Given that survival prediction looks very similar to regression,
it is tempting to evaluate such models using measures like L1-loss
(which can lead to models like censored support vector regression~\cite{shivaswamy2008support}).
% Discrimination and calibration are the two primary evaluation metrics seen in survival analysis,
% but we all consider the regression approach to survival analysis and evaluated models on the L1-loss.\note[HH]{Russ -- did you want to mention censored support vector regression?}
% For clinicians, the L1-loss can measure the usefulness of a point estimate of survival time. 
A small L1-loss shows that a model can help with many important tasks,
such as decisions about hospice, 
and for deciding about various treatments, 
based on their predicted survival times.
% \note[RG]{Issue: how to handle censored instances? Also, what point estimate to use }
However, simply because a model has the best performance for L1-loss does not mean the estimates are useful
-- consider the complete NACD dataset where \MTLR\ has the best performance with an average L1-loss of 43.97 months.
While this is the lowest average error,
predicting the time of death up to an error of 43.97 months ($\approx$3.7 years) 
is likely not helpful to a patient,
especially as the maximum follow-up time was 84.3 months.

While the best model may not represent a ``good'' model, our empirical results still showed \MTLR\ had the lowest L1-loss on six of eight datasets,
although all \ISD\ models performed 
comparably
for the four \Nice\ datasets 
(with the exception of \RSFKM). 
We see that \KM\ 
is also competitive % has very competitive performance
for the \HC\ datasets, 
but given the construction of the L1-Margin loss,
this is not surprising;
see Appendix~\ref{app:L1-loss}.
Moreover, the three complex models \hbox{(\CoxENKP, \RSFKM, \MTLR)} 
appear comparable for the \hbox{\HD\ datasets.}

We also compared the models in terms of 
``Uncensored L1-loss'',
which just considers the loss on the uncensored instances;
see Table~\ref{tab:L1Unc} in Appendix~\ref{app:L1LossEmpirical}.
We see {\KM} is no longer competitive for the \HC\ % high censor
datasets, showing how influential this effect is.
%In fact, \KM\ is not competitive for any datasets. 
%Note from humza: Actually KM is beating everyone in DLBCL but after looking into it, it looks like this is due to some strong correlation between censoring and time in the dataset... I can explain in person if needed.
Instead, at least one of the complex models
\{\CoxENKP, \RSFKM, \MTLR\} 
outperforms {\AFT} and {\CoxKP} for every dataset.

%\note[RG]{1. can you plot this -- show the Unc-L1-loss for each.\\
%2. indicate the \# of uncensored instances for each}
% \note[HH]{Russ, what did you want to mention about the Log-L1 loss? I have included it in the Appendix for now.}
That appendix also motivates and defines the Log L1-loss, and
its Table~\ref{tab:LogL1} shows that \MTLR\ 
performs best in 4 of the datasets, and is either second or third best in the others.

\hding{Which \ISD-Model to Use?}
As shown above, which \ISD-model works best depends on properties of the dataset, and on what we mean by ``best''.
Table~\ref{tab:BestISD} summarizes our results here.

In general, for \Nice\ datasets,
\MTLR\ was superior for calibration
but for discrimination,
all \ISD-models were equivalent, 
 leading us to recommend using the
 simplest models: (\CoxKP, \AFT).
As we found that \RSFKM\ would {overfit} the training data
when the 
% \annote[RG]
number of features was small (here, less than 34),
% {shouldn't this depend more on RATIO of \#features/\#instances??},
we recommend avoiding \RSFKM\ when there are so few features.
% When the {number of features was small (here, less than 14),}
%\RSFKM\ would overfit the training data,
% leading us to recommend avoiding \RSFKM\ when there are so few features.
% \note[RG]{Not sure how to parse next sentence.
% Do you mean\\ use MTLR /CoxEN ONLY for few features?\\ Or\\ use MTLR /CoxEN in general 
%(and also RSFKM, if there are large number of features)

For \HC\ datasets,
we recommend \MTLR\ or \CoxENKP\ when there are not many features (\eg READ, BRCA)
% and also \RSFKM\ otherwise,
for both calibration and discrimination. 
%\note[RG]{Does MLTR, CoxEN also work for high-dim?}
Typically \CoxKP\ and \AFT\ 
had poor performance and high variability for \HC\ datasets. 
For \HD\ datasets 
with low censoring (less than 70\% \ie DLBCL),
\MTLR, \CoxENKP, and \RSFKM\ 
had the best
performance
for calibration. 
For discrimination, \RSFKM\ 
seemed slightly worse % saw a relatively better performance
for Concordance and Brier score, suggesting it may be a weaker model.

% \change[HH]{Of course, these results, in general, are based on a small number of datasets; they should all be tested on  a wider variety of datasets to see if these claims continue to hold.}
To explore whether examine if 
these findings hold in general, % ized, 
we examined 33 other public datasets --
16 (Low Dimension, Low Censoring), 
12 (Low Dimension, High Censoring),
4 (High Dimension, Low Censoring) and 
1 (High Dimension, High Censoring)
where High Censoring is $\geq$ 70\%. 
Note that all Low Dimensional datasets were taken from the TCGA website whereas the other (High Dimensional)
datasets arise from a variety of sources.
The results from these 33 datasets are consistent with the findings
reported here;
specific results can be found on the lead author's RPubs site%
\footnote{See http://rpubs.com/haiderstats/ISDEvaluationSupplement
}. 
Given the low overall number of \hbox{\HD}  datasets,
these findings should be examined on further datasets.

\begin{table}[tb] % [H]  
\centering
\caption{\label{tab:BestISD} 
Our recommendation for % A comparison of
\ISD\ models, for different types of datasets. 
Note we divide the \HD\ set into Low versus High censoring.
(DBCD is 73.22\% censored.) }   
\resizebox{\textwidth}{!}{
 \begin{tabular}{ccc|cc}
     \rowcolor{gray!50}  \multicolumn{2}{c}{Characteristic of Dataset}& Applicable Datasets& \multicolumn{2}{c}{Evaluation} \\
      \hline
      \rowcolor{gray!50}   \%Censored  & \#Dimensionality   & Name & Calibration & Discrimination\\
      \hline
      Low  & Low   &  GBM, GLI, \NacdCol, NACD & \MTLR & \CoxKP /\AFT \\
      High   & Low  &  READ, BRCA & \MTLR/\CoxENKP & \MTLR/\CoxENKP \\
      Low   & High  &  DLBCL & \MTLR/\CoxENKP/\RSFKM & \MTLR/\CoxENKP\\
      High  &  High & DBCD & \MTLR/\CoxENKP/\RSFKM &  \MTLR/\CoxENKP/\RSFKM\\
    \end{tabular}
}
\end{table}

\hding{Why use \ISD-Models?}
As noted above,
this paper considers only models that generate \ISD{}s
(\ie \SAc{P}{\infty}{i}).
This is significantly different from models that only generate
% as opposed to 
risk scores  (\SAc{R}{\OneAll}{i}), 
as those models % This distinction is important, as models that generate only risk scores
can only be evaluated 
using a discriminatory metric.
While this discrimination task (and hence evaluation) 
is helpful for some situations
(\eg when deciding which patients should receive a limited resource),
it is not helpful for others 
(\eg deciding whether a patient should go to a hospice,
or terminate a treatment). % , etc.).
% This distinction is an important one, models which generate risk scores can only be evaluated on a discriminatory metric which may be helpful for clinicians, \eg deciding which patients will receive a limited number of treatments, but meaningless for the patients themselves. 
A patient's primary focus will be on his/her own survival, 
not how they rank among others
-- hence 
the risk score such models produce
do not meaningfully inform individual patients.
%\note[RG]{Changed from "benefit" above as:
%the patient whose rank meant he could get a liver, would benefit from this...}

%\note[RG]{rewrote material below. Ok?}
The single point probability models, \SAc{P}{\OneOne}{i}, 
are a step in the % right 
direction for benefiting patients,
but they are still often inadequate,
as they apply only to a single time-point.
While hospital administrators may want to know about 
specific time intervals (\eg $\tzero=$``30-day readmission'' probabilities),
medical conditions seldom, if ever, are so precise.
This is problematic as these probabilities can change dramatically
over a short time interval
-- \ie whenever a survival curve has a very steep drop.
For example, consider Patient \#5 ($P5$) % s 6,7, and 10 
in Figure~\ref{fig:AllCurves} 
for the \MTLR\ model. % various \ISD\ models. % 's subfigures. 
Here, we would optimistic about this patient 
if we considered the single point probability model 
at $\tzero\!=\,$6months, as 
$\estCP{MTLR}{P5}{\hbox{6months}} = 0.8$,
but very concerned if we instead used $\tzero\!=\,$12months,
as 
$\estCP{MTLR}{P5}{\hbox{12months}} = 0.3$.
Note this trend holds for the other \ISD-models shown;
and also for many of the patients, including 
$P6,\ P7,\ P10$.

\comment{ Earlier version:
Consider a survival curve that 
has a very steep drop, 
such as patients 6,7, and 10 in Figure~\ref{fig:AllCurves} 
for the various \ISD\ models. % 's subfigures. 
A single point probability model for these patients would look optimistic for 6 months
-- ranging from 50\% to 80\% chance of survival --
however, by 12 months these same patients would only survive with probabilities ranging from 5\% to 35\%.
}

This suggests a model based on only a single time point 
may lead to inappropriate decisions for a patient.
Note also that such a model might not even provide consistent relative rankings over a pair of patients 
-- \ie it might provide different % inaccurate
discriminative conclusions.
Consider patients $P2$ and $P9$ in Figure~\ref{fig:AllCurves}[\MTLR].
Here, at $\tzero\!=\,$20months, we would conclude that
the purple $P9$ is doing worse (and so should get the available liver),
but at $\tzero\!=\,$30months, 
that the orange $P2$ is more needy.
(We see similar inversions for a few other pairs of patients in \MTLR,
and also for several pairs in the 
\RSF\ model.)

% is crucial for patients who may have steep drops in survival over time
% -- relying on a single point to decide on a patient's treatment route can lead to inappropriate decisions.
% \add[HH]% {Furthermore, the areas involving steep drops in survival probabilities may not be known ahead of time so (include commented out text when out of /add )

\comment{
Furthermore, as we typically do not know which times will 
involve these steep drops in survival probabilities, 
we will not know which specific $\tzero$ time
the \SAc{P}{\OneOne}{i}-model should use
-- nor even the set of times.

\note[RG]{Doesn't this just argue we should have evaluate at 12 months, not 6?
Or have both 6 and 12 months?\\
Need to explain why we need yet MORE points?}\note[HH]{Yes... but how do we know ahead of time that we need to do this? Building the curves do this by default.}

\note[RG]{Perhaps also discuss curves that CROSS -- 
eg, for MTLR, perhaps Orange and Purple (? 3 and 9?) --
and note that at time t1=15, \#3 is higher, but at time t2=40, \#9 is higher...?} \note[HH]{Is this a point against 1 -risk score models or...?}
}

Of course, one could argue that we just need to use
multiple single-time models.
Even here, we would need to {\em a priori}\ specific the 
set of time points 
-- should we use 6 months and 12 months,
and perhaps also 30 months?
And maybe 20 months?

This becomes a non-issue if we use
individual survival distribution (\ISD; \SAc{P}{\infty}{i}) models,
which produce an entire survival curve,
specifying a probability values for every future time point.
% which give each patient his/her entire survival curve,
% which lets that patient understand his/her individual chance of survival across time, 
% as opposed to a single point or simply ranking them among others.
Moreover, 
while risk score models can only be evaluated using a discrimination metric,
these \ISD\ models can be evaluated using all metrics,
making them an overall more versatile method 
for % of approaching 
survival analysis.
\comment{ Moved below:
Even when evaluating \ISD\ models discriminatively (using Concordance), 
the risk scores we advocate
(mean/median survival time)
% have tangible value, with a  
% -- the mean/median survival time used as a patient's risk score
has meaning to clinicians and patients,
whereas a general risk score, in isolation,
has no clinical relevance.
}

Bottom line: In general, a survival task is based on both a dataset,
and an objective, corresponding to the associated evaluation measure.
Our \ISD\ framework is an all-around more flexible approach,
as it can be evaluated using any of the 5 measures discussed here 
(Section~\ref{sec:Eval})
% Can use ISD for wide range of objectives, as it can be evaluated using 
-- both commonly-used and alternative.
% to either doctor or patient.
Importantly, when evaluating \ISD\ models discriminatively (using Concordance), 
the risk scores we advocate
(mean/median survival time)
% have tangible value, with a  
% -- the mean/median survival time used as a patient's risk score
have meaning to clinicians and patients,
whereas a general risk score, in isolation,
has no clinical relevance.
Moreover, the resulting survival curves are easy to visualize,
which adds further appeal.
% \note[RG]{due to the inherent benefits to individual patients}

\section{Conclusion}  %6
\label{sec:conclusion}

\hding{Future Work}

This paper has focused on the most common situation for survival analysis: 
where all instances in the training data are described using a fixed number of features 
(see the matrix in Figure~\ref{fig:LearnISD}),
there is no missing values,
and each instance either has a specified time of death,
or is right-censored -- 
\ie we have a lower bound on that patient's time of death.
There are many techniques for addressing the
first two issues
-- 
such as ways to ``encode'' a time series of EMRs 
as a fixed number of features,
or using mean imputations.
There are also relatively easy extensions to
some of the models (\eg \MTLR)
to handle left-censored instances (where the dataset specifies
an upper-bound on the patient's time of death), or interval censored.
These extensions, however, are beyond the scope of the current paper.

\hding{Contributions}

This paper has surveyed 
several % the
different approaches to survival analysis,
including assigning individualized risk scores
\SAc{R}{\OneAll}{i},
assigning individualized survival probabilities
for a single time point
\SAc{P}{\OneOne}{i},
modeling a population level survival distribution,
\SAc{P}{\infty}{g},
and 
primarily \ISD\ (individual survival distribution;
\SAc{P}{\infty}{i}) models. 
We  discussed the advantages of having an
individual survival distribution for each patient,
as this can help patients and clinicians to
make informed decisions about treatments, 
lifestyle changes, and end-of-life care. 
We discussed how \ISD\ models can be 
used to compute % incorporate
Concordance measures for discrimination and L1-loss,
but should primarily be evaluated using calibration metrics 
(Sections~\ref{sec:1-Calib}, and~\ref{sec:d-calibration}) 
as these measure the extent to which the individual survival curves
represent the ``true'' survival of patients.

Next, we identified various types of \ISD-models,
and empirically evaluated them over a wide range of survival datasets
-- over a range of \#features, \#instance and \%censoring.
This analysis showed that
% In turn, we demonstrated on a variety of survival datasets that of the \ISD s considered, 
\MTLR\ was typically superior for the L1-loss, Integrated Brier score, and Concordance, 
but most importantly,
showed it outperformed or matched all other models for the calibration metrics. 

In conclusion, this paper explains why we encourage researchers, and practioners,
to use \ISD-models (and especially ones similar to \MTLR) 
to produce meaningful survival analysis tools,
by showing how this can help 
patients and clinicians make informed healthcare decisions.

%  survival analysis will have far-reaching benefits, specifically regarding patient-clinician interactions to make informed healthcare decisions.
 
% Putting emphasis on \ISD\ models for future work in survival analysis will have far-reaching benefits, specifically regarding patient-clinician interactions to make informed healthcare decisions.

% \note[BH]{Just noticed that our references 6 and 7 seem to be the same.  The title of one (in the Bib file) is "GailModel1989", although the duplicated paper is from 1999.  Is this supposed to be a different paper (Gail is an author)?}

\section*{Acknowledgements}
We gratefully acknowledge funding from
NSERC,
Amii,
and Borealis AI (of RBC). We also thank Adam Kashlak for his insightful discussions regarding D-Calibration.

{ \small
\bibliographystyle{abbrv}
\bibliography{survival}
}

\appendix
\newpage

\section{Extending Survival Curves to 0} % ??
\label{app:SC-to-0}
In practice, survival curves often stop at a non-zero probability
-- see Figure~\ref{fig:AllCurves} and Figure~\ref{fig:linearCurve}[left] below.
This is problematic as it means they do not correspond to 
complete distribution (recall a survival curve should be 
``$1-$CDF(t)'',  
where CDF is the Cumulative Distribution Function)
which leads to problems for many of the metrics,
as it is not clear how to compute the mean, or the median, value of the distribution. 
One approach is to extend each of the curves, horizontally,
to some arbitrary time 
and then drop each to zero 
(the degenerate case being dropping the survival probability to zero at the last observed time point). 
This approach has downsides:
Dropping the curve to zero at the last observed time point
produces curves whose mean survival times 
are actually a lower bound on the patient's mean survival time,
which is likely too small.
In the event that the last survival probability is above 0.5 (as is often the case for highly censored datasets) this may bias our estimate of the L1-loss,
which is based on the median value.
\comment{
-- by dropping the curve to zero at the last observed time point
the mean survival times (recall that we are computing the loss with respect to the model's mean time as its prediction) will likely result in mean survival times that are too small, 
that is by dropping the curve to zero we are predicting a lower bound on the patients mean survival time -- potentially biasing the L1-loss. }
Alternatively, 
if we instead extend each curve to some arbitrary time and 
then drop the curve to zero, 
we need to 
decide on that extension,
which also % make a choice for each dataset encountered which 
could bias the L1-loss. %  depending on the choice.

% Often, after a survival model has been trained, survival probabilities do not extend to zero within the survival times used for training -- see Figure~\ref{fig:AllCurves} above and Figure~\ref{fig:linearCurve} [left] below. Two standard approaches are to (1) drop the curve to zero at the last predicted time or (2) extend the curve to infinity. Both approaches have downsides -- by dropping the curve to zero the mean survival times (recall that we are computing the loss with respect to the model's mean time as it's prediction) will likely result in mean survival times which are too small, that is by dropping the curve to zero we are predicting a lower bound on the patients mean survival time -- potentially biasing the L1-loss. Alternatively, by extending the curve to infinity all mean survival times will be infinite, making both discrimination and L1-loss measures meaningless.

Since both standard approaches have clear downsides 
(and there is no way of knowing how the survival curves act beyond the sampled survival times),
we chose to 
simply
extrapolate survival curves using a simple linear fit:
for each patient $\inst{i}$,
draw a line from (0,\,1) 
-- \ie time is zero and survival probability is 1 --
to the last calculated survival probability,
(\(t_{max},\,
\estCP{}{t_{max}}{\inst{i}}\)),
then extend this line to the time for which survival probability equals 0
--  \ie \((t^0(\inst{i}),\, 0)\) 
% A linear curve is fit from the (0,1) point, \ie time is zero and survival probability is 1, to the last calculated survival probability, (\(t_{max},\hat{S}(t_{max})\)) and extended to the 0 survival probability point, \ie \((t^0, 0)\) 
-- see Figure~\ref{fig:linearCurve}[right].
Note that curves cannot cross within the extended interval,
which means this extension will not change the discriminatory criteria.

\begin{figure}[!htb]
\centering
\includegraphics[width = \textwidth]{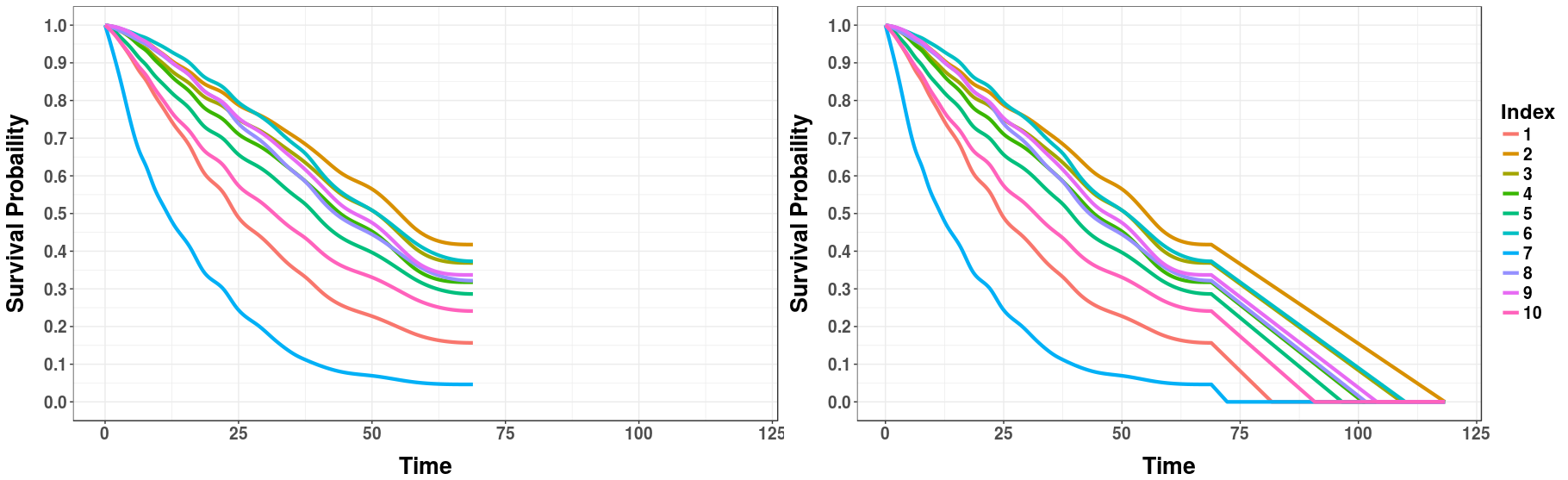}
\caption{\label{fig:linearCurve}
On left, survival curves generated from \MTLR\ for the 
\NacdCol\ % NACD -- Colorectal 
dataset. 
Left shows this model's survival curves end at 68.9 months.
% \MTLR\ is only capable of learning models that produce survival curves for times included in the training data, which here means its curves end at 68.9 months.
On right, linear extensions of those survival curves go as far as 118 months. 
% Note that curves cannot cross within the extended interval -- \ie discriminatory criteria will not be impacted.
}
\end{figure}

There are extreme cases where a survival model will predict
a survival curve with survival probabilities of 1 (up to machine precision) for all survival times 
(think ``a horizontal line, at $p=1$'')
--
this occurred for unregularized models on high-dimensional datasets. 
In these cases,
this linear extrapolation will never reach 0. 
% a linear fit will not reconcile the issue of obtaining infinite mean survival times. 
To address this, 
we fit the 
% \annote[RG]{(within-fold)}{Why mention this?
% Aren't all of the analyzes "within-fold''?}
Kaplan-Meier curve with the linear extension described above
to compute $t^0_{KM}$;
we then replace any infinite prediction with this value.
% was replaced by $t^0_{KM}$.
Additionally, as the Kaplan-Meier curve is to represent the survival curve on a \textit{population} level, 
we also 
% \change[RG]{replaced any predicted mean survival times greater than}
{truncated any patient's median survival time by}
$t^0_{KM}$.

\comment{ \note[RG]{Is this policy really guaranteed:
$t^0_{KM}$ is just the 
max value of the MEAN time over the population;
couldn't some patients have a mean time that is more than this?}\note[BH]{$t^0_{KM}$ is not the mean time for KM, but the MAX time -- all patients are deceased by this time}
%\change[BH]{the maximum survival time generated by the Kaplan-Meier curve with the linear extension, \ie denoting the mean survival time generated by a learned survival curve as 
%\(\hat{\mu}\), we had \(\hat{\mu} = \min{\{\hat{\mu}, t_{KML-max}\}}\), where \(t_{KML-max}\) was the maximum survival time generated by the Kaplan-Meier curve with the linear extension}
}

\section{Evaluation Measures Supplementary Information} % A
\label{app:Evaluation}
This appendix provides additional information about the various evaluation measures.

\comment{\note[BH]{Below, would it aid readability if we substituted "observed" in many places that we have "uncensored"?}}

%\subsection{Other Information about Concordance} % A.1
\subsection{Concordance} % A.1
\label{app:Concordance}

As discussed % Previously described 
in Section~\ref{sec:Concordance}, 
Concordance 
is designed % aims
to measure the discriminative ability of a model. 
This is challenging for 
% Challenges for computing Concordance are introduced in the event of 
censored data. 
For example, suppose we have two patients who were censored at \(t_1\) and \(t_2\). 
Since both patients were censored,
there is no way of knowing which patient died first and
hence the risk scores for these patients are incomparable.
However, 
if one patient's censored time is later than the death time of another patient, we do know the true survival order of this pair: the second patient died before the first.

To be precise, we first need to define the set of \textit{comparable pairs}, 
which is the subset of pairs of indices
(here using the validation dataset (\(V\)) and recalling that \(\delta = 1\) indicates a patient who died (uncensored)) containing all pair of instances when we know which patient died first:
\begin{equation}
\CP{V}\ =\ \left\{ [i,j]\ \in\ V \times V\  
{\left|\ 
t_i < t_j\ \hbox{and}\ 
\delta_i = 1\
\comment{\begin{array}{l}
\delta_i = 1\  \hbox{and}\ \delta_j = 1 \quad \hbox{or}\\
\delta_i = 1\  \hbox{and}\ \delta_j = 0\ 
\end{array} }
\right\}} \right.
\end{equation}
%\note[RG]{Changed to just $\delta_i = 1$}
Notice when the earlier event is uncensored (a death),
we know the ordering of the deaths
(whether the second time is censored or not)
\comment{ That is, either both 
{events are observed, or 
one is censored while the other is observed,}
and the censored event occurs after the observed event }
-- see Figure~\ref{fig:ConcordanceCensored}.
The $t_i<t_j$ condition is to prevent double-counting such that
$| \CP{V} | \leq \binom{|V|}{2} $.

\def\est#1#2{r(\inst{#2})}

\begin{figure}[tb] % [ht] % Fig 12
\centering
\includegraphics[width=0.6\textwidth] % ,height=2in]
{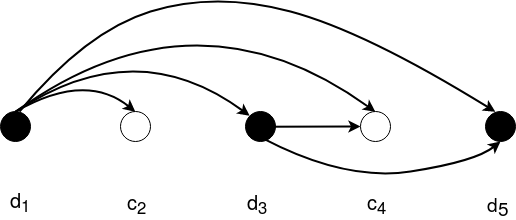}

% \note[RG]{Can we have a more interesting example -- with say patient \#2 and \#3 being censored (not #5)? (Having last patient being censored is not interesting.)}\note[HH]{Do you like this more?}
\caption{\label{fig:ConcordanceCensored} 
Depiction of Concordance comparisons,
including censored patients. 
Black and white circles indicate uncensored and censored patients, respectively. 
Each $d_i$ is the death time for an uncensored patient, 
and each $c_j$ is the censoring time for a censored patient.
We can only compare:
uncensored patients who died \textit{prior} to a 
censored patient's censoring time,
or an uncensored patient's death time.
Here, time increases as we go left-to-right; hence $\death{1} < \censor{2} < \death{3} < \censor{4} < \death{5}$.
Here, we can compare 6 of the ${5 \choose 2} = 10$ pairs of patients. 
Figure adapted from \cite{wang2017machine}. }
\end{figure}

We then consider how many of the possible pairs our predictor put in the correct order:
That is, of all $[i,j]$ pairs in $\CP{V}$, we want to know how often
$\est{}{i} > \est{}{j}$ given that $t_i < t_j$.
Hence, the Concordance index of $V$, with respect to the risk scores, $r(\cdot)$, is
\begin{equation}
\hat{C}(\,V, \ r(\cdot)\,) \quad =\quad 
\frac{1}{|\CP{V}|} \sum_{i: \delta_i = 1}\, \sum_{j:\, t_i < t_j} \Indic{r(\inst{i}) > r(\inst{j})}.
\label{eqn:CI}
\end{equation}

One issue is how to handle ties, in either
% An additional point to consider for Concordance is the
% handling of ties of both 
risk scores or death times 
-- \ie for two patients, Patient~A and Patient~B,
consider either \(r(\inst{A}) = r(\inst{B})\) or \(d_A = d_B\). 
The two standard approaches % to handling ties 
are  
(1)~to give the model a score of 0.5 for ties (of either risk scores or death times), or 
(2)~to remove tied pairs entirely~\cite{yan2008investigating}.
The first option % Choosing the first option, to give models a score of 0.5 for ties 
is equivalent to Kendall's tau, 
while the second leads to % whereas removing ties entirely renders equivalence to 
the Goodman-Kruskal gamma. 
The empirical evaluations 
(given in Section~\ref{sec:Empirical Results})
use the first, 
% report the Concordance with ties included 
as this % doing so
gives Kaplan-Meier a Concordance index of 0.5 for all models.
If we use the second option (excluding ties),
then the Concordance for the 
Kaplan-Meier model is not well-defined.

%end commenting out t*-concordance
%\subsection{Other Information about L1-loss} % A.3
\subsection{L1-loss, and variants} % B.2
\label{app:L1-loss}

As discussed in Section~\ref{sec:L1-loss}, 
survival analysis can be 
viewed % treated
as a regression problem
that is % by
attempting to minimize the difference between an estimated time of death and the true time of death. 
However, typical regression problems 
require having precise target values for each instance;
here, many instances are censored 
-- \ie providing only lower bounds for the target values.
% being predicted, % regressed upon 
% rather than a precise target value. 
One option is to simply remove all the censored patients and use the L1-loss given by Equation~\ref{eqn:L1Unc}
(which we call ``Uncensored L1-Loss'');
however, this will likely bias the true loss.
Table~\ref{tab:L1Unc} in
Appendix~\ref{app:L1LossEmpirical} provides
the results for this Uncensored L1-loss 
over the 8 datasets.
(We see that \MTLR\ is best for 6 of these datasets.)

\def\bestGuess#1{B\!G(#1)}
\def\weight#1{\alpha_{#1}}
One way to incorporate % approach for incorporating
censoring is to use the Hinge loss for censored patients,
which  assigns $0$ loss to any patient % $\inst{k}$
whose censoring time $c_k$ is prior to the estimated median survival time,
 $\medi{k}$ -- \ie a loss of 0 if 
 $c_k < \medi{k}$ --
 and a loss of \(c_k - \medi{k}\) 
 if the censoring time is greater than 
 $\medi{k}$. % the estimated survival time. 
 That is: % Mathematically, the L1-Hinge loss can be expressed as
\begin{eqnarray}\small
L1_{hinge}(\,V,\, \{\medi{j}\}_j\,) &=& \frac{1}{|V|}
 \left[\,\,\sum_{j \in V_U} |\death{j} - \medi{j}|  % was k, changed to j
 \ +\ % \!\!\! 
 \sum_{k \in V_C} [\censor{k} - \medi{k}]_+
   \right].
     \label{eqn:L1-hinge}
\end{eqnarray}
where  \(V_U\)
is the subset of the validation dataset that is uncensored,
and \(V_C\) is the censored subset, 
and \([a]_+\) is the positive part of \(a\), \ie 
$$[a]_+ \quad =\quad \max\{a, 0\}\quad =\quad\left\{
\begin{array}{ll}
a & \hbox{if}\ a \geq 0\\
0 & \hbox{otherwise}
\end{array}.
\right.$$
% One can observe that 
This formulation is an optimistic lower bound on the L1-loss for two reasons: 
(1)~it gives a loss of 0 if the censoring occurs prior to the estimated survival time,
implying that \(d_k =\medi{k}\),
% where \(d_k\) is not observable, 
and (2)~it gives 
 a loss of \(c_k - \medi{k}\) if the censoring time occurs after the estimated survival time,
 which assumes that \(d_k = c_k\).
 Both are the best possible values for the unknown $d_k$,
 given the constraints..

One weakness of the L1-Hinge loss is that
if a model predicts very large survival times for all patients (both censored and observed),
the hinge loss will give 0 loss for the censored patients;
%\note[RG]{Why is this true?}\note[BH]{Suppose we predict the maximal time $\tau + \epsilon$ in our dataset - this will be greater than every censored patient's censoring time, and thus we will assign an error of 0 to every censored patient}; 
in datasets with a large proportion of censored patients, this leads to an optimistic score overall. Thus the hinge loss will favor models that tend to largely overestimate survival times as opposed to those models underestimating survival time.

A third variant of L1-loss, 
the {\em L1-Margin loss}, 
assigns a ``Best-Guess'' value to the death time corresponding to \(c_k\),
which is  % This ``Best-Guess'' is
the patient's conditional expected survival time given they have survived up to \(c_k\) -- given by
\begin{equation}
\bestGuess{c_k} \quad =\quad  \censor{k}\ + \
\frac{\int_{\censor{k}}^\infty S(t) % \estCP{}{t}{\inst{k}}
\, dt }
     {S(\censor{k})}
     \label{eqn:BG}
\end{equation}
where \(S(\cdot)\) is the survival function;
Theorem~\ref{thm:conditionalKM} proves this value corresponds to the conditional expectation. 
In practice we use Kaplan-Meier estimate, \(\estP{KM}{\cdot}\), 
generated from the training dataset (disjoint from the validation dataset) 
as our estimate of \(S(\cdot)\) 
in Equation~\ref{eqn:BG}.

We also realized that these $\bestGuess{c_k}$ estimates
are more accurate for some patients, 
than for others.
If $\censor{k} \approx 0$ --
that is, if the patient was censored 
near the beginning time --
then we know very little about the true timing of when the death occurred, so the estimate
$\bestGuess{\censor{k}}$ is quite vague, 
which suggests we should give very little weight to the associated loss, 
 $|\bestGuess{\censor{k}} - \medi{k}|$.
 Letting $\weight{k}$ be the  weight associated with these terms,
we would like $\weight{k} \approx 0$.
 On the other hand, if $\censor{r}$ is large 
 -- towards the longest survival time observed 
 (call it $d_{max}$) --
then there is a relatively narrow gap of time where this $\inst{r}$ could have died
(probably within the small interval $(\censor{r}, d_{max})$);
here, we should 
give a large weight to loss associated with this estimate. 
% weight the loss heavily. 

This motivates us to define % Using the Best-Guess for the L1-Margin loss, we have
\begin{eqnarray}
L1_{margin}(\,V,\, \{\medi{j}\}\,) &=& 
\small \frac{1}{|V_U| + \sum_{k \in V_C} \alpha_k}
 \left[\,\,\sum_{j \in V_U} |\death{j} - \medi{j}|
 \ +\ % \!\!\! 
 \sum_{k \in V_C}
  \weight{k} |\bestGuess{\censor{k}} - \medi{k}|
   \right]\hspace{1cm}  \label{eqn:L1-margin}
\end{eqnarray}
where $\weight{k}$ reflects the confidence in each Best-Guess estimate.
\comment{   Moved this above
If $\censor{k} \approx 0$ --
that is, if the patient was censored at the beginning time --
then we know very little about the true timing of the event, so the estimate
$\bestGuess{\censor{k}}$ is quite vague, and so we give very little weight to the associated loss, 
 $|\bestGuess{\censor{k}} - \mean{k}|$,
 that is, we would like $\weight{k} \approx 0$.
 On the other hand, if $\censor{k}$ is large 
 -- towards the longest survival time observed 
 (call it $d_{max}$),
then there is a relatively narrow gap of time where this $\inst{k}$ could have died
(probably between $(\censor{k}, d_{max})$);
here, we should weight the loss heavily. 
}%
To implement this, 
we set % choose to assign 
$\alpha_k = 1 - \hat{S}_{KM}(c_k)$,
which 
gives little weight to instances with early censor times 
but considers late censor times to be almost equivalent to an observed death time. Note this is the version of L1-loss we presented in 
\hbox{Figure~\ref{fig:L1Evaluation}},
with details in \mbox{Table~\ref{tab:L1}}.

For completeness, 
we prove Equation~\ref{eqn:BG}.
\comment{
include the details showing that $\bestGuess{c_k} =  \censor{k} + 
\frac{\int_{\censor{k}}^\infty S(x)\, dx }{S(\censor{k}) }$.
}%
% While the theorem is reiterated here, 
(This claim is also proven by Gupta and Bradley~\cite{gupta2004representing},
which uses % but is given there in terms of
\textit{mean residual life} rather than
% as opposed to 
{\em expected total life}.) %  that we give below.

\begin{thm}
The conditional expectation of time of death, \(D\), 
given that a patient was censored at time \(c\),
is given by: 
$E[D\,|\, D> c]\ =\  c + \frac{\int_{c}^\infty S(x)\, dx }{S(c)}$.
\label{thm:conditionalKM}
\end{thm}
\begin{proof}
Let $D$ be the r.v.\ for the time when a patient dies,
and define 
$$S(c)\quad =\quad P(D > c) \quad=\quad \int_{c}^{\infty} P(D = t)\, dt$$ 
as the survival function -- \ie the probability that the patient dies after time $c$.
Given this, 
 the conditional probability is %  we have the conditional probability to be:
\def\Pr#1{P(\, #1\,)}
\def\CPr#1#2{\Pr{#1\,|\, #2}}
$$
\CPr{ D=t}{D > c} 
\quad =\quad \frac{\Pr{ D=t, \ D > c} }{\Pr{D > c}}
\quad =\quad \frac{\Pr{ D=t, \ D > c} }{S(\, c\,)}
\quad =\quad 
\left\{
\begin{array}{rl}
 0    &\hbox{if}\ t<c  \\
 \frac{\Pr{ D=t} }{S(\, c\,)}  &\hbox{otherwise}
\end{array}
\right.
% \frac{\Pr{ D=t, \ D > c} }{S(\, c\,)}
.
$$
\begin{eqnarray}
E[ \,D\,|\, D> c\,] &=& \int_{c}^{\infty} t\, \frac{\Pr{D = t}}{S(\,c\,)}\, dt  \nonumber \\
 &=&   \frac{1}{S(\,c\,)} \left[
 \int_{c}^{\infty} c\,\, \Pr{D =t}\, dt
    \quad+\quad {\int_{c}^{\infty} (t-c)\, \Pr{D = t}\, dt} \right]  \nonumber\\
 &= & \frac{1}{S(\,c\,)} \left[c\, S(\,c\,)
    \quad+\quad  \int_{c}^{\infty} \left(\int_{c}^{t}dx\right)\, P(D= t)\, dt\right]  \nonumber\\
 &= &  c \,\,+\quad \frac{1} {S(\,c\,)}\left[\int_{c}^{\infty} \left(\int_{x}^{\infty}P(D= t)\, dt\right)dx\, \right]
   % \tag{1}
     \label{eqn:Tonelli} \\
 &= & c \,\,+\quad \frac{\int_{c}^{\infty} S(\, x\,)\, dx} {S(\,c\,)}  \nonumber.
\end{eqnarray}
\end{proof}
\noindent Step~\ref{eqn:Tonelli} is an application of Tonelli's theorem \cite{saks1937theory},
which lets 
% -- \ie letting 
us swap the order of integration for a non-negative function.
As desired, this quantity, $E[\, D\,|\, D> c\,]$, is always at least $c$.
Moreover, when $c=0$, this is 
$$0 \ +\ \frac{\int_{0}^{\infty} S(\, t\,)\, dt} {1}
\quad=\quad
\int_{0}^{\infty} S(\, t\,)\, dt
\quad=\quad E[\, D\,]
$$
which is the expected value of the distribution for this survival curve
(and exactly the claim of the Theorem).

\subsubsection{Log L1-loss} % B.2.1
\label{app:LogL1-loss}
\def\MrJ{J}
\def\MsS{S}

The L1-loss measure implicitly assumes that the quality of a prediction,
$\medi{j}$,
depends only on how close it is to the truth $\death{j}$-- 
\ie on $|\death{j} - \medi{j}|$.
But this does not always match how we think of the error:
if we predict Patient A will live for 120~months
then found that he actually lived 117~months, 
we would consider our prediction very accurate.
By contrast, if we predict Patient B
will live 1 month, but then find she lived 4 months,
we would consider this to be a poor prediction.
Notice, however,
the L1-loss for Patient~A
is 
$|\death{A} - \medi{A}|\ =\ |120 - 117| =\ $3 months,
which is the same as the L1-loss for Patient~B:
$|\death{B} - \medi{B}|\ =\ |1 - 4|\ =\ $3 months!

This motivates us to consider the {\em relative}\ error, 
rather than an {\em absolute}\ error:
here, as our prediction for Patient A
is off by only 3 / 120 = 2.5\%,
we consider it good, 
whereas our prediction for Patient B
is off by 3 / 1 = 300\%.
The Log-L1-loss reflects this:%
\footnote{
Note that the times mentioned in 
 ``Doc, do I have a day, a week, a month or a year?''
are basically in a log-scale.
% Note ``log-scale'' relates to the standard question:
% ``Doc, do I have a day, a week, a month or a year?''.
}
\begin{equation}
\loss{LogL1}{\death{i}}{\medi{i}}\ =\ |\log(\death{i}) - \log(\medi{i})|
\label{eqn:LogL1-simple}
\end{equation}
To compute the average Log-L1-loss over the dataset $V_U$,
we can use Equation~\ref{eqn:L1Unc}
but using $\log(\death{j})$ rather than $\death{j}$, etc.
To avoid taking $\log 0$, 
we replace 0 with half the minimum, positive death time (see Section~\ref{sec:SubtlePts}).
Table~\ref{tab:LogL1} in
Appendix~\ref{app:L1LossEmpirical} provides
the results here, over the 8 datasets.
(We see that \MTLR\ is best for 4 of these datasets.)

\subsection{1-Calibration} % A.2
\label{app:1-Calib}
To demonstrate the description from Section~\ref{sec:1-Calib},
consider the following example:
If there are $n=50$ patients, then 50/10 = 5 will be in each bin,
and the first bin $\b{1}$ will contain the 5 with lowest predicted probability values,
% -- perhaps \{ 0.32, 0.34, 0.43, 0.43, 0.48 \} --
and the second bin $\b{2}$ will contain the next smallest 5 values,
% -- say \{ 0.55, 0.56, 0.58, 0.61, 0.62 \} -- 
and so forth
-- \eg
\begin{eqnarray*}
\b{1} & =& \{ 0.32,\ 0.34,\ 0.43,\ 0.43,\ 0.48 \}\\
\b{2} & =&\{ 0.55,\ 0.56,\ 0.61,\ 0.61,\ 0.72 \} \\
&\vdots&\\
\b{10} & =&\{ 0.85,\ 0.85,\ 0.86,\ 0.87,\ 0.87 \} \\
\end{eqnarray*}
Now consider the 5 patients who belong to $\b{1}$.
As the average of their probabilities is
$\frac{0.32 + 0.34 + 0.43 + 0.43 + 0.48}{5}\ =\ 0.4$, 
we should expect 40\% of these 5 individuals
to die in the next 5 years -- that is, 2 should die.
We can then compare this prediction ($0.40 \times 5 = 2$)
with the actual number of these $\b{1}$ patients who died.
We can similarly compare the number of $\b{2}$ patients who actually died to the number predicted 
(based on the average of these 5 probability values, 
which here is $0.61 \times 5 = 3.05$), 
and so forth.

In general, we say that the predictor is 
1-Calibrated 
if these $B$ predictions, for the $B = 10$ bins, 
are sufficiently close to the actual number of deaths with respect to these bins. 
Here, we use the Hosmer–Lemeshow statistical test
(given in Section~\ref{sec:1-Calib}) 
to see if the observed results were significant;
repeating Equation~\ref{eqn:hlstat}:
% Recall the Hosmer-Lemeshow test from Section~\ref{sec:1-Calib}, 
\begin{align*}
\reallywidehat{HL}\,(\,V_U,\ 
\estCP{}{\tzero}{\cdot}
\,) 
\quad  = \quad \sum_{j=1}^B 
\frac{(O_j - n_j\,\bar{p}_j)^2}{n_j\,\bar{p_j}\,(1-\bar{p_j})},
\end{align*}
where \(O_j\) is the number of observed events, \(n_j\) is the number of patients, \(\bar{p_j}\) is the average predicted probability, and subscript \(j\) refers to within the \(j\)th  of \(B\) bins. 

\subsubsection{Incorporating Censoring into 1-Calibration} % B.3.1
Survival data typically contains some amount of censoring, making the exact number of deaths for the \(j\)th bin, \(O_j\), unobservable when the bin contains patients censored before \(\tzero\). 
That is, given a censored patient whose censoring time occurred before the time of interest (\(c_i < \tzero\)) the patient may or may not have died by \(\tzero\). 
There are many standard techniques for incorporating censoring~\cite{guffey2013hosmer};
we use the D'Agostino-Nam translation~\cite{dagnostionam}, 
which uses the \textit{within bin} Kaplan-Meier curve in place of \(O_j\). Specifically, the test statistic is given by, 
\begin{align}
\label{eqn:HosmerDag}
\reallywidehat{HL}_{DN}\,(\,V,\ 
\estCP{}{\tzero}{\cdot}
\,) 
% \reallywidehat{HL}_{DN}\,(V, \hat{S}(t^*\,|\,\cdot)\,)
\quad =\quad
\sum_{j=1}^B \frac{(\ n_j\ (1-KM_j(\tzero))\ -\ n_j\,\bar{p}_j)^2}{n_j\,\bar{p_j}\,(1-\bar{p_j})},
\end{align}
where \(KM_j(\tzero)\) is the height of the Kaplan-Meier curve
generated 
by the patients in the \(j\)th bin, evaluated at \(\tzero\). 
We use  \(1-KM_j(\tzero)\) as
% Note the use of \(1-KM_j(\tzero)\) since 
we are predicting the {\em number of deaths}\ 
and not % whereas 
\(KM_j(\tzero)\) 
which instead gives the probability of \textit{survival} at \(\tzero\).
Note also % Another important distinction is 
that \(\reallywidehat{HL}_{DN}\) follows a \(\chi^2_{B-1}\) distribution,
as opposed to the \(\chi^2_{B-2}\) distribution 
for Equation~\ref{eqn:hlstat}.
% given in Section~\ref{sec:1-Calib}~\cite{dagnostionam}. 

%\subsection{Other Information about Brier score} % A.2
\subsection{Brier Score Details} %  B.4
\label{app:brier}
This section supplements the description of the Brier score given in 
Section~\ref{sec:BrierScore}, discussing
(1)~the decomposition of the Brier score into calibration and discrimination components, (2)~the failure of the Integrated Brier score to incorporate the full distribution of probabilities in survival curves, and (3)~how to incorporate censoring into the Brier score.

\subsubsection{Brier Score Decomposition} % B.4.1
\label{app:BS-decomp}
As mentioned in Section~\ref{sec:BrierScore}, the Brier score can be separated into calibration and discriminatory components. The original separations were the the work of Sanders \cite{sanders1963subjective} and Murphy \cite{murphy1972scalar,murphy1973new} and later put into the context of calibration and discrimination (also known as refinement) by DeGroot and Fineberg \cite{degroot1983comparison}. 

Recall the notation and mathematical expression of the Brier score for a set of uncensored instances, \(V_U\),
\begin{align*}
BS\left(\ 
\estCP{}{\tzero}{ \cdot}
% \hat{S}(\tzero| \cdot)
,\ \{ \inst{i} \}\ \right)
\quad= \quad \frac{1}{|V_U|} \sum_{i \in V_U} \left(\Indic{d_i \leq \tzero} - \hat{S}(\tzero| \inst{i})\right)^2.
\end{align*}
To simplify notation, let \(p_i = 
\estCP{}{\tzero}{ \inst{i}}
% \hat{S}(\tzero| \inst{i})
\). 
The separation of the Brier score requires that a discrete, distinct number of predictions exist;
here, assume there are \(K\) distinct values for \(p_k\) for \(k = 1,\ldots K\). 

Further, let \(n_k\) be the total number of patients with \(p_k\)
as their prediction and 
hence \(|V_U| = \sum_{k=1}^K n_k\).
Finally, let \(\lambda_k\) be the observed proportion of patients 
who have died by \(\tzero\) and thus
\((1- \lambda_k)\) is the proportion still alive. 
The separation theorem of the Brier score states that \(BS = C + D\), where \(C\) and \(D\) are nonnegative calibration and discriminatory scores where
\begin{eqnarray}
C & =& \frac{1}{|V_U|} \sum_{k=1}^K n_k(\lambda_k - p_k)^2
 \label{eqn:calibration}
 \\
% \end{align}  and \begin{align}
D &=& \frac{1}{|V_U|} \sum_{k=1}^K n_k\lambda_k(1 - \lambda_k).
\label{eqn:discrimination}
\end{eqnarray}

Note the calibration score, \(C\), is nearly equivalent 
(up to a factor of \(n_k\)) to the numerator of the Hosmer-Lemeshow test (Equation~\ref{eqn:hlstat}). 
However, the Hosmer-Lemeshow test subscript refers to bins whereas here the subscript refers to a distinct value of \(p_k\). 
One can see that \(C\) represents a calibration score as the estimated probabilities, \(p_k\), must be close to the true proportion of deaths,
\(\lambda_k\) in order to have a small score (lower is better). 
In fact, to satisfy \(C=0\), all predictions, \(p_k\) must be equal to \(\lambda_k\) (Equation~\ref{eqn:calibration}). 

There are also similarities between \(D\) and the denominator of the Hosmer-Lemeshow test. However, note Equation~\ref{eqn:discrimination} uses the the true proportion of deaths \(\lambda_k\),
whereas the Hosmer-Lemeshow test uses an estimated value, \(\bar{p}\). 
\comment{
\note[BH]{Doesn't the HL test only use an estimate when you incorporate censored data?}\note[HH]{No... \(\bar{p}\) is based on survival probabilities whereas \(\lambda\) is the true survival probabilities, i.e. this would be equivalent to \(O_i\)}
}
Note that \(D\) has
a ``good'' (low) score 
if all patients associated with a prediction probability \(p_k\)
have the same status 
-- \ie they either all die or are all still alive. 
% \annote[RG]{Hence, $D$ can be viewed as a discriminatory measure.}{ 
% I don't understand. Why does this means that a patient with $p_i$ is more likely to die that one with $p_j$ when $p_i < p_j$?}
 \comment{
Here, \(D\) can be viewed as a discriminatory measure 
since,
for every unique prediction probability \(p_k\), 
a ``good'' (low) score occurs if patients have a similar event
-- \ie they either all die or are all still alive. 
}
To understand why this means $D$ is a discriminatory measure,
consider % It may be easier to see the discriminatory qualities of \(D\) in
the extreme case where \(BS(\cdot,\cdot) = 0\), which means both $D=0$ and $C=0$. 
For \(D = 0\), all patients associated with each probability value
must either be dead by \(\tzero\) or all be alive at \(\tzero\)
-- \ie \(\lambda_k \in \{0,1\}\) for \(k = 1,2\);
note only $K=2$ is possible here. 
In turn, for \(C = 0\), we require \(p_k = \lambda_k\) for \(k = 1 ,2\), 
that is \(p_k \in \{0,1\}\)
-- 
all predictions will be 1 or 0. 
Here we are discriminating perfectly between the patients who have died and the patients who are still alive,
with a model that predicts only 1's or 0's.
% and in turn having our model predict only 1s and 0s. 
%b\note[RG]{I agree with the extreme case, but it still seems a stretch to think of this as discriminatory in general...}
Of course, we should not require a model to 
estimate survival probabilities
to be precisely 1 or 0,
for the same reason that we do not expect 
% unless the time of interest is very early [late] in time or
the learned distribution 
to correspond to the Heaviside distribution 
shown in Figure~\ref{fig:L1=0}.

\subsubsection{Integrated Brier score does not involve % Failure to Incorporate 
the Entire Distribution} % B.4.2
\label{app:IBSWeakness}
At the beginning of Section~\ref{sec:d-calibration},
we claimed the Integrated Brier score (IBS) does not 
utilize % ensure
the survival curves' full distribution of probabilities over all times.
% are utilized. 
For example, on a \KM\ curve, 
we expect that 10\% of patients will die in every 10\% interval
-- \eg 10\% of all patients will die in the [0.5, 0.6) interval. 
While D-Calibration will 
debit % count against
a model that fails to do this,
this Integrated Brier score 
{does not require this}. 
The most obvious example %of this
is the perfect model,
where each patient is given the appropriate Heaviside distribution
(Figure~\ref{fig:L1=0}) at his/her time-of-death:
here the only probabilities are \{0,1\}
-- here  IBS\((\cdot, \cdot) = 0\), even though no patient's 
$\estCP{Heaviside}{\death{i}}{\inst{i}}$ is ever in [0.5, 0.6).
However, as we have previously noted, 
the inherent stochasticity of the world means
{ that 
meaningful distributions should include non-zero probabilities in other places as well,
rather than placing all weight on a single time point.}

Since the Integrated Brier score fails to account for this, 
there is no guarantee that probabilities are meaningful across individual survival curves.
This motivated us to introduce D-Calibration,
to determine whether a proposed \ISD-model produces meaningful distributions,
with probabilities that reflect the number of deaths that have occurred in the population.
\comment{
For this reason we have introduced  D-Calibration
which ensures that models build meaningful distributions
--  \ie whose probabilities reflect the number of deaths that have occurred in the population.
}
To see that these two metrics are 
measuring different aspects, % evaluating different objects 
note that 
% consider datasets GBM and GLI -- for each dataset, while 
the Integrated Brier scores for the (\AFT, \CoxKP, \CoxENKP, and \MTLR) models 
are all well within 1 standard error of one another for the GBM dataset,
but only \CoxENKP\ and \MTLR\ are D-Calibrated.
(This is also true for the GLI dataset.)

\subsubsection{Incorporating Censoring into the Brier score}
\label{sec:BS-censor}
In 1999, Graf~\etal~\cite{graf1999assessment}
proposed a 
way to compute the Brier Score for censored data, by 
 using \textit{inverse probability of censoring weights} (IPCW), 
which requires estimating the censoring survival function,
denoted as \(\hat{G}(t)\) over time points $t$. 
We can estimate \(\hat{G}(t)\) by the Kaplan-Meier curve of the \textit{censoring distribution}
-- \ie swapping those who are censored with those who are not, 
(\(\delta^{Cens}_i = 1 - \delta_i\)) and 
building the standard Kaplan-Meier model. 
% Theoretical justifications supporting the use of IPCW are not within the scope of this paper but are available from Gerds~\etal~\cite{gerds2008performance, gerds2006consistent}. 
Intuitively, this IPCW weighting counteracts the sparsity of later observations -- 
if a patient dies early, there is a good chance that $d_i < c_i$ 
meaning the event is observed, 
but if the patient survives for a long time,
it becomes more likely that $c_i < d_i$ meaning this patient will be censored.
% the event goes unobserved).
Gerds~\etal~\cite{gerds2008performance, gerds2006consistent}
formalizes and proves this intuition. % theoretical justifications supporting this use of IPCW.

The censored version of the Brier score for a given time, \(t^*\), is calculated as 
% \note[HH]{I had to make the equation text size smaller to get it to fit... any other ideas?}
{\footnotesize  
\begin{align}
\label{eqn:censoredBrier}
BS_{\tzero}\left(\, V, \ \hat{S}(\tzero| \cdot)\,\right)
\,\,=\,\,
\frac{1}{|V|} \sum_{i = 1}^{|V|} 
\left[\frac{\Indic{t_i \leq \tzero, \delta_i = 1}
\left(0 - \hat{S}(\tzero| \inst{i})\right)^2}{\hat{G}(t_i)}  + \frac{\Indic{t_i > \tzero}\left(1 - \hat{S}(\tzero| \inst{i})\right)^2}{\hat{G}(\tzero)}\right],
\end{align}
}

\noindent
where \(t_i = \min\{d_i, c_i\}\). 
The first part of  Equation~\ref{eqn:censoredBrier}
considers only uncensored patients whereas the second part counts all patients whose event time is greater than \(\tzero\). 
The patients who were censored \textit{prior} to \(\tzero\) 
are not explicitly included,
but contribute based on their influence in 
$\hat{G}(\cdot)$. 

As \(\hat{G}(t)\) is a decreasing step 
function of $t$, %  and hence
\(\frac{1}{\hat{G}(t)}\) is increasing,
which means % this mechanism designates 
that patients who survive longer than \(\tzero\) have larger weights than patients that died earlier, 
since the longer surviving patients were more likely to become censored.
\comment{
\note[BH]{I think it would be useful to point out that G is decreasing, so 1/G is increasing and observations of survival beyond \(\tzero\) are thus given higher weight (because patients are less likely to have reached that point without being censored)}\note[HH]{I added a blurb above... is this what you were trying to get at?}
}

%\subsection{Other Comments about D-Calibration} % A.5
\subsection{D-Calibration} % A.5
\label{app:D-Calib}

\def\tSF#1#2{S(\, #1\,|\, \inst{#2}\,)}
\def\SFinv#1#2{S^{-1}(\, #1\, |\, \inst{#2}\,)}
\def\eSF#1#2#3{\hat{S}(\, #1\,|\, \inst{#2},\,#3\,)}
\def\meas#1{m_{#1}}
\def\eprobAA#1#2#3{\hat{g}_{#1}(\,#2\,)} 
\def\eprob#1#2{\eprobAA{}{#1}{#2}}
\def\eprobP#1#2{\eprobAA{+}{#1}{#2}}
\def\eprobN#1#2{\eprobAA{-}{#1}{#2}}
\def\tprob#1{g_{-}(\,#1\,)}
\def\tprobC#1{g_{+}(\,#1\,)}
\def\pr#1{P(\,#1\,)}
\def\cpr#1#2{\pr{#1\,|\,#2}}

We begin this section by justifying why, 
in the case of all uncensored patients, 
(1)~the distribution of the survival function, \(\{S(t)\}_t\), should follow a uniform distribution,
then (2)~%
Following this discussion, we show 
how to incorporate censored patients into the D-Calibration estimate,
and finally,
(3)~that %Subsequently, we prove 
this combination of censored and uncensored patients will produce a uniform distribution for the goodness-of-fit test to test against.

For this analysis,
we assume each patient $\inst{i}$ has a true survival function, 
$\tSF{t}{i}$,
which is the probability that this patient will die after time $t$.
Assume each patient has a time of death, \(d_i\) and a censoring time, \(c_i\), and \(t_i = \min{\{d_i, c_i\}}\) is the observed event time. 
We also 
assume % employ the standard assumption
that censoring time is independent of death time, \(c_i \perp d_i\). 
Given a validation set \(|V|\),
we first examine the case of all uncensored patients
-- \ie % , that is, we have
\(t_i = d_i\) for \(i = 1,\ldots, |V|\).

\begin{lemma} % 1
The distribution of a patient's survival probability at the time of death
$\Csurv{\death{i}}{\inst{i}}$
is uniformly distributed on [0,1].
\comment{
Consider the distribution of probabilities at patients' event times, 
$\{\Csurv{t_i}{\inst{i}}\}_i$.
\note[RG]{Why not $d_i$? ... then don't have to mention uncensored ..}
If all \(\inst{i}\) were uncensored 
($t_i = d_i\ \forall i$), % \(\{t_i = d_i\}_i\),
then these event-probability distributions will be uniformly distributed on [0,1].}
\label{lem:Uniform}
\end{lemma}
\begin{proof}
The probability integral transform~\cite{angus1994probability} states that,
for any random continuous variable, \(X\), 
with cumulative distribution function given by \(F_x(\cdot)\), 
the random variable \(Y = F_x(X)\) will follow a uniform distribution on [0,1], denoted as \(U(0,1)\).
Thus, given randomly sampled event times, \(t\), we have \(F(t) \sim U(0,1)\). 
As the survival function is simply \(S(t) = 1 - F(t)\),
its distribution is \(1 - U(0,1)\),
which also follows \(U(0,1)\) 
and hence \(S(t) \sim U(0,1)\). 
\comment{
\note[BH]{It's unclear to me why the concept of identical patients is introduced here; and if they're necessary here, it's unclear why they aren't needed for Case 2}\note[HH]{Case 2 is suggesting that they don't need to be identical -- each patient will have its own survival curve, each of which will be uniform. Then we are taking a mixture of uniforms.}

{\em Case 2}: 
In general, 
each patient $\inst{i}$ will have his/her own survival function, $\tSF{\cdot}{i}$.
Case~1 showed that for each patient $i$, individual survival functions for event times will each follow a uniform distribution,
\(\tSF{t_i}{i} \sim U(0,1)\).
Now assume each patient $i$ occurs with probability $\pr{i}$, such that \(\sum_{i=1}^{|V|} \pr{i} = 1\). This means the distribution over these event times will be a weighted, finite mixture of uniforms.
One can then make the following observation that for any \(\alpha \in [0,1]\),
\note[RG]{Does this mean for ALL $i$'s? Or for A RANDOMLY-DRAWN $i$? Or ...}
\begin{eqnarray*}
\pr{\{\tSF{t_i}{i}\}_i > \alpha}
    & =&
    \sum_{i =1}^{|V|} \pr{
\tSF{t_i}{i}
    % S(t_i |x_i) 
    > \alpha\ } \ \pr{i}
     \nonumber\\
 &=& (1 -\ \alpha) \times \sum_{i = 1}^{|V|} \pr{i}\\
  &=&(1 -\ \alpha)
\end{eqnarray*}
Note that 
\( \pr{\dots \alpha}\ =\ (1 - \alpha\)) is the survival function for the \(U(0,1)\) distribution and 
since survival functions uniquely define distributions,
we have that \(\{\tSF{t_i}{i})\}_i \sim U(0,1)\)\ .
}
% \qed
\end{proof}
% \\[1ex]

This Lemma % lem:Uniform
shows that,
given the true survival model, producing 
\( \tSF{\cdot}{i} % S(\cdot|\cdot)
\) curves, 
the distribution of $\tSF{\death{i}}{i}$
should be uniform over event times. 
Thus if a learned model accurately learns the true survival function,
\(\estCP{\Theta}{\cdot}{\cdot} \approx S(\cdot|\cdot)\), 
we will expect the distribution across event times to be uniform. 
This is then tested using the goodness-of-fit test assuming each bin contains an equal proportions of patients. 
% An alternative to the goodness-of-fit test which discretizes the space into bins are the Kolomogrov-Smirnov test or Anderson-Darling test which test the entire empirical cumulative distribution function against a hypothesized distribution (the uniform distribution in this case)

Of course,
conditions become more complicated when considering censored patients. 
Suppose we have a censored patient
-- \ie \(t_i  = c_i\) --
such that $\Csurv{\censor{i}}{\inst{i}} = 0.25$. 
Since the censoring time is a lower bound on the true death time,
we know that \(\Csurv{\death{i}}{\inst{i}} \leq 0.25\), 
since \(c_i < d_i\) and survival functions are monotonically decreasing as event time increases. 
If we are using % In the case of using 
deciles,
we would like to know the probability that the time of death occurred in the [0.2,0.3) bin
-- \ie \(P\left(\ S(d_i | \inst{i}) \in [0.2,0.3) \ |\ S(d_i| \inst{i}) \leq 0.25\right)\). 
Using the rules of conditional probability,
this is computationally straightforward%
\footnote{To simplify notation, we drop the conditioning on \(\inst{i}\) of \(S(\cdot|\cdot)\).}:

\def\pPR#1{P(\, #1\,)}
\def\pCPR#1#2{\pPR{#1\, |\, #2}}

\begin{align*}
\pCPR{S(d_i) \in [0.2,0.3) }{ S(d_i) < 0.25 } \quad&=\quad 
 \frac{\pPR{S(d_i) \in [0.2,0.3), \, S(d_i) < 0.25}}
  {\pPR{S(d_i) < 0.25 }}\\[8pt]
&=\quad\frac{P(S(d_i) \in [0.2,0.25))}{P(S(d_i) < 0.25)}\\[8pt]
&=\quad\frac{0.05}{0.25} \tag{as \(S(\cdot) \sim U(0,1)\)}\\[6pt]
&=\quad 0.2
\end{align*}
Similarly, we can use the same logic as above to compute
these probabilities for the other two bins, $[0.1, 0.2)$ and $[0.0, 0.1)$:
% Similarly we would want to know these probabilities for the other two bins, \ie % \newline
%\(\pCPR{S(d_i) \in [0.1,0.2) }{ S(d_i) \leq 0.25}\) and \(P\left(S(d_i) \in [0.0,0.1) \,\,|\,\, S(d_i) \leq 0.25\right)\). Using the same logic as above, we have

\begin{align*}
\pCPR{S(d_i) \in [0.1,0.2) }{ S(d_i) < 0.25} % P(S(d_i) \in [0.1,0.2) | S(d_i) < 0.25) 
\quad&= \quad \frac{P(S(d_i) \in [0.1,0.2), \, S(d_i) < 0.25)}{P(S(d_i) < 0.25)}\\[8pt]
&=\quad\frac{P(S(d_i) \in [0.1,0.2))}{P(S(d_i) < 0.25)}\\[8pt]
&=\quad\frac{0.1}{0.25} \tag{as \(S(\cdot) \sim U(0,1)\)}\\[6pt]
&=\quad 0.4
\end{align*}
and similarly for the \([0.0, 0.1)\) bin. Note that these probabilities sum to one, \((0.2 + 0.4 + 0.4) = 1\), as desired. 

This example % The example above
motivates the following procedure to incorporate censored patients into
the D-Calibration process: 
Given \(B\) bins that equally divide [0,1] into 
intervals of width \(BW = 1/B\), 
suppose a patient is censored at time \(c\) with
associated % respective 
survival probability \(S(c)\). 
Let \(b_1\) be the infimum probability of the bin 
that contains % in which
\(S(c)\) -- % is contained, 
\eg 0.2 for the example above where \(S(c_i) = 0.25 \in [0.2, 0.3)\).
Then we assign the following weights to bins:
\begin{enumerate}[label=(\Alph*)]
\item Bin \([b_1, b_2)\) (which contains \(S(c)\)): \(\frac{S(c) - b_1}{S(c)} = 1 - \frac{b_1}{S(c)}\)
\item All following bins
(\ie the bins whose survival probabilities are all less than \(b_1\)):
\(\frac{BW}{S(c)} = \frac{1}{B \cdot S(c)}\),
\end{enumerate}
% where all following bins are the bins containing survival probabilities less than \(b_1\). 

Note this formulation follow directly from the example above.
This weight assignment effectively ``blurs'' censored patients across the bins following the bin 
where % in which 
the patient's learned survival curve,
\( \estCP{\Theta}{c_i}{i} \) % \hat{S}_{\Theta}(c_i|\inst{i})\),
placed the censored patient.

To further illustrate this concept of blurring a patient across bins, 
consider a patient who is censored at \(t=0\) with \(S(c_i) = 1\). 
This patient is then blurred across all \((B=10)\) bins, adding a weight of 0.1 to all 10 bins. 
Alternatively, if a patient is censored very late, with \(S(c_i) \leq 0.1\) 
then the patient is not blurred at all -- 
only a weight of 1 is added to the last bin.
\comment{Note the weakness of D-Calibration here -- if all patients are censored at \(t=0\) then the model will be D-Calibrated (all bins contain exactly one tenth of all the patients) even though the model need not have learned anything. 
\note[BH]{How could any model learn anything in this degenerate case?  A more useful example might be a case where you have a mix of labelled data (with meaningful event/censored times) and unlabelled data (all censored at 0) - if the unlabelled dataset is very large (eg. 90\% of the overall data), then your bins will look quite balanced no matter how the labelled data is split.}
\note[HH]{This has tended to be Russ' favorite counter example to D-calibration so I included it. What is the definition of ''unlabelled'' here? We should discuss in the meeting how to change to a "better" worst case for D-Calibration.}}

% \note[RG]{I considered moving the "weakness" paragraph to after we define the D-calibration method... but decided it worked better here. }

This identifies a weakness of D-Calibration:
if a validation set contains \(N_0\) patients censored at time 0,
then all bins are given an equal weight of \(N_0/B\); 
if \(N_0\) is large relative to the total number of patients, 
then the bins may appear uniform,
no matter how the other % remaining
patients are distributed, 
which means any model based on such heavily ``time 0 censored'' data
would be considered to be D-Calibrated.

To perform the goodness-of-fit test, 
we must first calculate the observed proportion of patients within each bin. %  must be calculated. 
Let \(N_k\) represent the observed proportion of patients within the interval \([p_k, p_{k+1})\)
--\eg \([p_k, p_{k+1}) = [0.2,0.3)\) in the example above. 
We can formally calculate: %  \(N_k\):
\comment{\begin{align*}
N_k\quad =\quad \frac{1}{|V|}\sum_{i = 1}^{|V|}\, \bigg[
   \quad & \Indic{S(d_i) \in [p_k, p_{k+1}) \,\wedge\, d_i \leq c_i}  \tag{Uncensored} \\
%+  \,\,&  \Indic{S(c_i) \in [p_k, p_{k+1}) \,\wedge\, c_i < d_i}  - \frac{p_k}{S(c_i)} \,\cdot\, \Indic{S(c_i) \in [p_k, p_{k+1}) \,\wedge\, c_i < d_i} \tag{1} \\
\qquad +  \,\,&   \frac{S(c_i)-p_k}{S(c_i)} \,\cdot\, \Indic{S(c_i) \in [p_k, p_{k+1}) \ \wedge\  c_i < d_i} \tag{1} \\
\qquad +\,\, &  \frac{(p_{k+1} - p_k)}{S(c_i)} \,\cdot\, \Indic{S(c_i) \geq p_{k+1} \ \wedge\  c_i < d_i} \bigg].\tag{2} % \\
\end{align*}
}
{\begin{eqnarray}
N_k\quad =\quad \frac{1}{|V|}\sum_{i = 1}^{|V|}\, \bigg[
  & & \Indic{S(d_i) \in [p_k, p_{k+1}) \,\wedge\, d_i \leq c_i}  \label{eqn:Uncensored} % \tag{Uncensored}
   \\
&+  &   \frac{S(c_i)-p_k}{S(c_i)} \,\cdot\, \Indic{S(c_i) \in [p_k, p_{k+1}) \ \wedge\  c_i < d_i} \label{eqn:1} % \tag{1} 
\\
 & + &  \frac{(p_{k+1} - p_k)}{S(c_i)} \,\cdot\, \Indic{S(c_i) \geq p_{k+1} \ \wedge\  c_i < d_i} \bigg]. \label{eqn:2} % \tag{2} % \\
\end{eqnarray}
}

\noindent Above, 
(\ref{eqn:Uncensored}) % (Uncensored) 
refers to the weight that the patients with observed events 
contribute to % have on
the \(k^\textrm{th}\) bin 
-- \ie each uncensored patient whose survival probability at time of death lands in \([p_k, p_{k+1})\) contribute a value of $1$. 
Here, we consider  \(d_i = c_i\)  
% Also note here that if \(d_i = c_i\) we consider this 
to be an uncensored event. 
Next, (\ref{eqn:1}) % (1)
gives the weight from the censored patients whose survival probability at time of censoring is within the \(k^\textrm{th}\) bin (item (A) above). Lastly, (\ref{eqn:2}) % (2)
gives the weights from censored patients whose survival probability was contained in a previous bin (item (B) above). 

Theorem~\ref{thm:DCalGOF} below proves that
the expected value of \(N_k\) is equal for all bins
-- \ie \(\mathbb{E}[N_k] = p_{k+1} - p_k\)
-- which allows us to apply the goodness-of-fit test with uniform proportions.
% \note[RG]{why not first give the proof, then discuss this strictness claim??}\note[HH]{I have reworded this section and added a comment after the proof.}
\comment{
To apply the goodness-of-fit test with uniform proportions,
we need % want
to show that the expected value of \(N_k\) is equal for all bins
-- \ie \(\mathbb{E}[N_k] = p_{k+1} - p_k\).
The proof can be found below,
but
}

\comment{
The intuition here is that if a 
survival curve contains a large flat area then there is non-zero probability mass for \( S(c_i) = S(d_i) \) when \( c_i \neq d_i \). 
If this is the case then terms in the proof below will fail to cancel with one another leaving us with non-equivalent proportions within each bin (specifically higher proportions within bins which contain these flat lines).
}

We assume that all survival curves are \textit{strictly} monotonically decreasing meaning we have the equality,
$\death{i} \leq \censor{i}\ \iff\ S(\death{i}) \geq S(\censor{i})$).
% $d \leq c\ \iff\ S(d) \geq S(c)$.
% {1. I added the $i$ subscript -- otherwise, suggests every death is before every censoring.\\
% 2. Should this be $<$ not $\leq$? }
This equivalence lets us replace  $\death{i} \leq \censor{i}$
with $S(\death{i}) \geq S(\censor{i})$,
within the indicator functions in \(N_k\). 
\comment{ Using this equality we make this replacement
($S(\death{i}) \geq S(\censor{i})$ instead of $\death{i} \leq \censor{i} $) 
within the indicator functions in \(N_k\). }%
To simplify notation, we % Additionally, we make some notational simplifications, 
define \( I_k := [p_k, p_{k+1})\), 
\(S_c := \tSF{c}{} % S(c\, |\,\inst{})
\), and
\(S_d:=  \tSF{d}{}
% S(d\, |\,\inst{})
\). 
% \change[RG]{Within the proof below we instead show}
{The proof below shows that }
the expected value of the 
summand
within Equations~(\ref{eqn:Uncensored}) -- (\ref{eqn:2})
above
is equal to \(p_{k+1} - p_k\) -- 
\ie we ignore \(\frac{1}{|V|}\sum_{i = 1}^{|V|}[\cdot]\) and take the expected value of the 
term inside the summation.

\begin{thm}
\label{thm:DCalGOF}
Given the formula for \(N_k\) 
(Equations (\ref{eqn:Uncensored}) - (\ref{eqn:2})),
if the  true survival function \(S(\cdot | \cdot)\)
is strictly monotonically decreasing 
then proportions are equal across all bins
-- \ie \(\mathbb{E}[N_k] = p_{k+1} - p_k\).

\end{thm}
\begin{proof}

\begin{align*}
\mathbb{E}[N_k] =\,& \mathbb{E}\bigg[ \, \Indic{S_d \in I_k \,\wedge\, S_d \geq S_c} \\
% &+  \Indic{S_c \in [p_k,p_{k+1}) \, \,\wedge\,  \, S_c > S_d}  - \frac{p_k}{S_c}\cdot \Indic{S_c \in [p_k,p_{k+1}) \, \,\wedge\,  \, S_c > S_d} \\
&\quad +  \frac{S_c-p_k}{S_c}\cdot \Indic{S_c \in I_k \, \,\wedge\,  \, S_c > S_d} \\
&\quad + \frac{(p_{k+1} - p_k)}{S_c} \cdot \Indic{S_c > S_d \,\wedge\, S_c \in [p_{k+1},1]} \bigg] \\ \\
=\,& \quad \mathbb{E}\big[\,\Indic{S_d \in I_k \,\,\wedge\,\, S_d \geq S_c}\,\big] \\
%&+  \mathbb{E}\big[\,\Indic{S_c \in [p_k,p_{k+1}) \, \,\wedge\,  \, S_c > S_d}\,\big]  - \mathbb{E}\left[\frac{p_k}{S_c}\cdot \Indic{S_c \in [p_k,p_{k+1}) \, \,\wedge\,  \, S_c > S_d}\right] \\
&+  \mathbb{E}\left[\frac{S_c-p_k}{S_c}\cdot \Indic{S_c \in I_k \, \,\wedge\,  \, S_c > S_d}\right] \\
&+ \mathbb{E}\left[\frac{(p_{k+1} - p_k)}{S_c} \cdot \Indic{S_c > S_d \,\wedge\, S_c \geq p_{k+1}} \right] \\ \\
=\,&\quad \Pr[\,S_d \in I_k \,\,\wedge\,\, S_d \geq S_c\,] \\ 
&+ \Pr[S_c \in I_k \, \,\wedge\,  \, S_c > S_d] \, -p_k\,\mathbb{E}\left[ \frac{1}{S_c} \cdot \Indic{S_c > S_d \,\wedge\, S_c \in I_k}\right] \\ 
&+ (p_{k+1} - p_k)\mathbb{E}\left[\frac{1}{S_c} \cdot \Indic{S_c > S_d \,\wedge\, S_c \geq p_{k+1}} \right]\\\\
%
%=\,& \Pr[S_d \in [p_k,p_{k+1}) \,\,\wedge\,\, S_d \geq S_c] +  \Pr[S_c \in [p_k,p_{k+1}) \, \,\wedge\,  \, S_c > S_d] \, \\&-p_k\,\mathbb{E}\left[ \frac{1}{S_c} \cdot \Indic{S_c > S_d \,\wedge\, S_c \in I_k}\right] \\ &+p_{k+1}\mathbb{E}\left[\frac{1}{S_c} \cdot \Indic{S_c > S_d \,\wedge\, S_c > p_{k+1}} \right] - p_k\mathbb{E}\left[\frac{1}{S_c} \cdot \Indic{S_c > S_d \,\wedge\, S_c > p_{k+1}} \right]\\\\\\
%
=\,& \Pr[S_d \in I_k \,\,\wedge\,\, S_d \geq S_c]
\quad+\quad  \Pr[S_c \in I_k \, \,\wedge\,  \, S_c > S_d] \, \tag{I}\\
&-p_k\,\mathbb{E}\left[ \frac{1}{S_c} \cdot \Indic{S_c > S_d \,\wedge\, S_c  \geq p_k}\right] \tag{II} \\ 
&+p_{k+1}\mathbb{E}\left[\frac{1}{S_c} \cdot \Indic{S_c > S_d \,\wedge\, S_c \geq p_{k+1}} \right] \tag{III}\\
\end{align*}
Focusing on the second probability in
line~(I), % Part I, 
note \(S_c \in I_k = [p_k, p_{k+1})\) and \(S_c > S_d\) imply that \(S_d \in [0,p_{k+1})\) which can be expanded to the cases for \(S_d < p_k\) and \(S_d \in I_k\). 
Using this, we reformulate the probability by noting the equivalence of the event space,
%\[
%\Pr[S_d \in [p_k,p_{k+1}) \,\,\wedge\,\, S_d \geq S_c] = \Pr[S_d \in [p_k,p_{k+1}) \,\,\wedge\,\, S_c < p_k] + \Pr[(S_c \,\wedge\, S_d) \in [p_k,p_{k+1}) \,\,\wedge\,\, S_d \geq S_c]
%\]
%and analogously,
\[
\Pr[S_c \in I_k \,\,\wedge\,\, S_c > S_d]\,\, = \,\,\Pr[S_c \in I_k \,\,\wedge\,\, S_d < p_k]\,\, +\,\, \Pr[(S_c \,\wedge\, S_d) \in I_k \,\,\wedge\,\, S_c > S_d].
\]
Combining the second piece above with the first probability
in line~(I), % Part I 
we again simplify by noting these probabilities bound \(S_c < p_{k+1}\),
%\[\Pr[(S_c \,\wedge\, S_d) \in [p_k,p_{k+1}) \,\,\wedge\,\, S_d \geq S_c] + \Pr[(S_c \,\wedge\, S_d) \in [p_k,p_{k+1}) \,\,\wedge\,\, S_c > S_d] = \Pr[(S_c \,\wedge\, S_d) \in [p_k,p_{k+1})].\]
\[\Pr[S_d  \in I_k \,\,\wedge\,\, S_d \geq S_c] \,+\, \Pr[(S_c \,\wedge\, S_d) \in I_k \,\,\wedge\,\, S_c > S_d]\,\, = \,\,\Pr[S_d  \in I_k \,\,\wedge\,\, S_c < p_{k+1} ].\]
Using this simplification we can rewrite the entirety of
line~(1), % Part I,
$$\begin{array}{cccc}
&  \Pr[\ S_d \in I_k \,\,\wedge\,\, S_d \geq S_c\ ] 
      &+&\Pr[\ S_c \in I_k \,\,\wedge\,\, S_c > S_d\ ] \\
= &\Pr[\ S_d \in I_k \,\,\wedge\,\, S_c < p_{k+1}\ ] 
 &+&\Pr[\ S_c \in I_k \,\,\wedge\,\, S_d < p_k\ ]
\end{array}
$$
\comment{\begin{align*}
%\Pr[S_d \in [p_k,p_{k+1}) \,\,\wedge\,\, S_d \geq S_c]+\Pr[S_c \in [p_k,p_{k+1}) \,\,\wedge\,\, S_c > S_d] = \quad &\Pr[(S_c \,\wedge\, S_d) \in [p_k,p_{k+1})] \\
% +\,\,&\Pr[S_d \in [p_k,p_{k+1}) \,\,\wedge\,\, S_c < p_k] \\
% +\,\,&\Pr[S_c \in [p_k,p_{k+1}) \,\,\wedge\,\, S_d < p_k]
\Pr[S_d \in I_k \,\,\wedge\,\, S_d \geq S_c]\,+\,\Pr[S_c \in I_k \,\,\wedge\,\, S_c > S_d] \,\,=\,\,  &\Pr[S_d \in I_k \,\,\wedge\,\, S_c < p_{k+1}] \\
 \,\,+\,\,&\Pr[S_c \in I_k \,\,\wedge\,\, S_d < p_k]
 \end{align*}}
Recalling the independence assumption, \(c \perp d\), we have the following equalities:
\begin{align*}
%\Pr[(S_c \,\wedge\, S_d) \in [p_k,p_{k+1})] \quad &= \quad \Pr[S_d \in [p_k,p_{k+1})] \cdot \Pr[S_c \in [p_k,p_{k+1})], \\
%\Pr[S_d \in [p_k,p_{k+1}) \,\,\wedge\,\, S_c < p_k] \quad &= \quad \Pr[S_d \in [p_k,p_{k+1})] \cdot \Pr[S_c < p_k], \\  
%\Pr[S_c \in [p_k,p_{k+1}) \,\,\wedge\,\, S_d < p_k] \quad &= \quad \Pr[S_c \in [p_k,p_{k+1})] \cdot \Pr[S_d < p_k].
&\Pr[S_d \in I_k \,\,\wedge\,\, S_c < p_{k+1}] \, &=& \,\, \Pr[S_d \in I_k] \cdot \Pr[S_c < p_{k+1}] \, &=& \,\, (p_{k+1}-p_k) \, \Pr[S_c < p_{k+1}]  , \\  
&\Pr[S_c \in I_k \,\,\wedge\,\, S_d < p_k] \, &=& \,\, \Pr[S_c \in I_k] \cdot \Pr[S_d < p_k] \, &=& \,\,\, p_k \, \Pr[S_c \in I_k] ,
\end{align*}
where the final equalities are due to the uniformity of the survival function on \(d\), \(S(d) \sim U(0,1)\).
%\begin{align*}
%%\Pr[S_d \in [p_k,p_{k+1})] \cdot \Pr[S_c \in [p_k,p_{k+1})] \quad &= \quad (p_{k+1} - p_k) \, \Pr[(S_c \in [p_k,p_{k+1})],\\
%\Pr[S_d \in [p_k,p_{k+1})] \cdot \Pr[S_c < p_k] \quad &= \quad (p_{k+1} - p_k)\, \Pr[S_c < p_k],\\  
%\Pr[S_c \in [p_k,p_{k+1})] \cdot \Pr[S_d < p_k] \quad &= \quad p_k \, \Pr[S_c \in [p_k,p_{k+1})].
%\Pr[S_d \in I_k] \cdot \Pr[S_c < p_{k+1}] \quad &= \quad (p_{k+1} - p_k)\, \Pr[S_c < p_{k+1}],\\  
%\Pr[S_c \in I_k] \cdot \Pr[S_d < p_k] \quad &= \quad p_k \, \Pr[S_c \in I_k].
%\end{align*}
This then leaves the final simplification of 
line~(I) % Part I 
as,
\begin{align*}
%\Pr[S_d \in [p_k,p_{k+1}) \,\,\wedge\,\, S_d \geq S_c]+\Pr[S_c \in [p_k,p_{k+1}) \,\,\wedge\,\, S_c > S_d] \,\, &=  (p_{k+1} - p_k) \Pr[(S_c \in [p_k,p_{k+1})] \\
%&+\,\, (p_{k+1} - p_k)\, \Pr[S_c < p_k]\\
%&+\,\, p_k \, \Pr[S_c \in [p_k,p_{k+1})].
\Pr[S_d \in I_k \,\,\wedge\,\, S_d \geq S_c]+\Pr[S_c \in I_k \,\,\wedge\,\, S_c > S_d] \,\, =\,&  (p_{k+1} - p_k)\, \Pr[S_c < p_{k+1}]\\
&+\,\, p_k \, \Pr[S_c \in I_k].
\end{align*}

Now we address line~(II) % Part II 
and analagously line~(III): % Part III:
\begin{align*}
-p_k\,\mathbb{E}\left[ \frac{1}{S_c} \cdot \Indic{S_c > S_d \,\wedge\, S_c  > p_k}\right] 
%& = -p_k\,\left(\,\int_0^{p_k}\int_{p_k}^1\frac{1}{S_c} \, f(S_c) \,dS_c\,dS_d + 
%\int_{p_k}^1\int_{S_d}^1\frac{1}{S_c} \, f(S_c) \,dS_c\,dS_d \right) \tag{Def. of \(\mathbb{E[\cdot]}\) }\\\\
& = -p_k\,\left(\,\int_{p_k}^1\int_0^{S_c}\frac{1}{S_c} \, f(S_c) \,dS_d\,dS_c  \right) \tag{Def. of \(\mathbb{E[\cdot]}\) }\\\\%
%& = -p_k\,\left(p_k\,\mathbb{E}\left[\frac{1}{S_c}  \cdot \Indic{S_c > p_{k}}\right]  + 
%\int_{p_k}^1\int_{S_d}^1\frac{1}{S_c} \, f(S_c) \,dS_c\,dS_d \right) \\\\
& = -p_k\,\left( 
\int_{p_k}^1\frac{S_c}{S_c} \, f(S_c) \,dS_c \right) \\\\
%
%&=-p_k\,\left(p_k\,\mathbb{E}\left[\frac{1}{S_c}  \cdot \Indic{S_c > p_{k}}\right]  + 
%\int_{p_k}^1\int_{p_k}^{S_c}\frac{1}{S_c} \, f(S_c) \,dS_d\,dS_c \right) \tag{1}\\\\
%
%& =-p_k\,\left(p_k\,\mathbb{E}\left[\frac{1}{S_c}  \cdot \Indic{S_c > p_{k}}\right]  + 
%\int_{p_k}^1\frac{S_c - p_k}{S_c} \, f(S_c) \,dS_d\,dS_c \right)\\\\
%
%&=-p_k\,\left(p_k\,\mathbb{E}\left[\frac{1}{S_c}  \cdot \Indic{S_c > p_{k}}\right]  + \Pr[S_c > p_k] - p_k\,\mathbb{E}\left[\frac{1}{S_c}  \cdot \Indic{S_c > p_{k}}\right]  \right) \\\\
%
&=-p_k\,\Pr[S_c > p_k] \\
\end{align*}

Here \(f\) is the probability distribution function (PDF) for the distribution generated by the survival function applied to a \textit{censored} observation. 
As the censoring distribution is unknown \(f(S_c)\) is also unknown whereas \(f(S_d)\) would be the PDF of the uniform distribution. 

Following the steps above for 
line~(III) % Part III 
analogously gives us
\[
p_{k+1}\,\mathbb{E}\left[ \frac{1}{S_c} \cdot \Indic{S_c > S_d \,\wedge\, S_c  > p_{k+1}}\right] 
\quad=\quad p_{k+1}\,\Pr[S_c > p_{k+1}]
\]

Combining the simplifications of 
lines (I), (II) and (III), % Parts I, II and III
we have the following,
\begin{align*}
%\mathbb{E}[N_k] &= (p_{k+1} - p_k)  \,\Pr[S_c \in I_k] \,+\, (p_{k+1} - p_k)\, \Pr[S_c < p_k] \,+\, p_k \, \Pr[S_c \in I_k] \tag{I}\\
%&-p_k\,\Pr[S_c > p_k]\tag{II}\\
%&+p_{k+1}\,\Pr[S_c > p_{k+1}]\tag{III} \\\\
%&= p_{k+1}\,\Pr[S_c \in I_k] \,+\, p_{k+1} \,\Pr[S_c < p_k] \,+\,p_{k+1}\,\Pr[S_c > p_{k+1}] \,+\, p_k \,\Pr[S_c \in I_k]\\
%& - p_k\, \Pr[S_c \in I_k] -p_k\,\Pr[S_c > p_k] -p_k\, \Pr[S_c < p_k] \\\\
%&=p_{k+1}\,\left(\Pr[S_c \in I_k] \,+\, \Pr[S_c < p_k] \,+\,\Pr[S_c > p_{k+1}]\right) \\
%& - p_k\left( \Pr[S_c < I_k] \,+\,\Pr[S_c > p_k] \right) \\\\
%&= p_{k+1} - p_k
\mathbb{E}[N_k]\ =\, & (p_{k+1} - p_k)\, \Pr[S_c < p_{k+1}] \,+\, p_k \, \Pr[S_c \in I_k] \tag{I}\\
&-\ p_k\,\Pr[S_c > p_k]\tag{II}\\
&+\ p_{k+1}\,\Pr[S_c > p_{k+1}]\tag{III} \\\\
=\, &p_{k+1}\,\left(\Pr[S_c < p_{k+1}] \,+\ \Pr[S_c > p_{k+1}]\right) \\
& -\ p_k\left( \Pr[S_c < p_{k+1}] - \Pr[S_c \in [p_k, p_{k+1}) + \Pr[S_c>p_k] \right) \\\\
=\, & p_{k+1} - p_k
\end{align*}
\end{proof}

This proof requires the assumption that survival curves are \textit{strictly} monotonically decreasing on [0,1]. 
This means survival curves will not contain any
 large flat areas,
 which means there will not be 
 non-zero probability mass for \( S(c_i) = S(d_i) \) when \( c_i \neq d_i \),
 which means certain terms in the proof below 
 would fail to cancel with one another,
 leaving us with non-equivalent proportions within each bin (specifically higher proportions within bins that contain these flat lines).

A natural corollary of Theorem~\ref{thm:DCalGOF}
is that all consistent estimators of the true survival distribution will be D-Calibrated 
(if the true survival distribution is strictly monotonic).
Further, if survival time is independent and identically distributed (i.i.d.) across patients
then there will only be 
{one true survival curve} for all patients,
% {why?}
% \note[HH]{Because it's identically distributed...? and we aren't considering the conditional distribution}
and thus,
as Kaplan-Meier is uniformly consistent~\cite{breslow1974large,csorgHo1983rate}:
\begin{lemma}
The Kaplan-Meier distribution
is % will be 
asymptotically D-Calibrated.
\label{lem:KM-D-Calib}
\end{lemma}
This is consistent with
the results given in 
Section~\ref{sec:Empirical Results},
which showed that \KM\ 
always passed the D-Calibration test with a \(p\)-value 1.000, in all 8 datasets.
Under all uncensored data,
we would expect the typical 5\% Type I error rate for claiming \(p<\)0.05 as significant, 
however in the presence of censored data
a correct estimate of the survival distribution the proportion within bins become
smoothed, boosting the \(p\)-value.

\begin{figure} \centering
\includegraphics[width=\textwidth,height=2in]
{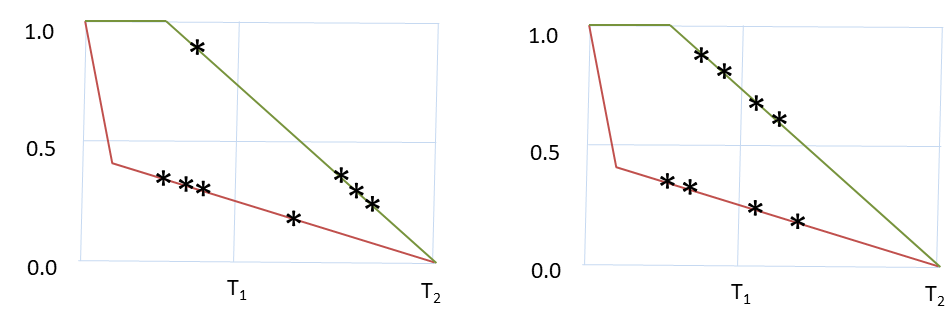}
\caption{  \label{fig:1Cal-Not-DCal}
Simplified models to illustrate:
[left] a model can have perfect 1-Calibration for a time, but not be D-Calibrated,
and 
[right] a model can have perfect D-Calibration, 
but not be 1-Calibrated for a time. (See text for description.)
}
\end{figure}

\begin{prop}
It is possible for a \ISD\ model to be
perfectly D-calibrated but not 1-calibrated at a time $\tzero$;
and for (another) \ISD\ model to be perfectly 1-calibrated at time $\tzero$ but not D-calibrated.
\label{Lma:1Cal.not.DCal}
\end{prop}

\begin{proof}
\comment{While 1-Calibration and D-Calibration do appear similar,
they are very different.
In particular, 
it is possible for a model to be perfectly D-calibrated, 
but have a very low 1-Calibration score (at a time),
and vice versa.
To illustrate that
}
{\bf ``1-Calibration $\not\Rightarrow$ D-Calibration'':}\quad
Consider the model shown in
Figure~\ref{fig:1Cal-Not-DCal}[left].
Here, the green curve corresponds to 4 apparently-identical patients $\{ \inst{g,1}, \dots, \inst{g,4} \}$, 
and the red curve,
to apparently-identical
$\{ \inst{r,1}, \dots, \inst{r,4} \}$.
The ``$\ast$''s mark the time when each patient died, denoted as $\death{\inst{}}$ 
for $\inst{}$. 
% \add[HH]{For simplicity,} 
We intentionally use simple examples,
with no censored patients,
% and also both of the
with curves that go to 0.
% S( t, x ) = P( D > t | x )  = probability (assigned by the model) that patient x will live at least t.
% Also,  d(x) = time that patient x died.
Note this model assigns $\estCP{}{T_1}{\inst{g,i}}\ =\ 0.75$
for each of the 4 green patients,
and $\estCP{}{T_1}{\inst{r,j}}\ =\ 0.25$ for 
each of the 4 red patients

To show that this model is 1-Calibrated, with respect to $T_1$:
Recall we first sort the  $\estCP{}{T_1}{\inst{}}$ values, then partition them into $k$ sets.
Here, we consider $k=2$, rather than the deciles earlier.
The first set contains the 4 patients with $\estCP{}{T_1}{\inst{}}  = 0.75$
  (\ie the green patients);
  and the second,
  the 4 patients with $\estCP{}{T_1}{\inst{}}  = 0.25$.
 Now note that 3 of the 4 
 ``$\estCP{}{T_1}{\inst{}}  = 0.75$ patients''
 are alive at $T_1$;
 and 1 of the 4 
 ``$\estCP{}{T_1}{\inst{}}  = 0.25$ patients''
 are alive at $T_1$
 -- which means this model is perfectly 
 1-Calibrated at $T_1$.

% (wrt $T_2$:  all 8 patients have S( T2, x) = 0 …)

However, this model is not D-Calibrated:
To be consistent with the earlier 1-Calibration analysis,
we partition the time intervals into 2 sets (not 10),
as shown in Figure \ref{fig:1Cal-Not-DCal}.
Here,
$\estCP{}{\death{\inst{}}}{\inst{}} \in [0.5, 1]$   holds for only 1 patient,
and
$\estCP{}{\death{\inst{}}}{\inst{}} \in [0, 0.5]$  holds for 7;
if the model was D-Calibrated, each of these sets should contain 4 patients.

\noindent {\bf ``D-Calibration $\not\Rightarrow$ 1-Calibration'':}
\quad
See 
Figure~\ref{fig:1Cal-Not-DCal}[right],
where again, each line represent 4 different patients;
notice the outcomes are different from those on the left.
To see that this model is D-Calibrated,
note there are 4 patients with
$\estCP{}{\death{\inst{}}}{\inst{}} \in [0.5, 1]$   (the green patients),
and 4 with
$\estCP{}{\death{\inst{}}}{\inst{}} \in [0, 0.5]$
(for the red patients).
However, the model is not 1-Calibrated, at $T_1$:
  Of the 4 patients with $\estCP{}{T_1}{\inst{}}  = 0.75$,
  2 are alive at $T_1$;
 and 
 of the 4 patients with $\estCP{}{T_1}{\inst{}}  = 0.25$,
%   (\ie the red patients), 
  2 are alive at $T_1$.
 To be 1-Calibrated, 
 there should be 3 living patients in the first set, 
 and 1 in the second;
 hence this model is not 1-Calibrated at $T_1$.
\end{proof}

%
 % \note[RG]{Note this proof would hold if each of the 4 Probability$\times$Time cells contained just 1 death, rather than 2. Would that be easier?
 % Note that 3 of those 4 would match the left, but we would need one in the upper right...
 % NO!  Want to have 3/4 and 1/4 .. so need 4 to show this is not just round-off }
 
% \input{DcalIsNot1Cal}  % this is just to show the proof ..
% The material above is from ``DCalIsNot1Cal.tex'' \hrule
\subsection{Other Subtle Points} % A.6
\label{sec:SubtlePts}
All of these tools for producing survival curves are able to deal with ``right censored'' events:
where the censored event time is a {\em lower bound}\ of the time of death.
(This corresponds to, perhaps, the termination of a study, or when a participant left the study early.)
There are other types of censoring, including ``left censoring'', which provides an upper-bound on the time of death (\eg when a survey finds that the patient is currently dead, but does not know when previously this happened),
and ``interval censoring'', when we can constrain the time of death to some interval.
While there are extensions of each of these tools that can accommodate these alternate types of censoring, here we considered the most common case of having right-censored  instances, and included only datasets that had only such instances.

%\note[RG]{added this...}
As a second subtle issue:
some of the methods involve taking the log of a predicted value, or of a true value;
see Appendix~\ref{app:LogL1-loss}.
This is clearly problematic if that value is 0
-- \eg if a patient died during a transplantation surgery.
To avoid these errors, we replace any such 0 with the $\eta$ for a database, 
which is defined as 1/2 of the minimum observed positive time of any event, in that dataset.
That is, we ignore all time=0 events,
and then consider the smallest remaining value.
If that value is, say, 1.0 day, then we set
$\eta = $0.5 days.
Note that all other times are left unchanged.

\section{Comments about Various \ISD's} %B
\label{app:ISDDetails}
\subsection{Comments about \CoxKP}
\label{app:CoxKP}
Notice Equation~\ref{eqn:CoxPH} embodies two strong assumptions:
(1)~that the individual features % covariates 
are independent of one another 
(\eg the outcome does not depend on a non-linear combination of the features), and 
(2)~that these covariates are independent of time 
-- which means that a blood test is as important just after an operation, as it is a year later, or a decade later.
These assumptions mean the survival curves for different patients will have the same basic shape, and will not cross;
see Figure~\ref{fig:AllCurves} (middle-left).
These simplifications allow the Cox model to suggest important information about individual features by examining the single coefficient $\beta_j$ associated with the $j^{th}$ feature,
\eg {\em does ``being male''
 increase the risk of dying from this specific cancer, or 
 does it protect against this outcome (or neither).}
 This ``neither'' case suggests that a given feature is not relevant to the prediction;
for this reason,
we used univariate Cox is a feature selection technique for our results.

By contrast, \MTLR\ and \RSFKM\ do not make these extreme assumptions, which means that a given feature can have different levels of importance at different times. Moreover, the curves for different patients can cross; 
see Figure~\ref{fig:AllCurves}. 
More relevant, however, we found that \MTLR\ is more often D-Calibrated, 
and hence more meaningful for individual patients, than this ``predictive Cox'' system;
see Table~\ref{tab:Dcal}. 
While this Cox analysis of survival
may not be directly relevant for individual patients,
there are still extreme benefits in being able to identify important features. 
By observing how different features impact survival, clinicians can be made aware of treatments or lifestyle changes that best help patients survival.

%\subsection{Simple Introduction to \MTLR} % B.2
\subsection{Overview of \MTLR} % B.2
\label{app:PSSP-intro}

\def\vt#1{\vec{\theta}_{#1}}
\newcommand{\vtset}{{\mathbf{\Theta}}}
\def\andB{} % otherwise  + b_i
\newcommand{\real}{\mathbb{R}}
\newcommand{\x}{\vec{x}}
\newcommand{\vb}{\vec{b}}

{
Consider%
\footnote{
This paragraph is paraphrased from \cite{PSSP-NIPS};
reprinted with permission of publisher/author.
}
modeling the probability of survival of patients at each of a vector of time points
$\tau = [t_1, t_2, \ldots, t_m]$
-- \eg $\tau$ could be the 60 monthly intervals from 1 month up to 60 months. 
We can set up a series of logistic regression models:
For each patient, represented as $\x$,
% for each of these:
\begin{equation} 
  \CPRsub{\vt{i}}{T \geq t_i}{\x}
  \quad=\quad 
   \left(1 + \exp(\vt{i}\cdot\x \andB)\right)^{-1},\qquad 1\leq i\leq m, \
   \label{eq:lr_series} 
\end{equation}
where $\vt{i}$ % and $b_i$ 
are the time-specific parameter vectors. % and thresholds. 
While the input features $\x$ stay the same for all these classification tasks,  
the binary labels $y_i = [T\geq t_i]$ can change depending on the threshold $t_i$.%
\comment{
This particular setup allows us to answer queries about the survival probability of individual patients at each of the time snapshots $\{t_i\}$, getting close to our goal of modeling a personal survival time distribution for individual patients.  
The use of time-specific parameter vector naturally allows us to capture the effect of time-varying covariates, similar to many dynamic regression models \cite{gamerman1991dynamic, hastie1993varying}. 
}
We encode the survival time $\death{}$ of a patient as a  sequence of binary values:
$y = y(\death{}) = [y_1, y_2, \ldots, y_m]$,
where $y_i = y_i(\death{}) \in \{0,1\}$ denotes the survival status of the patient at time $t_i$, 
so that $y_i = 0$ (no death event yet) for all $i$ with $t_i <\death{}$, and $y_i = 1$ (death) for all $i$ with $t_i \geq \death{}$.
% (see Figure~\ref{fig:km_threshold_encoding}[middle]). 
%The binary value $y_i$ is set to 0 (no death event yet) for $i$ such that $t_i <s$, and is set to 1 (death) for all $i$ with $t_i \geq s$ (see Figure~\ref{fig:encoding}). 
% We denote such an encoding of the survival time $s$ as $y(s)$, and let $y_i(s)$ be the value at its $i$th position. 
Here there are $m+1$ possible legal sequences of the form%
\footnote{
Notice there are no `0's after a `1'.
This is the `no zombie' rule: once someone dies, 
that person stays dead.
}
$[0,0,\ldots,1,1,\ldots,1]$, including the sequence of all `0's and the sequence of all `1's. 
Our \MTLR\ model computes 
the probability of observing the survival status sequence
$y = [y_1, y_2, \ldots, y_m]$ % can be represented 
as:
% by the following generalization of the logistic regression model:
\[
  \CPRsub{\vtset}{Y\!\! =\!\! [y_1, y_2, \ldots, y_m]}{\x}
  \quad=\quad \frac{\exp(\sum_{i=1}^m y_i \times \vt{i}\cdot\x \andB )}
     {\sum_{k=0}^m \exp(f_{\vtset}(\x, k))}, 
\]
where $\vtset = [\vt{1}, \ldots, \vt{m}]$, and
$f_{\vtset}(\x, k) = \sum_{i=k+1}^m (\vt{i}\cdot\x \andB)$ 
for $0\leq k\leq m$ is the score of the sequence with the event occurring in the interval $[t_k, t_{k+1})$ before taking the logistic transform, with the boundary case $f_{\vtset}(\x, m) = 0$ being the score for the sequence of all `0's. 
Given a dataset of $n$ patients $\{ \inst{r} \}$ with associated time of deaths $\{ \death{r} \}$, 
we find the optimal parameters (for the \MTLR\ model) $\vtset^*$ as
\begin{equation}
\vtset^*\ =\
\arg\max_{\vtset} 
\sum_{r=1}^n \left[\sum_{i=1}^m y_j(\death{r})(\vt{i}\!\cdot\!\x_r\andB)\! - \log \sum_{k=0}^m \exp f_{\vtset}(\x_r,k) \right]
  -  
\frac{C}{2}\!\sum_{j=1}^m\! \|\vt{j}\|^2\! 
% +\! \frac{C_2}{2}\sum_{j=1}^{m-1}\! \|\vt_{j+1}\!-\!\vt_j\|^2\!   
\label{eq:multitask_lr_opt}
\end{equation}
where the $C$ (for the regularizer) is found by an internal cross-validation process.
}
\comment{
In a nutshell, the \MTLR-Learner first learns a logistic regression model
$f_t(\inst{}, \vt{t} ) \approx (1+ \exp(\inst{}^\top \vt{t}))^{-1}$
for each of a number ($r$) of time points $t \in \{ t_1, t_2, \dots, t_r \}$
(\eg perhaps for the $r=60$ times $t$ ranging over 1month, 2months, ..., 60months).
If there are $k$ covariates, this meaning \MTLR-Learner will learn $k+1$ values for each time -- that is, 
it will learn the $k+1$-tuple $\vt{1}$ associated with the time $t_1$,  
then the $k+1$-tuple $\vt{2}$ associated with the time $t_2$,
and so forth.
The value of $f_i(\inst{}, \vt{i} )$ is designed to correspond to the probability mass function (PMF), for the interval $[t_i, t_{i+1}]$.
The associated survival function for any patient $\inst{j}$ is 1.0 minus the Cumulative Distribution Function, which for each time $t$ is the sum of PMFs up until $t$. 
}

There are many details here -- \eg 
to insure that the survival function starts at 1.0, and decreases monotonically and smoothly until reaching 0.0 for the final time point; 
to deal appropriately with censored patients; 
to decide how many time points to consider ($m$); and 
to minimize the risk of overfitting (by regularizing),
and by selecting the relevant features.
The paper by Yu~\etal~\cite{PSSP-NIPS} provides the details.

Afterwards, the \ISD-Predictor can use the learned \MTLR-model 
$\vtset^*= [\vt{1}, \dots, \vt{m}]$
to produce a curve for a novel patient, 
who is represented as the vector of his/her covariates $\inst{j}$.
This involves computing the $m$ values, 
$[f_1(\inst{j}, \vt{1} ), \dots,\ f_r(\inst{j}, \vt{m} )]$;
the running sum of these values is essentially the survival curve.
We then use splines to produce a smooth monotonically decreasing curve -- such as the 10 such curves shown in 
Figure~\ref{fig:AllCurves} (bottom-right).

\subsection{Extension to Random Survival Forests (\RSFKM)}
\label{app:RSFKM}
Given a labeled dataset, 
a random survival forest learner will  produce
a set of $T$ decision trees from a bootstrapped sample of the training data.
It grows each tree recursively, starting from the root
-- identifying each position with the set of patients 
who arrive there.
For each position, the growth stops if there are 
fewer than  $d_0$ deaths
(where $d_0$ is chosen via cross-validation).
Otherwise, it identifies the feature for this node:
it first randomly draws a small random subset of the features to consider,
then selects the feature (from that subset)
that maximizes 
the difference in survival between two daughter nodes,
based on the logrank test statistic (or some other chosen splitting rule).
This becomes the rule of that node;
and the learner then considers its two daughters,
by splitting on the node's feature.

\comment{For each node in every tree, a small random subset of the features is chosen on which to split such that the feature maximizes the difference in survival between daughter nodes. Here the difference in survival is based on the logrank test statistic . These trees are grown to full length such that terminal nodes have no fewer than $d_0$ deaths, where $d_0$ is chosen via cross-validation.}
% {Is this correct?}\note[HH]{Made Adjustments with a "soft" delete.}

%\annote[RG]{Each of these trees is restricted to use only a small random subset of the features, and only be a certain depth.}{Is this correct?}

Each leaf node in each tree corresponds to the 
set of training instances that reached that node.
Given these learned trees,
to classify a novel instance $\inst{}$,
the random forest performance system will 
drop $\inst{}$ into each of the trees,
which will lead to $T$ different leaf nodes, 
then use the $T$ subsets of training instances to make a decision. Since each terminal node in the random survival forest 
contains a set of instances,
we can use these instances to produce a Kaplan-Meier curve.%
\footnote{
While the original paper does not consider survival curves,  documentation
\url{https://kogalur.github.io/randomForestSRC/theory.html\#section8.1}
describing the inner workings of the R package states that survival curves in terminal nodes are created via the Kaplan-Meier estimator.}

Once the survival forest has been learned (with \(T\) trees),
a patient is dropped into each of the \(T\) survival trees,
leading to $T$ leaf nodes, which produces 
\(T\) Kaplan-Meier curves. The \RSFKM\ implementation then ``averages'' these curves,
by taking a point-wise average across the curve
for all time points -- see Figure~\ref{fig:RSFKMPlot}.%
\footnote{
The method for generating individual survival curves could not be found in any of the literature by the authors of random survival forests. Survival curves were reverse-engineered by the authors of this paper -- all survival curves tested matched the methodology explained here.} 

% It should be noted
Note that the risk score generated by the  median of the individual survival curves (produced here) does not 
necessarily result in the same ordering of patients as the risk scores of the original \RSF\ implementation,
which uses averaged cumulative hazards as a risk score. 
For this reason,
we also applied the original \RSF\ process to the datasets presented in the paper.
We found that the Concordance scores were similar to that of \RSFKM; 
\MTLR\ still outperformed \RSF\ on the datasets
where \MTLR\ outperformed \RSFKM\ (data not shown).

\begin{figure}[t]
\centering
\includegraphics[width = 0.5\textwidth]{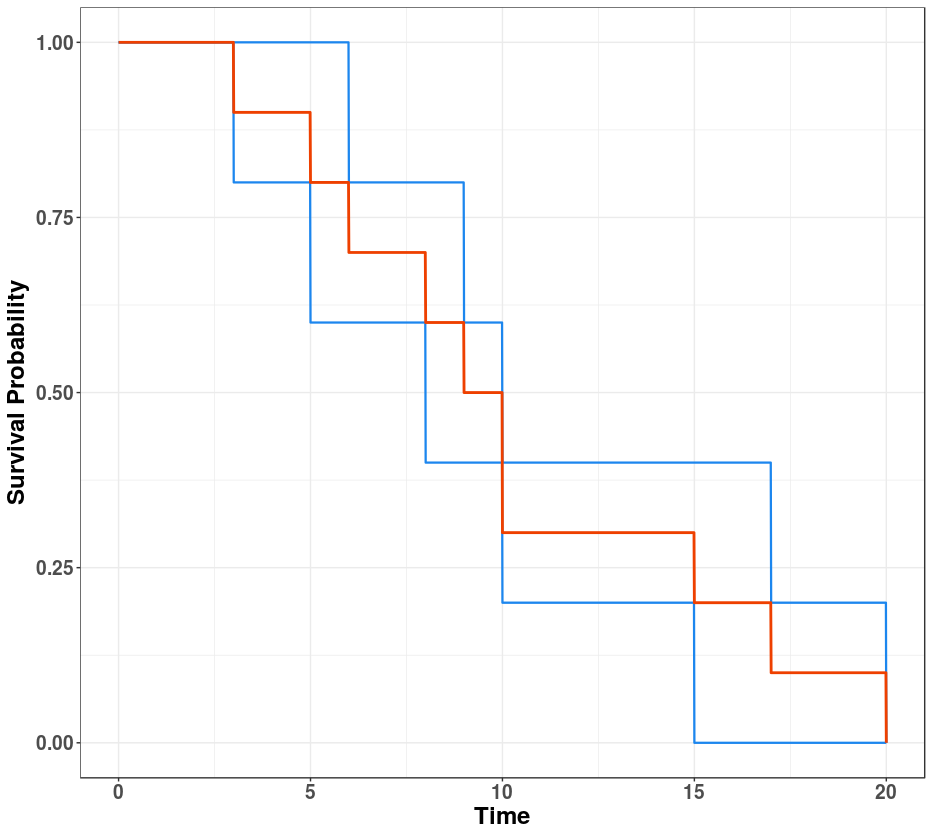}
\caption{\label{fig:RSFKMPlot}
This figure illustrates how to combine two different survival curves, to produce a new one. (\RSFKM\ uses this idea to ``merge'' the curves
obtained from the various leaf nodes reached by a novel instance.)
Here, two survival curves, given in blue,
are averaged to produce the survival curve shown in dark orange.
Note that the averaged curve is generated from a point-wise average, \ie new calculations must only be computed at each death time -- that is a drop in either (blue) Kaplan-Meier curve.}

\end{figure}

\comment{
\subsection{Comments about \CoxENKP}
\label{app:CoxKPEN}

The objective function of \CoxENKP a mixture of the partial log-likelihood for the Cox model and the penalty term, 
\newcommand{\norm}[1]{\left\lVert#1\right\rVert}
\[
\lambda \left(\frac{1- \alpha}{2}\norm{\beta}_2^2 + \alpha\norm{\beta}_1\right),
\]
where \(\beta\) are the feature coefficients and \(\lambda\) and \(\alpha\) are tuning parameters. Note that \(\alpha =1\) corresponds to the LASSO penalty. The \textit{cv.cocktail} function given by the \textit{fastcox} package~\cite{coxKPEN} in R will, by default, test 100 values for \(\lambda\) for a given value of \(\alpha\) using 5CV. As there is no default values for \(\alpha\) we performed grid search using \(\alpha \in \{0.01,0.2,0.4,0.6,0.8,1\}\) -- 0 is not an allowable parameter for \(\alpha\).  
}

\section{Detailed Empirical Results}  % C
\label{app:EmpiricalDetail}

This sub-appendix includes the tables 
that correspond to the figures given in Section~\ref{sec:Empirical Results}. 
Further, Appendix~\ref{app:1-CalDetails}
provides the all \(p\)-values for the 1-Calibration tests.

\subsection{Concordance} % C.1
\label{sec:Concord}

See Table~\ref{tab:Conc} for the results corresponding to Figure~\ref{fig:ConcEvaluation}.

\begin{table}[th]	\centering
\caption{\label{tab:Conc} Concordance results corresponding to Figure~\ref{fig:ConcEvaluation}. \textbf{Bold} values indicate the best (highest Concordance) performing model, for each dataset.}
\resizebox{\textwidth}{!}{
	\begin{tabular}{ >{\columncolor{LightCyan}}r|cccccHccc}
		\hline
		\rowcolor{gray!50}   & GBM & GLI & \NacdCol & NACD & READ & THCA & BRCA & DBCD & DLBCL \\ 
		\hline
%	N & 592 & 1105 & 950 & 2402 & 170 & 503 & 295 & 240 \\ 
%	\% Censored & 17.23 & 44.34 & 51.89 & 36.59 & 84.12 & 96.82 & 73.22 & 42.50 \\ 
%	$f_{final}$ & 6 & 10 & 34 & 46 & 8 & 14 & 2330 & 1771 \\ 
% \hline
	\KM & 0.500 (0.000) & 0.500 (0.000) & 0.500 (0.000) & 0.500 (0.000) & 0.500 (0.000) & 0.500 (0.000) & 0.500 (0.000) & 0.500 (0.000) & 0.500 (0.000)) \\ 
	\AFT & 0.692 (0.026) & 0.802 (0.011) & 0.722 (0.015) & 0.755 (0.008) & 0.700 (0.183) & 0.705 (0.285) & 0.738 (0.041)&0.474 (0.056) & 0.558 (0.046) \\ 
	\CoxKP & 0.696 (0.026) & 0.805 (0.01) & 0.722 (0.015) & 0.755 (0.008) & 0.716 (0.157) & 0.675 (0.300)&0.747 (0.030) & - & - \\ 
\CoxENKP & 0.698 (0.032) & 0.801 (0.006) & 0.726 (0.024) & 0.754 (0.009) & 0.731 (0.141) & \textbf{0.955 (0.032)} & 0.761 (0.032) &0.744 (0.055) & 0.695 (0.038) \\ 
  \RSFKM & 0.650 (0.038) & 0.776 (0.021) & \textbf{0.743 (0.024)} & \textbf{0.758 (0.005)} & 0.582 (0.093) & 0.950 (0.027) &0.704 (0.036)& 0.738 (0.060) & 0.639 (0.038) \\ 
  \MTLR & \textbf{0.703 (0.032)} & \textbf{0.812 (0.013)} & 0.734 (0.023) & 0.757 (0.008) & \textbf{0.768 (0.091)} & 0.913 (0.080) & \textbf{0.770 (0.027)} &\textbf{0.766 (0.065)} & \textbf{0.704 (0.034)} \\ 
				\hline
			\end{tabular}}
\end{table}

\subsection{Brier Score}

See Table~\ref{tab:Brier} for the results corresponding to Figure~\ref{fig:BrierEvaluation}.

\begin{table}[H]
	\caption{\label{tab:Brier} Integrated Brier score results corresponding to Figure~\ref{fig:BrierEvaluation}. 
{\bf Bold} values indicate the best performing model for each dataset
-- with the lowest Integrated Brier score. 
Note that this table (and all following tables)
may show ties (up to three digits),
but will only bold the one with the best performance,
based on additional digits, not shown.
 }
	\centering
	\resizebox{\textwidth}{!}{
		\begin{tabular}{ >{\columncolor{LightCyan}}r|cccccHccc}
			\hline
			\rowcolor{gray!50}  & GBM & GLI & \NacdCol & NACD & READ & THCA &BRCA& DBCD & DLBCL \\ 
			\hline
		%	N & 592 & 1105 & 950 & 2402 & 170 & 503 & 295 & 240 \\ 
	%		\% Censored & 17.23 & 44.34 & 51.89 & 36.59 & 84.12 & 96.82 & 73.22 & 42.50 \\ 
%			$f_{final}$ & 6 & 10 & 34 & 46 & 8 & 14 & 2330 & 1771 \\ 
			%\hline
			 \KM & 0.046 (0.001) & 0.034 (0.000) & 0.083 (0.001) & 0.089 (0.000) & 0.027 (0.005) & 0.011 (0.002) & 0.017 (0.001) &0.097 (0.003) & 0.109 (0.004) \\ 
			 \AFT & 0.041 (0.003) & 0.021 (0.002) & 0.066 (0.003) & 0.065 (0.002) & 0.026 (0.017) & \textbf{0.007 (0.005)} &0.0133 (0.002)& 0.254 (0.023) & 0.295 (0.038) \\ 
			 \CoxKP & 0.040 (0.003) & 0.022 (0.003) & 0.066 (0.003) & 0.067 (0.002) & 0.026 (0.017) & 0.010 (0.012)&0.014 (0.002) & - & - \\ 
\CoxENKP & 0.040 (0.003) & 0.021 (0.002) & 0.064 (0.003) & 0.065 (0.001) & 0.024 (0.011) & 0.007 (0.003) & 0.015 (0.002) &\textbf{0.070 (0.004)} & 0.078 (0.013) \\ 
  \RSFKM & 0.059 (0.009) & 0.051 (0.009) & 0.079 (0.006) & 0.079 (0.001) & 0.047 (0.013) & 0.007 (0.003) & 0.028 (0.002)&0.077 (0.003) & 0.095 (0.013) \\ 
  \MTLR & \textbf{0.039 (0.004)} & \textbf{0.019 (0.002)} & \textbf{0.062 (0.003)} &\textbf{ 0.063 (0.001)} & \textbf{0.023 (0.006)} & 0.008 (0.002) &\textbf{0.012 (0.001)} & 0.070 (0.003) & \textbf{0.078 (0.011)} \\  
			\hline
		\end{tabular}}
	\end{table}
\subsection{Empirical Values of L1-Loss, and Variants}
\label{app:L1LossEmpirical}

Here we give the the results for the Margin-L1-loss
(Table~\ref{tab:L1}) as given in Figure~\ref{fig:L1Evaluation}.
Additionally, we give results for 
the Uncensored L1-loss (Table~\ref{tab:L1Unc})
and the Log-Margin-L1-loss (Table~\ref{tab:LogL1}).

\begin{table}[H]
	\caption{\label{tab:L1} Margin-L1-loss results corresponding to Figure~\ref{fig:L1Evaluation}.  {\bf Bold} values indicate the best performing model for each dataset -- with the lowest Margin-L1-loss.}
	\centering
	\resizebox{\textwidth}{!}{
		\begin{tabular}{ >{\columncolor{LightCyan}}r|cccccHccc}
			\hline
			\rowcolor{gray!50} & GBM & GLI & \NacdCol & NACD & READ & THCA & BRCA & DBCD & DLBCL \\ 
			\hline
		%	N & 592 & 1105 & 950 & 2402 & 170 & 503 & 295 & 240 \\ 
	%		\% Censored & 17.23 & 44.34 & 51.89 & 36.59 & 84.12 & 96.82 & 73.22 & 42.50 \\ 
%			$f_{final}$ & 6 & 10 & 34 & 46 & 8 & 14 & 2330 & 1771 \\ 
			%\hline
			\KM & 1431.31 (59.25) & 2746.70 (91.85) & 56.45 (1.95) & 61.97 (0.50) & 3677.90 (222.77) & 26648.06 (3655.84) & 5392.04 (128.19)& 24.88 (0.78) & 20.28 (2.06) \\ 
			\AFT & \textbf{1240.60 (57.38)} & 1838.20 (105.23) & 47.89 (1.85) & 43.99 (1.44) & 4068.72 (1451.14) & 43553.71 (15361.15) & 5156.1 (264.50)&47.01 (4.38) & 27.29 (3.31) \\ 
			\CoxKP & 1278.18 (44.02) & 1824.06 (127.49) & 45.53 (2.23) & 44.26 (1.29) & 4799.91 (1460.52) & 37948.81 (15235.14) &6247.01 (612.3)& - & - \\ 
\CoxENKP & 1347.36 (51.6) & 1683.04 (110.82) & 45.25 (1.94) & 45.72 (1.43) & 3564.01 (1163.92) & 29937.76 (4821.03) & \textbf{4593.52 (370.75)} & 24.28 (2.14) & 15.97 (1.10) \\ 
  \RSFKM & 1399.17 (99.06) & 4503.49 (465.47) & 58.81 (2.28) & 49.69 (1.38) & 6805.00 (1710.50) & 30600.90 (6000.37) & 10934.45 (579.44)& 26.58 (2.35) & 17.68 (1.81) \\ 
  \MTLR & 1271.73 (37.71) & \textbf{1582.72 (131.1)} & \textbf{43.48 (2.52)} &\textbf{ 43.97 (1.20)} & \textbf{3417.49 (256.83)} & \textbf{22507.2 (5265.95)} &4669.55 (153.50)& \textbf{20.01 (1.47)} & \textbf{15.52 (1.97)} \\ 
			\hline
		\end{tabular}}
	\end{table}

\begin{table}[H]
	\caption{\label{tab:L1Unc} Uncensored L1-loss (not the L1-Margin loss given in Figure \ref{fig:L1Evaluation}).  {\bf Bold} values indicate the best performing model for each dataset -- with the lowest L1-loss.}
	\centering
	\resizebox{\textwidth}{!}{
		\begin{tabular}{ >{\columncolor{LightCyan}}r|cccccHccc}
  \hline
			\rowcolor{gray!50} & GBM & GLI & \NacdCol & NACD & READ & THCA & BRCA & DBCD & DLBCL \\ 
			\hline
%			N & 592 & 1105 & 950 & 2402 & 170 & 503 & 295 & 240 \\ 
%			\% Censored & 17.23 & 44.34 & 51.89 & 36.59 & 84.12 & 96.82 & 73.22 & 42.5 \\ 
%			$f_{final}$ & 6 & 10 & 34 & 46 & 8 & 14 & 2330 & 1771 \\ 
%           \hline
  \KM & 318.53 (10.49) & 520.32 (9.56) & 19.40 (0.15) & 12.17 (0.08) & 1829.99 (1261.87) & 24927.82 (3508.1) & 2418.79 (24.55)& 18.69 (0.64) & \textbf{2.90 (0.20)} \\ 
  \AFT & 291.28 (28.99) & 524.78 (64.89) & 19.58 (1.38) & 12.54 (1.07) & 1738.27 (1191.64) & 19979.11 (17054.45) & 2681.47 (453.66)&23.77 (3.95) & 13.48 (3.01) \\ 
  \CoxKP & 281.61 (26.53) & 542.35 (76.2) & 18.91 (1.35) & 12.63 (1.12) & 2248.56 (1263.39) & 13803.72 (17229.5) &3141.20 (521.90) & - & - \\ 
  \CoxENKP & 284.86 (14.93) & 482.70 (44.76) & 16.14 (1.12) & \textbf{10.64 (0.66)} & 1936.34 (1186.59) & 8782.37 (5531.82) &2647.76 (399.62) & 11.52 (1.51) & 3.19 (1.38) \\ 
  \RSFKM & 373.14 (72.48) & 1204.95 (276.78) & 28.33 (2.15) & 15.61 (1.07) & 4099.14 (1482.32) & 8047.65 (6798.73)& 5671.22 (511.99) & 13.07 (1.51) & 3.91 (0.93) \\ 
  \MTLR & \textbf{272.20 (27.15)} & \textbf{436.99 (62.40)} & \textbf{15.87 (1.03)} & 10.71 (0.57) & \textbf{1411.46 (306.97)} &\textbf{7691.34 (3677.10)}& \textbf{2167.09 (208.32)} & \textbf{9.30 (1.30)} & 3.70 (0.92) \\ 
   \hline
\end{tabular}}
\end{table}

\begin{table}[H]
	\caption{\label{tab:LogL1} Log-Margin-L1-loss (not the L1-Margin loss given in Figure \ref{fig:L1Evaluation}).  {\bf Bold} values indicate the best performing model for each dataset -- with the lowest Log-Marign-L1-loss.}
	\centering
	\resizebox{\textwidth}{!}{
		\begin{tabular}{ >{\columncolor{LightCyan}}r|cccccHccc}
  \hline
			\rowcolor{gray!50} & GBM & GLI & \NacdCol & NACD & READ & THCA & BRCA &  DBCD & DLBCL \\ 
			\hline
%			N & 592 & 1105 & 950 & 2402 & 170 & 503 & 295 & 240 \\ 
%			\% Censored & 17.23 & 44.34 & 51.89 & 36.59 & 84.12 & 96.82 & 73.22 & 42.5 \\ 
%			$f_{final}$ & 6 & 10 & 34 & 46 & 8 & 14 & 2330 & 1771 \\ 
%           \hline
    \KM & 1.98 (0.01) & 2.35 (0.02) & 1.76 (0.04) & 2.20 (0.01) & 2.05 (0.19) & 3.41 (0.26) & 1.78 (0.03) & 2.02 (0.04) & 3.37 (0.11) \\ 
   \AFT & \textbf{1.64 (0.08)} & 1.36 (0.06) & 1.44 (0.01) & \textbf{1.48 (0.04)} & 2.00 (0.60) & 2.96 (1.20) & 1.65 (0.14)& 6.21 (0.68) & 10.54 (2.62) \\ 
  \CoxKP & 1.68 (0.07) & 1.35 (0.06) & 1.42 (0.02) & 1.48 (0.04) & 2.02 (0.56) & 2.96 (1.13)&1.73 (0.11) & - & - \\ 
  \CoxENKP & 1.80 (0.01) & 1.35 (0.05) & 1.43 (0.03) & 1.55 (0.05) & \textbf{1.88 (0.56)} & 2.63 (0.42) & \textbf{1.56 (0.06)} & 1.77 (0.08) & 2.70 (0.11) \\ 
  \RSFKM & 1.85 (0.20) & 1.98 (0.10) & 1.51 (0.05) & 1.54 (0.04) & 2.62 (0.47) & \textbf{2.52 (0.64)} & 2.21 (0.12) & 1.90 (0.05) & 3.04 (0.14) \\ 
  \MTLR & 1.67 (0.07) & \textbf{1.25 (0.08)} & \textbf{1.37 (0.05)} & 1.50 (0.04) & 1.95 (0.17) & 2.89 (0.55) & 1.57 (0.06) &\textbf{ 1.67 (0.07)} & \textbf{2.60 (0.18)} \\ 
   \hline
\end{tabular}}
\end{table}

\subsection{1-Calibration}
Each table corresponds to a different percentile of event times for each dataset. Moving down the 10th, 25th, 50th, 75th, and 90th percentiles are given. \textbf{Bolded} values indicate models which passed 1-Calibration (\(p > 0.05\)).
The ``Total'' column of each table gives the total number of datasets 
passed by each model -- 
that is, the values in that % the total 
columns correspond to Table~\ref{tab:cum1Cal}.

\label{app:1-CalDetails}

\begin{table}[ht]
		\caption{1-Calibration Results at \(\tzero\) = 10th Percentile of Event Times}
	\centering
	\resizebox{\textwidth}{!}{
		\begin{tabular}{ >{\columncolor{LightCyan}}r|cccccHccc|c}
			\hline
			\rowcolor{gray!50}  & GBM & GLI & \NacdCol & NACD & READ & THCA & BRCA & DBCD & DLBCL & Total \\ 
			\hline
%			N & 592 & 1105 & 950 & 2402 & 170 & 503 & 295 & 240 \\ 
%			\% Censored & 17.23 & 44.34 & 51.89 & 36.59 & 84.12 & 96.82 & 73.22 & 42.5 \\ 
%			$f_{final}$ & 6 & 10 & 34 & 46 & 8 & 14 & 2330 & 1771 \\ 
%           \hline
  \AFT & 0.001 & \textbf{0.159} & \textbf{0.794} & 0.012 & \textbf{1.000} & 0.000 &\textbf{0.919}& 0.000 & 0.000 & 4 \\ 
  \CoxKP & 0.001 & \textbf{0.140} & \textbf{0.794} & 0.008 & \textbf{0.999} & 0.000 &\textbf{0.782}& - & - & 4 \\ 
  \CoxENKP & 0.000 & 0.033 & 0.043 & 0.000 & \textbf{0.999} & \textbf{0.871} & \textbf{0.561} & \textbf{0.454} & \textbf{0.646} & 4 \\ 
  \RSFKM & 0.000 & 0.000 & \textbf{0.078} & 0.016 & \textbf{0.998} & \textbf{1.000} & 0.000 & \textbf{0.164} & \textbf{0.273} & 4 \\ 
  \MTLR & \textbf{0.908} & \textbf{0.450} & \textbf{0.440} & 0.047 & \textbf{1.000} & \textbf{0.506} & \textbf{0.929} &  0.000 & \textbf{0.177} & 6\\ 
			\hline
		\end{tabular}}

	\end{table}
	
	\begin{table}[ht]
		\caption{1-Calibration Results at \(\tzero\) = 25th Percentile of Event Times}
		\centering
		\resizebox{\textwidth}{!}{
			\begin{tabular}{ >{\columncolor{LightCyan}}r|cccccHccc|c}
				\hline
				\rowcolor{gray!50}  & GBM & GLI & \NacdCol & NACD & READ & THCA & BRCA & DBCD & DLBCL & \textbf{Total} \\ 
				\hline
%			N & 592 & 1105 & 950 & 2402 & 170 & 503 & 295 & 240 \\ 
%			\% Censored & 17.23 & 44.34 & 51.89 & 36.59 & 84.12 & 96.82 & 73.22 & 42.5 \\ 
%			$f_{final}$ & 6 & 10 & 34 & 46 & 8 & 14 & 2330 & 1771 \\ 
%           \hline
				\AFT & 0.000& 0.040 & \textbf{0.586} & 0.009 & 0.000& 0.000& \textbf{0.205} & 0.000& 0.000 & 2\\ 
				\CoxKP & 0.000& 0.008 & \textbf{0.379} & 0.003 & 0.000& 0.000&\textbf{0.535}& - & - & 2 \\ 
			\CoxENKP & 0.000 & 0.002 & 0.003 & 0.000 & \textbf{0.238} & \textbf{0.915} & 0.044 & \textbf{0.436} & \textbf{0.547} & 3 \\ 
              \RSFKM & 0.000 & 0.000 & \textbf{0.312} & 0.006 & 0.000 & \textbf{0.999}& 0.000 & 0.042 & \textbf{0.227} & 2 \\ 
              \MTLR & \textbf{0.963} & \textbf{0.312} & \textbf{0.645} & \textbf{0.254} & \textbf{0.449} & 0.046 &\textbf{0.448}& \textbf{0.177} & \textbf{0.052} &8 \\ 
				\hline
			\end{tabular}}
		\end{table}
		
		\begin{table}[ht]
		\caption{1-Calibration Results at \(\tzero\) = 50th Percentile of Event Times}
			\centering
			\resizebox{\textwidth}{!}{
				\begin{tabular}{ >{\columncolor{LightCyan}}r|cccccHccc|c}
					\hline
					\rowcolor{gray!50}  & GBM & GLI & \NacdCol & NACD & READ & THCA & BRCA & DBCD & DLBCL & \textbf{Total}\\ 
					\hline
%			N & 592 & 1105 & 950 & 2402 & 170 & 503 & 295 & 240 \\ 
%			\% Censored & 17.23 & 44.34 & 51.89 & 36.59 & 84.12 & 96.82 & 73.22 & 42.5 \\ 
%			$f_{final}$ & 6 & 10 & 34 & 46 & 8 & 14 & 2330 & 1771 \\ 
%           \hline
					\AFT & \textbf{0.117} & 0.030 & 0.035 & 0.043 & 0.000& 0.000& 0.000& 0.000& 0.000& 1\\ 
					\CoxKP & \textbf{0.495 }& 0.005 & 0.038 & \textbf{0.124} & 0.000& 0.000 & 0.017& - & - & 2 \\ 
  \CoxENKP & 0.019 & 0.000 & 0.000 & 0.000 & 0.049 & \textbf{0.978} & 0.000& 0.025 & \textbf{0.822} & 1 \\ 
  \RSFKM & 0.000 & 0.000 & \textbf{0.761} & 0.001 & 0.000 & \textbf{0.992} & 0.000 & 0.000 & \textbf{0.068} & 2 \\ 
  \MTLR & \textbf{0.796} & \textbf{0.306} & \textbf{0.813} & \textbf{0.112} & \textbf{0.995}  & \textbf{0.765} &0.013& 0.041 & \textbf{0.262} & 6 \\ 
					\hline
				\end{tabular}}
			\end{table}
			
			\begin{table}[ht]
				\caption{1-Calibration Results at \(\tzero\) = 75th Percentile of Event Times}
				\centering
				\resizebox{\textwidth}{!}{
					\begin{tabular}{ >{\columncolor{LightCyan}}r|cccccHccc|c}
						\hline
						\rowcolor{gray!50}  & GBM & GLI & \NacdCol & NACD & READ & THCA & BRCA & DBCD & DLBCL & \textbf{Total}\\ 
						\hline
%			N & 592 & 1105 & 950 & 2402 & 170 & 503 & 295 & 240 \\ 
%			\% Censored & 17.23 & 44.34 & 51.89 & 36.59 & 84.12 & 96.82 & 73.22 & 42.5 \\ 
%			$f_{final}$ & 6 & 10 & 34 & 46 & 8 & 14 & 2330 & 1771 \\ 
%           \hline
						\AFT & \textbf{0.378} & 0.000& 0.002 & 0.002 & 0.000& 0.000& 0.000&  0.000& 0.000& 1\\ 
						\CoxKP & 0.008 & 0.000& 0.003 & 0.004 & \textbf{0.087} & 0.000&0.016 &- & - & 1\\ 
  \CoxENKP & \textbf{0.338} & 0.000 & 0.000 & 0.000 & 0.001 & \textbf{0.355} & 0.000& 0.003 & \textbf{0.436} & 2 \\ 
  \RSFKM & 0.000 & 0.000 & \textbf{0.070} & 0.003 & 0.000 & 0.010 & 0.000& 0.002 & 0.038  & 1\\ 
  \MTLR & \textbf{0.140} & \textbf{0.565} & 0.044 & 0.045 & 0.026 & 0.011 & 0.000 & 0.036 & \textbf{0.218} & 3 \\
						\hline
					\end{tabular}}
				\end{table}
				
				\begin{table}[ht]
					\caption{1-Calibration Results at \(\tzero\) = 90th Percentile of Event Times}
					\centering
					\resizebox{\textwidth}{!}{
						\begin{tabular}{ >{\columncolor{LightCyan}}r|cccccHccc|c}
							\hline
							\rowcolor{gray!50}  & GBM & GLI & \NacdCol & NACD & READ & THCA & BRCA &  DBCD & DLBCL & \textbf{Total} \\ 
							\hline
%			N & 592 & 1105 & 950 & 2402 & 170 & 503 & 295 & 240 \\ 
%			\% Censored & 17.23 & 44.34 & 51.89 & 36.59 & 84.12 & 96.82 & 73.22 & 42.5 \\ 
%			$f_{final}$ & 6 & 10 & 34 & 46 & 8 & 14 & 2330 & 1771 \\ 
%           \hline
							\AFT & 0.000& 0.000& 0.000& 0.000& 0.000& 0.000& 0.000 & 0.000& 0.000 & 0\\ 
							\CoxKP & 0.000& 0.000& 0.000& 0.000& 0.000& 0.000& 0.000 & - & - & 0 \\ 
  \CoxENKP & \textbf{0.050} & 0.000 & 0.004 & 0.000 & 0.000 & \textbf{0.058} & 0.000 & 0.010 & \textbf{0.112} & 2\\ 
  \RSFKM & 0.000 & 0.000 & 0.000 & 0.000 & 0.000 & 0.000 & 0.000 & 0.000 & 0.023 & 0\\ 
  \MTLR & \textbf{0.109} & \textbf{0.148} & 0.000 & 0.001 & 0.000 & 0.001 & 0.000 & \textbf{0.098} & \textbf{0.157} & 4\\ 
							\hline
						\end{tabular}}
					\end{table}
\end{document}